\documentclass{article}

 \usepackage[preprint]{neurips_2026}

\usepackage[utf8]{inputenc} % allow utf-8 input
\usepackage[T1]{fontenc}    % use 8-bit T1 fonts
\usepackage{hyperref}       % hyperlinks
\usepackage{url}            % simple URL typesetting
\usepackage{booktabs}       % professional-quality tables
\usepackage{amsfonts}       % blackboard math symbols
\usepackage{nicefrac}       % compact symbols for 1/2, etc.
\usepackage{microtype}      % microtypography
\usepackage[most]{tcolorbox}
\usepackage{enumitem}
\usepackage{xcolor}         % colors

\usepackage{tcolorbox}
\usepackage{tikz}
\usetikzlibrary{positioning, arrows.meta}
\usepackage{graphicx}
\usepackage{wrapfig}
\usepackage{placeins}
\usepackage{multirow}

\usepackage{amsmath}
\usepackage{enumitem}
\usepackage{amsthm}

\newtheorem{proposition}{Proposition}[section]

\newtheorem{definition}{Definition}[section]
\newtheorem{remark}{Remark}[section]

\usepackage{xspace}
\newcommand{\method}{PU-DPO\xspace}

\DeclareMathAlphabet\mathbfcal{OMS}{cmsy}{b}{n}

\def\to{{\,\rightarrow\,}}

\mathchardef\mhyphen="2D

\newcommand{\norm}[1]{{ \left\lVert\right\rVert }}
\newcommand{\vertiii}[1]{{\left\vert\kern-0.25ex\left\vert\kern-0.25ex\left\vert #1
    \right\vert\kern-0.25ex\right\vert\kern-0.25ex\right\vert}}

\title{Positive-Unlabeled Preference Optimization For Chest X-ray Report Generation}

\author{%
Yuta Kobayashi\thanks{Equal contribution, Corresponding author: yk3043@cumc.columbia.edu}\\Columbia University \And
Pradyun Ramesh$^*$\\Columbia University \And
Muhammad Ahmed Chaudhry\\Stanford University \And
Vincent Jeanselme\\Columbia University \And
Judy Wawira Gichoya \\ Emory University \And
Sanmi Koyejo \\ Stanford University \And
Kathleen Capaccione\\Columbia University \And
Shalmali Joshi\\Columbia University
}

\begin{document}

\maketitle

\begin{abstract}
Vision-Language Models (VLMs) for radiology report generation are typically trained on retrospective clinical reports, which suffer from omission noise: clinically present findings are left unreported due to the omission of subtle findings. For example, prior studies show that cardiomegaly may be omitted from ICU chest X-ray reports when the imaging request is focused on monitoring support device placement \citep{jain2021visualchexbert}. As a result, models trained with standard approaches inherit these omissions, learning to under-report findings themselves. We propose \method, a preference optimization framework to prevent omission noise from corrupting the preference signal. We reformulate the objective under a positive-unlabeled (PU) learning framework, treating absent mentions as unlabeled rather than truly negative. Our framework provides preference supervision using constructed contrastive pairs, generated using edits to model responses, producing variants that explicitly mention or omit a specific finding. Generated responses that mention the finding are naturally preferred in the context of visual evidence. Across semi-synthetic experiments and analyses on real-world chest radiograph benchmarks where adjudicated labels are available, \method yields consistent gains in detection rates and recovery of hidden positives across multiple pathologies, and is more robust to omission noise than prior approaches.
\end{abstract}

\section{Introduction}
Automated report generation is a promising application of Vision-Language Models (VLMs), where, given a medical image, the model conditionally generates text detailing pathological findings and abnormalities. Typically, these models are trained on retrospective clinical reports rather than prospective. That is, exhaustively labeled annotations (that explicitly record both present and absent/negated findings), are rarely available at scale. Historically, the predominant training paradigm for radiology report generation has been Supervised Fine-Tuning (SFT), where the model maximizes the conditional likelihood of the clinician-generated reference report, which is treated as a gold-standard `ground-truth'  given the input chest X-ray, a process conceptually analogous to imitation learning~\citep{ouyang2022training}. While this process successfully aligns the model's output with reporting standards, it is fundamentally constrained by the quality of annotations. 

We introduce the notion of \emph{omission noise} in radiology report generation as one-sided noise in which positive visual findings are unlabeled in retrospective reports, not because of diagnostic uncertainty or interpretative disagreement but due to clinical workflow and prioritization \citep{rudolph2021interpretation, majkowska2020chest, jain2021visualchexbert}. A clinician might omit the diagnosis of mild cardiomegaly owing to the subtlety compared to other pathologies prioritized in the Intensive Care Unit (ICU) \citep{jain2021visualchexbert, yu2022anatomy}. We confirm this in a study we conducted with practicing radiologists. That is, retrospective reports may omit clinically present findings (Table~\ref{tab:study}, Section~\ref{sec:radiologist_study}), suggesting that the supervisory signal itself is systematically incomplete.

\begin{figure}[t!]
    \centering
    \includegraphics[width=0.99\linewidth]{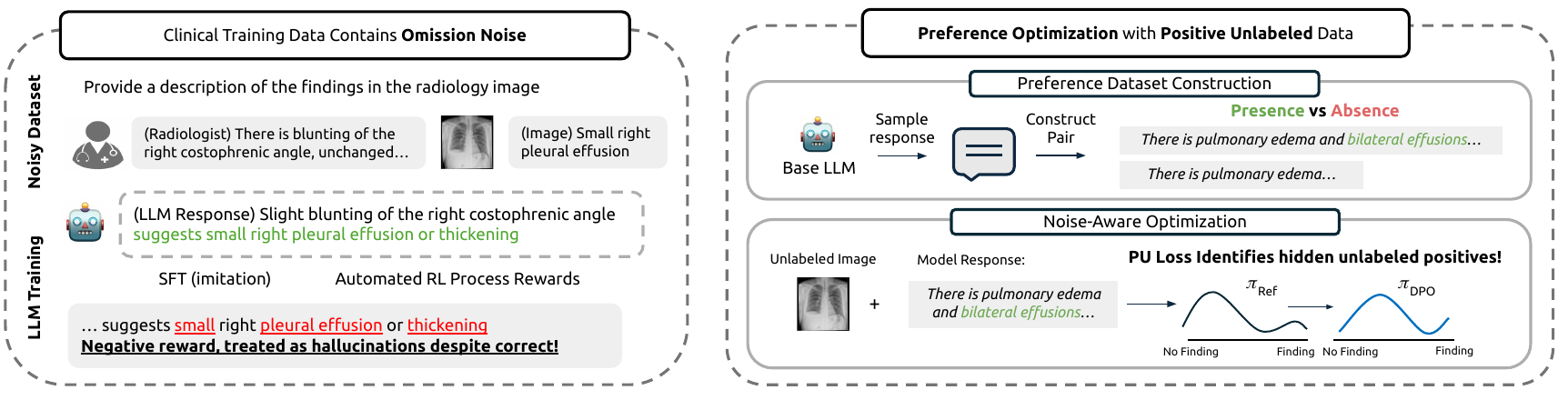}
    \caption{\textbf{Overview of Omission Noise and Proposed PU Learning Framework.} (\emph{Left}) We highlight the challenge of omission noise in training medical VLMs, where clinical findings are visually present but omitted from reports due to standard clinical workflow variability. (\emph{Right}) Our proposed approach mitigates this by constructing preference pairs that evaluate the presence versus absence of a single finding. We then apply a noise-aware preference optimization technique rooted in PU learning, enabling the model to detect hidden positives otherwise lost in the unannotated data.}
    \label{fig:overview}
    \vspace{-22pt}
\end{figure}

Report generation quality has been improved using post-training alignment methods that incorporate reinforcement learning (RL) to optimize for clinical correctness rather than surface-level text similarity, using automated, accuracy-based reward functions (e.g., RadGraph-F1, RadCLIQ) aggregated over the entire report \citep{fan2025chestx, liang2025chexpo, liu2026scaling}. This is also combined with training strategies that distill "reasoning" ability into the VLM by generating chain-of-thought (CoT) tokens prior to the final response \citep{myronenko2025reasoning}. However, approaches that derive reward signals or CoT supervision from the retrospective reference report remain inherently bound by omission noise (Fig.~\ref{fig:overview}). Consequently, principled strategies for distilling fine-grained generation capabilities from noisy annotations are lacking.

% lead to critical shortcomings: (1) aggregate reward functions lack mechanisms for precise concept-level credit assignment, and (2) these automated metrics remain inherently bound to the omission noise of the original retrospective reports. Consequently, principled strategies for distilling fine-grained generation capabilities from noisy annotations are lacking. 

To address this shortcoming, we propose Positive-Unlabeled Direct Preference Optimization (\method) with positive-and-unlabeled (PU) learning, a scalable and principled framework for RL-based finetuning that explicitly targets the fine-grained subtypes of disease presentation. Our finetuning framework is flexible and compatible with both standard and CoT-based VLMs. Our approach consists of two main steps. First, we decompose a target pathology into subtypes defined by characteristics such as anatomical loanatomical location and severity, and generate a synthetic preference dataset by editing on-policy responses from a reference VLM to produce targeted contrastive pairs: one that explicitly mentions the finding, and one that omits it. 

\begin{wraptable}{r}{0.45\textwidth}
    \label{tab:study} 
    \vspace{-15pt}
    \centering
    \footnotesize 
    
    % --- THE FIX ---
    \setlength{\tabcolsep}{2pt} % Reduces the space between columns. Adjust '3pt' as needed.
    
    \caption{Omission and hidden positive rates of findings in retrospective reference reports in the CheXpert dataset, compared to radiologists' annotations of radiograph in the study.}
    \renewcommand{\arraystretch}{1.3} 
    
    \begin{tabular}{@{} l c c @{}}
        \toprule
        Condition & Omission Rate & Hidden Positive Rate \\
        \midrule
        Pleural Effusion & 16/142 (11.3\%) & 16/81 (19.8\%) \\
        Atelectasis & 21/55 (38.2\%) & 21/213 (9.9\%) \\
        Cardiomegaly & 14/32 (43.8\%) & 14/195 (7.2\%) \\
        \bottomrule
    \end{tabular}
\end{wraptable}

However, under standard preference optimization, omission noise creates systematic unlabeled "false negatives" that teach the model, incorrectly, that excluding the pathology is preferred. To handle this, we adopt PU learning \citep{elkan2008learning, bekker2020learning}. We treat explicitly mentioned pathologies as positive labels and absent mentions as unlabeled rather than truly negative. We operate under the selected-completely-at-random (SCAR) assumption, which posits that the probability of each unique pathology being labeled (i.e., mentioned in the report) is uniform and independent of the input. We reformulate the standard DPO objective under this assumption to account for omission noise, and pair this with a known or estimated class prior.

Our contributions are:
\begin{itemize}[leftmargin=*, nosep, itemsep=4pt, topsep=2pt]
    \item \textbf{Novel Objective Function:} We propose \method, a principled approach for preference optimization under omission noise, demonstrated in chest radiograph report generation. We construct preference pairs comparing the presence and absence of observations and incorporate PU learning into a novel objective function.
    \item \textbf{Scalable Alignment Algorithm:} We introduce a scalable fine-tuning framework that integrates \method for granular multi-pathology detection. Our approach is model-agnostic, supporting both standard and CoT-based VLMs.
    \item \textbf{Empirical Evaluation:} We validate \method in semi-synthetic settings and across two chest radiograph benchmarks and three base VLM architectures, demonstrating consistent robustness to omission noise compared to standard alignment baselines.
\end{itemize}

% \begin{figure}[!ht]
% \vspace{-0.5em}
%     \begin{minipage}{0.6\textwidth}
%         Our contributions are:
%         \begin{itemize}[leftmargin=*] % Use enumitem parameters here
%             \item We propose \method, a principled approach for preference optimization under omission noise, demonstrated in chest radiograph report generation (see Table~\ref{tab:omission_rates}). We construct preference pairs comparing the presence and absence of observations and incorporate PU learning into a novel objective function.
%             \item A scalable finetuning algorithm incorporating \method for granular multi-pathology detection in clinical report generation, compatible with both "chain-of-thought" reasoning and standard VLMs.
%             \item Experiments across two chest radiograph benchmarks and three base VLMs demonstrating robustness to omission noise.
%         \end{itemize}
%     \end{minipage}
%     \hfill 
%     \begin{minipage}{0.35\textwidth}
%         \centering
%         \includegraphics[width=\linewidth]{figures/figure1_teaser.pdf}
%         \caption{Our approach maintains high F1 under synthetic omission noise in report generation.}
%         \label{fig:placeholder}
%     \end{minipage}
%     \vspace{-1em}
% \end{figure}

\section{Related Work}
\textbf{Report Generation via VLMs.}
We focus on VLM architectures for chest X‑ray (CXR) report generation that fuse visual tokens directly into a large-scale, autoregressive language decoder, such as Flamingo‑CXR~\citep{tanno2025collaboration}, Llava-RAD~\citep{zambrano2025clinically}, and MedGemma~\citep{sellergren2025medgemma}. To optimize the autoregressive policy, the prevailing strategy involves a two-stage process: Supervised Fine-Tuning (SFT) followed by Reinforcement Learning (RL). Methods such as RadVLM and MRG‑R1 \citep{gundersen2026radvlm, wang2025mrg} utilize Group‑Relative Policy Optimization (GRPO) with clinically informed reward functions. RadVLM directly optimizes RadCliQ~\citep{yu2023evaluating}, a metric combining BLEU \citep{papineni2002bleu}, ROUGE \citep{lin2004rouge}, BERTScore~\citep{zhang2019bertscore}, CheXbert similarity~\citep{smit2020combining}, and RadGraph‑F1~\citep{jain2021radgraph}, to better capture radiological accuracy when compared to a reference report. Frameworks such as UniRG and VALOR \citep{liu2026scaling, bose2025visual} scale these optimizations using additional image-based rewards to improve visual grounding. Parallel to this, recent work has sought to incorporate visual reasoning using Chain-of-Thought (CoT) tokens in language modeling \citep{myronenko2025reasoning}. However, we argue that even with perfect grounding and reasoning, these strategies remain fundamentally bottlenecked and inherently subject to the omission noise in retrospective clinical reports.

\textbf{Positive-Unlabeled (PU) Learning.} PU-learning is the problem of learning a binary classifier from positive and unlabeled data \citep{de1999positive, bekker2020learning}. The formulation of PU as one-sided label noise contrasts from general bidirectional label noise \citep{frenay2013classification}. PU learning has been studied under various structural formulations (e.g., case-control versus one-sample scenarios) \citep{niu2016theoretical} and depends heavily on specific identifiability and labeling mechanisms, such as the Selected Completely At Random (SCAR) assumption or allowing forms of selection bias \citep{bekker2018learning, kato2019learning}. Various methods have been proposed, but the approach most relevant to our work is the unbiased empirical risk approach \citep{elkan2008learning, du2015convex, kiryo2017positive}, which requires an estimation of the positive class prior to correct for risk estimation in the unlabeled set containing a hidden mixture of both true positives and true negatives \citep{du2014class, jain2016nonparametric, garg2021mixture}. We detail the formulation and assumptions used in our work in Section~\ref{sec:pu_setup}.

\textbf{Training Language Models under Noisy Feedback.} Transitioning to LLMs necessitates a redefinition of what constitutes a label, as labels are often token sequences such as responses, reasoning chains, or preference feedback. To mitigate noisy responses in SFT, one approach, RobustFT \citep{luo2024robustft}, detects and denoises uncertain responses using an ``Expert" LLM, using clean samples as demonstrations. For noise in preference data characterized by symmetric preference label flips, methods have been proposed to denoise preference feedback or use noise-aware optimization methods \citep{liang2024ropo, kong2024perplexity, chowdhury2024provably}, with theoretical guarantees \citep{im2025well}. We contrast this with our novel framing of preference noise as a structured, one-sided missing-positive problem.
% Our work models omission noise using the PU-learning framework to create synthetic data. Preference labels capture the ambiguity over preferring pathology exclusion, combined with a DPO framework for optimization.

\textbf{Offline Synthetic Preference Data Generation.} Diverse offline preference data collection has been an effective approach for preference optimization \citep{xiong2023iterative, tajwar2024preference}, and can be achieved through synthetic data generation \citep{wang2023self, dong2024self} and manual editing \citep{gao2024aligning}. For report generation, EditGRPO \citep{zhang2025editgrpo} corrects on-policy samples during training. 

\section{Background and Problem Setup}
Let $X \in \mathcal{X}$ denote an input image (e.g., a chest X-ray), and let $Y = \{Y_1, \dots, Y_d\}$ denote a set of $d$ target findings corresponding to a multi-task learning scenario. We consider each finding as a binary random variable $Y \in \mathcal Y=\{-1,+1\}$. Random variables are denoted in uppercase, and their realizations in lowercase. For notational simplicity, we consider the problem setup for diagnosing a single, specific clinical finding, where the underlying presence or absence of the finding is denoted by $Y$. We extend our proposed method to a multi-task setting with $d$ pathology subtypes defined by granular attributes such as anatomical location and severity in Section~\ref{sec:general_alg}.

\subsection{Positive-Unlabeled Data}
\label{sec:pu_setup}
To address omission noise, we frame our problem using the
Positive-Unlabeled (PU) setting, specifically under the \emph{one-sample} scenario~\citep{elkan2008learning, coudray2023risk}. In this scenario, we observe a dataset $\mathcal{D} = \{(x_i, s_i)\}_{i=1}^{n}$ where samples $x_{i}$ are
drawn i.i.d.\ from the marginal distribution $\mathbb{P}_X$ and
$s_i \in \{{+}1, {-}1\}$ indicates whether sample $x_i$ has been
\emph{labeled} as positive. Crucially, $s_i = -{1}$ does not imply
$x_i$ is a true negative, it is simply unlabeled. We denote the
class-conditional distributions as
$\mathbb{P}_p = \mathbb{P}(X \mid Y = {+}1)$ and
$\mathbb{P}_n = \mathbb{P}(X \mid Y = {-}1)$, so the marginal
decomposes as
$\mathbb{P}_X = \alpha\,\mathbb{P}_p + (1 - \alpha)\,\mathbb{P}_n$,
where $\alpha = \mathbb{P}(Y = {+}1)$ is the true positive class
prior. 

Under the \emph{selected-completely-at-random} (SCAR) assumption, the labeling mechanism satisfies
$\mathbb{P}(S{=}{+}1 \mid X, Y{=}{+}1) = \mathbb{P}(S {=} {+}1 \mid Y {=} {+}1)= c$ for some constant
$c \in (0, 1]$, i.e., the likelihood of labeling a case is independent of $X$. In other words, there is no selection bias in the labeling of positive cases: $\mathbb{P}(X \mid Y{=}{+}1, S{=}{+}1) = \mathbb{P}(X \mid Y{=}{+}1, S{=}{-}1)$ \citep{elkan2008learning}. The labeled positives
$\mathcal{D}_p = \{x_i \in \mathcal{D} : s_i = 1\}$ are therefore
an i.i.d.\ sample from $\mathbb{P}_p$, while the full dataset
$\mathcal{D}$ is a sample from the marginal $\mathbb{P}_X$.

% \begin{assumption}[Selected-At-Random (SAR)]
%     The probability of a true positive finding being selected to be labeled is the propensity score $e(X) = \mathbb{P}(S={+}1 \mid X, Y={+}1)$
% \end{assumption}

\paragraph{Research Question} We now consider a setting where we have
access to an observational dataset
$\mathcal{D}_{\text{obs}} = \{(x_i, \tilde{a}_i, s_i)\}_{i=1}^{n}$ of
images ($x_i$), reports ($\tilde{a}_i$), extracted label indicators ($s_i$) corresponding to samples from the joint distribution $(X, \tilde A, S) \sim \mathbb{P}_{obs}$. Each report $\tilde{a}_i$ may
describe multiple clinical findings $j\in\{1,...,d\}$. Each finding $j$ is
associated with a latent ground truth $y_{ij} \in \{+1, -1\}$
indicating its true presence or absence in image $x_i$. However, the report only partially reveals these labels: let $s_{ij}$ be the
labeling indicator denoting whether observation $j$ is explicitly
mentioned in report $\tilde{a}_i$. When $s_{ij} = {+}1$, we assume that it is present $y_{ij} = {+}1$ (assuming no false-positives and label extraction error); however, $s_{ij} = {-}1$ does not imply
$y_{ij} = {-}1$. The true label may be unknown as the finding may be present but unreported.

A positive-unlabeled structure is induced at the level of individual findings: explicitly mentioned findings correspond to observed positives, while unmentioned findings are unlabeled. We treat each finding as a conditionally independent PU learning problem with a per-finding PU structure.

\begin{tcolorbox}[
    colback=gray!5,
    colframe=gray!75,
    title=Research Question,
    fonttitle=\bfseries,
    boxsep=4pt,
    left=6pt,
    right=6pt,
    top=6pt,
    bottom=6pt
]
How can preference optimization be corrected when positive visual findings are systematically unlabeled in retrospective reports?
\end{tcolorbox}

\subsection{Direct Preference Optimization} 
\label{sec:training_alg}
To contextualize our approach, we first establish the standard formulation of Direct Preference Optimization (called DPO) \citep{rafailov2023direct}.

We denote $\pi_\phi$ as a vision-language model (VLM) parameterized by $\phi$, which takes in an input image $X \in \mathcal X$ (in practice with an instruction), and outputs a discrete probability distribution $\pi_\phi(\cdot \mid x)$ over the vocabulary space $\mathcal V$. We denote $A \in \mathcal A$ as the generated model response, and $\pi_\phi(a \mid x)$ as the VLM's probability of generating the output response given input $x$. Preference optimization requires a dataset of pairwise responses, defined as follows:

\begin{definition}[Preference data]
Consider two responses $a_w, a_l$ for an input prompt $x$. We denote $a_w \succ a_l$ if $a_w$ is preferred over $a_l$. We call $a_w$ the preferred response and $a_l$ the non-preferred response. Each triplet $(x, a_w, a_l)$ is referred to as a preference. Furthermore, the empirical dataset $\mathcal{D} = \{(x_i, a_{w,i}, a_{l,i})\}_{i=1}^N$ consists of $N$ such triplets sampled from a preference distribution.
\end{definition}

We consider the objective corresponding to the generalized form of preference optimization:
\begin{align*}
    \mathbb{E}_{(x, a_w, a_l) \sim \mathcal{D}} \left[ f \left( \beta \left( \log \frac{\pi_\phi(a_w \mid x)}{\pi_{\text{ref}}(a_w \mid x)} - \log \frac{\pi_\phi(a_l \mid x)}{\pi_{\text{ref}}(a_l \mid x)} \right) \right) \right],
\end{align*} 
where $\beta > 0$ is a scaling parameter, $\pi_{\text{ref}}$ is the base model (for example, initialized by SFT), and the function $f: \mathbb{R} \to \mathbb{R}$ maps the preference margin to a loss value. Instantiating $f$ with the negative log-sigmoid function, $f(z) = -\log \sigma(z)$, recovers the DPO formulation.

The above objective can also be interpreted as maximizing a latent reward model $r_\phi(x, a) = \beta \log \pi_\phi(a \mid x) / \pi_{\text{ref}}(a\mid x)$ \citep{rafailov2023direct}. By defining the preference score difference as $z = r_\phi(x, a_w) - r_\phi(x, a_l)$, the objective effectively maximizes the latent reward for the preferred response $a_w$ while minimizing it for the dispreferred response $a_l$. Under the Bradley-Terry model \citep{bradley1952rank}, this margin $z$ yields the predicted probability that $a_w$ is preferred over $a_l$: $\mathbb{P}_\phi(a_w \succ a_l \mid x) = \sigma(z)$.

To optimize the policy, the standard approach is to perform Maximum Likelihood Estimation (MLE) over a noise-free dataset $\mathcal{D}$. This corresponds to minimizing the expected negative log-likelihood of the true preferences with respect to the model parameters $\phi$:
$$ \mathcal{L}_{\text{MLE}}(\phi) = \mathbb{E}_{(x, a_w, a_l) \sim \mathcal{D}} \left[ -\log \mathbb{P}_{\phi}(a_w \succ a_l \mid x) \right] = \mathbb{E}_{(x, a_w, a_l) \sim \mathcal{D}} \left[ -\log \sigma(z) \right]$$

We consider a variant of DPO where $a_w$ and $a_l$ are constructed from sampling on-policy responses from $\pi_\phi$ \citep{tajwar2024preference}. However, a naive approach assumes that the preference labels in $\mathcal{D}$ are ground-truth: $a_w$ is genuinely preferred over $a_l$ for every pair. This assumption is violated due to omission noise. To bridge this gap, we introduce a framework that formalizes DPO for contrastive clinical comparisons and replacing the observed preference-label assumption with a PU loss estimator.

\section{Method}
\label{sec:method}
In Section~\ref{sec:pu-dpo}, we first formalize how preference pairs are constructed via contrastive edits to a base response, producing variants that include or omit a
clinical finding. Then we incorporate positive-unlabeled learning into the objective to
correct for the systematic mislabeling that arises when absent mentions in reports do not indicate absent findings. Finally, in Section~\ref{sec:general_alg}, we detail how to incorporate these ideas into a full general post-training algorithm.

\subsection{\method: Positive-Unlabeled Direct Preference Optimization using Contrastive Edits}
\label{sec:pu-dpo}
\paragraph{Contrastive Edits} We drop the sample index $i$ for ease of exposition and propose our method for a single finding $j$. Given a base response $a$ to prompt $S$ sampled from $\pi_{\text{ref}}$, we construct a contrastive pair via an edit function $g: \mathcal{A} \times \mathcal{Y} \rightarrow \mathcal{A}$, producing $a_j^+ = g(a, +1)$ and $a_j^- = g(a, -1)$, where $a_j^+$ includes the finding $j$ and $a_j^-$ omits it. We denote $\mathbb{P}$ as the joint distribution over $(X, A_j^+, A_j^-, Y_j)$.

Recall that $\mathbb{P}_\phi(a_j^+ \succ a_j^- \mid x)$ is the probability that the response containing the finding $j$ is preferred. The ground truth preference is determined by $Y_j$: when $Y_j = {+}1$, $a_j^+$ should be preferred (i.e., $\mathbb{P}_\phi(a_j^+ \succ a_j^- \mid x)$ should be high), and when $Y_j = {-}1$, $a_j^-$ should be preferred (i.e., $\mathbb{P}_\phi(a_j^+ \succ a_j^- \mid x)$ should be low). The standard log-loss for noiseless positive-negative learning is then given by
\begin{align}
\label{loss:standard}
    R_j(\phi) &= \mathbb{E}_{\mathbb{P}}\left[-\mathbb{I}_{Y_j={+}1}\log \mathbb{P}_\phi(a_j^+ \succ a_j^- \mid x) - \mathbb{I}_{Y_j={-}1}\log (1 - \mathbb{P}_\phi(a_j^+ \succ a_j^- \mid x))\right] \\
    &= \alpha R_{p,j}^+(\phi) + (1-\alpha)R_{n,j}^-(\phi)
\end{align}
where the decomposition is the weighted sum of the expected log loss for true positives and negatives: $R_{p,j}^+(\phi) = \mathbb{E}_{\mathbb P_p}[-\log \mathbb{P}_\phi(a_j^+ \succ a_j^- \mid x)]$ and $R_{n,j}^-(\phi) = \mathbb{E}_{\mathbb P_n}[-\log (1-\mathbb{P}_\phi(a_j^+ \succ a_j^- \mid x))]$. 

\paragraph{Learning From Positive Unlabeled Data} 
\label{sec:pu-dpo}
This framing allows us to model Positive-Unlabeled (PU) data as a special case of preference noise, where we model the dominant error mode as one-sided omissions of true positives. This allows us to directly integrate a PU learning approach into the preference optimization objective. 

In the PU setting, we cannot directly compute Equation~\eqref{loss:standard} and $R_n^-(\phi)$ is unknown. We can rewrite the log-loss for \method as follows
\begin{align}
\label{loss:pu}
    R_j(\phi) = \alpha R_{p,j}^+(\phi) - \alpha R_{p,j}^-(\phi) + R_{X,j}^-(\phi)
\end{align}
where $R_{X,j}^-(\phi) = \mathbb{E}_{\mathbb{P}_X}[-\log \mathbb{P}_\phi(a_j^- \succ a_j^+ \mid x)]$ is the expected loss for preferring absence $a_j^-$, taken over the full marginal distribution of images. Because this distribution mixes true negatives with all positives (both labeled and hidden), $R_{X,j}^-(\phi)$ is a contaminated estimate of the expected loss for preferring absence $R_{n,j}^-(\phi)$. The PU decomposition resolves this by subtracting off the contamination, where the term $\alpha R_{p,j}^-(\phi)$ captures the expected loss contribution from positives (both labeled and hidden) being incorrectly pushed toward absence, weighted by the class prior $\alpha$. Removing this term from $R_{X,j}^-(\phi)$ leaves an unbiased estimate of $(1-\alpha)R_{n,j}^-(\phi)$, recovering the noiseless expected preference loss using only quantities we can compute from the observed data.

We show that the estimator in Equation~\eqref{loss:pu} is an unbiased estimate of Equation~\eqref{loss:standard} in Appendix~\ref{app:pu_unbiased}. Gradient derivations to enable learning are provided in Appendix~\ref{app:gradient}, showing that the oracle, noise-free gradient is reconstructed in expectation under SCAR and correct specification of $\alpha$.

\subsection{Bridging the Gap: A Multi-task Loss for Granular Pathology Detection}
\label{sec:general_alg}
We now describe how \method is integrated into a complete
post-training pipeline for report generation. The key design
choice is the granularity of the observation variable $Y$. Rather
than treating each attribute independently, we define $Y_j$ as a
disease subtype: a conjunction of finding, location, and clinical
characteristics. For example, $Y_j$ may represent ``small
left-sided pleural effusion'' or ``severe right basilar
atelectasis.'' Each subtype constitutes a separate PU problem
$(S_j, Y_j)$ with its own class prior $\alpha_j$ and labeling
rate, and we adopt a weaker multi-SCAR assumption to model heterogeneity in labeling rates: SCAR holds within each subtype, but labeling rate may vary across subtypes (Section~\ref{sec:pu_setup}). Therefore, the final multi-task expected preference loss is given by
\begin{equation}
\label{eq:multi-risk}
    R^{\text{multi}}(\phi) = \mathbb{E}_{j \sim \text{Unif}(1,d)} \left[ \alpha_j R_{p,j}^+(\phi) - \alpha_j R_{p,j}^-(\phi) + R_{X,j}^-(\phi) \right]
\end{equation}

\begin{remark}
We note that uniform sampling of $j$ leads to a shift in the joint distribution from the original dataset distribution $\mathbb{P}$ to a uniform target distribution over diagnostic tasks $\mathbb{P}_{\text{tar}}$. Additionally, under PU assumptions and sufficient support over edited responses, minimizing the multi-task loss in \eqref{eq:multi-risk} recovers the optimal marginal preference for each subtype under $\mathbb{P}_{\text{tar}}$. While this provides a principled local training signal, we note that globally optimal report generation is not guaranteed without additional assumptions on the joint reward.
\end{remark}
We provide theoretical intuition and additional discussion of these design choices in Appendix~\ref{app:theorem-prroof}. Building on this analysis, we now describe the resulting training algorithm at a high level; full implementation details for each step are deferred to Appendix~\ref{app:pu-training-protocol}.

\textbf{Preference Dataset Construction.} Starting from an SFT-initialized model $\pi_{\text{ref}}$, we sample a base response $a \sim \pi_{\text{ref}}(\cdot \mid x)$ for each input image $x$. We then draw a target subtype $j \sim \mathrm{Unif}(1, d)$ and pass $a$ to a separate editor LLM (GPT-OSS 120B), which produces a contrastive pair $(a_j^+, a_j^-)$: $a_j^+$ explicitly mentions subtype $Y_j$ while $a_j^-$ omits it. Editing proceeds in two stages: the editor first observes the reference report and revises $a$ to correct any inconsistencies, then conditions on the target subtype $Y_j$ to produce the contrastive edit. Editing prompts are provided in Appendix~\ref{app:prompts}, and we provide a study to audit these responses with a radiologist in Appendix~\ref{app:audit}. To ensure that edits remain on-policy with respect to $\pi_{\text{ref}}$, we discard any pair for which the token length of $a_j^+$ or $a_j^-$ differs from $a$ by more than $\tau$ tokens.

\textbf{Training}. At each training step, we sample a finding $j$  uniformly at random, and construct a batch consisting entirely of preference pairs comparing $a_j^+$ and $a_j^-$. This ensures that the inner loss estimate of Eq.~\eqref{eq:multi-risk} obtained in a training batch is computed over a homogeneous comparison axis (i.e. comparing the presence vs. absence for a single clinical finding) with a shared class prior $\alpha_j$. For a batch size of $n_s$ with $n_p$ positive samples, we can construct empirical estimates of the loss by computing individual terms in Eq.~\eqref{eq:multi-risk} such as $\hat R_{p,j}^+(\phi) = \frac{1}{n_p}\sum_{i=1}^{n_p}-\log \hat{\mathbb P}_\phi(a_j^+ \succ a_j^- \mid x_i)$. To ensure numerical stability, batches are constructed with stratified sampling guaranteeing $n_p > 0$ labeled positives.
 
\section{Experiments}
\label{sec:experiments}

\subsection{CheXpert Radiologist Study: Omission Noise in Retrospective Annotations}
\label{sec:radiologist_study}
To assess whether the training data contains omission noise, we conducted an independent review with two board-certified radiologists (Details deferred to Appendix~\ref{app:radio_study}). 
The study was approved by Columbia IRB (AAAV6261). 
Each radiologist reviewed a set of chest radiographs from the CheXpert dataset and labeled three key clinical findings. We derived a reference label for each finding by taking the disjunction of the two annotations: a finding was treated as potentially present if \emph{either} radiologist marked it as definitely or probably present. Comparing these reference labels against the original CheXpert reports (collectively from findings and impression sections), we measured how often a finding that may be judged as present by the radiologists was not mentioned in the original report. 

We define the omission rate as the frequency at which a finding is not explicitly mentioned in the original reports (and therefore our training data), despite being visually identified by at least one radiologist. Conversely, the hidden positive rate measures the likelihood that a finding is present (as determined by the radiologists) when it was omitted from the dataset's original report. Table~\ref{tab:study} demonstrates both the omission rate and the hidden positive rates. Notably, the reference reports can omit the finding in up to 43.8\% of cases, as seen with cardiomegaly. These results demonstrate that "unlabeled" data often contains visual evidence of clinical conditions that are not present in the original reports, strongly motivating our PU learning approach.

\subsection{Semi-synthetic Evaluation} We first validate that the PU correction in \method recovers fully supervised performance in the presence of positive label omissions. As a proof of concept, we focus on detecting two clinically significant findings from chest X-rays: pleural effusion and cardiomegaly. 

\textbf{Datasets} For each pathology, we construct a clean, fully labeled dataset from CheXpert by sampling images with confirmed positive and negative annotations. Because explicit true negative labels (indicating the definitive absence of a finding) exist for a subset of CheXpert cases, this setup provides an ideal gold standard for rigorously evaluating noise robustness. We create semi-synthetic training datasets by removing findings in true positive reports with different noise rates by taking confirmed positive cases and removing their mentions at varying rates $c \in \{0.1, 0.2, 0.3\}$, following a SCAR mechanism. This produces datasets with increasing proportions of hidden positives. 

\textbf{Models} For each noise rate, we finetune three models from the base Medgemma model \citep{sellergren2025medgemma}: (1) a DPO ablation that treats all unlabeled cases as true negatives, (2) Dr-DPO, a state-of-the-art approach for robust DPO under preference noise \citep{wu2024towards}, and (3) \method, which applies the PU correction. For our semi-synthetic experiments, we provide \method with the ground-truth class prior $\alpha = 0.5$, achieved by resampling the dataset so that positive and negative examples are balanced. We study robustness to $\alpha$ misspecification under this controlled setting (Appendix Table~\ref{tab:results-alpha-misspecification}). Each method is also trained on the initial fully-labeled noiseless dataset (i.e., $c = 0$), which serves as an oracle. We evaluate the models' pathology-detection performance on a held-out set of images and their corresponding ground-truth labels.

\begin{figure}[!h]
    \centering
    \includegraphics[width=\linewidth]{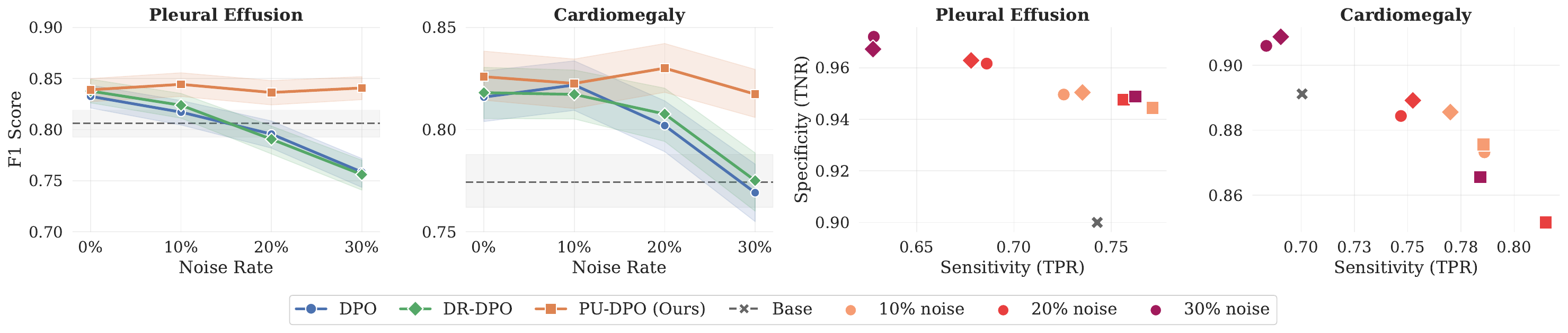}
    \caption{Our PU learning approach maintains a high F1 score as the noise rate increases. Sensitivity decreases significantly for naive approaches. $n_{test}=5000$ per task with 95\% CI via bootstrap}
    \label{fig:tprandpareto}
\end{figure}

\textbf{\method is an effective post-training mechanism.} We first note that preference optimization using contrastive edits is an effective strategy. Compared to the out-of-the-box base model performance (base) in Table~\ref{tab:results-pleural-effusion} and~\ref{tab:results-cardiomegaly} in Appendix~\ref{sec:synth_results}, finetuning improves sensitivity and specificity for all approaches at $0\%$ noise rate. We also examine the learned preference probabilities using log loss and Brier score using the ground truth disease label as the target in  Table~\ref{tab:results-uncertainty}  (Appendix~\ref{sec:synth_results}). The table also compares to SFT with our PU learning correction, serving as an ablation for using contrastive signals (and the explicit negative gradient with DPO) as opposed to only maximizing the likelihood of the preferred response. The results demonstrate that DPO is beneficial for learning to separate candidate presence vs. absence responses. However, we note in that Appendix Fig.~\ref{fig:calibration_pleural} that DPO leads to overconfidence and inferior calibration with respect to the distribution of pathologies in the dataset.

\textbf{PU correction recovers hidden positives during training.}
We evaluate each method by extracting predicted labels from the generated reports and computing the F1 score, sensitivity (recall of true positives), and specificity (recall of true negatives). Fig.~\ref{fig:tprandpareto} and appendix Table~\ref{tab:results-pleural-effusion} compare \method with naive labels against \method across increasing hidden positive rates. As the flip rate increases, naive DPO exhibits a monotonic decline in sensitivity, indicating that the model increasingly fails to mention findings that are truly present. This is an expected failure mode since hidden positives in the unlabeled set are treated as true negatives, teaching the model that omitting the finding is preferred. \method mitigates this by correcting for
the contaminated preference signal, maintaining sensitivity at 0.76-77 even at the highest noise rate while preserving comparable specificity (not introducing false positives/hallucinations). The overall F1 score reflects this: \method remains close to the oracle across all noise levels, whereas naive DPO degrades significantly.

One might expect that robustness to preference noise could handle the hidden positive problem. However, comparing Dr-DPO to \method in Table~\ref{tab:results-pleural-effusion}, we find that Dr-DPO presents low sensitivity at higher noise rates. Two factors explain this phenomenon. First, \method explicitly incorporates the class prior $\alpha$, which directly quantifies the expected contamination in the unlabeled set and allows the loss to correct for it. Dr-DPO has no access to this information. Second, Dr-DPO is motivated by distributionally robust optimization and relies on the log-likelihoods of the trained model to downweight the gradients of high-loss samples, under the assumption that systematically mislabeled pairs will incur disproportionately high loss. However, if the trained model does not inherently separate out preferences, then reweighting fails to identify and suppress them.

\subsection{Real-world CXR Report Generation} We evaluate \method on two standard, large-scale medical imaging benchmarks: MIMIC-CXR and CheXpert, focusing on three conditions (pleural effusion, cardiomegaly, and atelectasis) that are known to exhibit omissions \citep{majkowska2020chest}.

\textbf{Datasets} For our real-world evaluation, we test our approach using two distinct settings: (1) We directly assess whether real hidden positive samples were successfully recovered using a radiologist-adjudicated test set of image-label pairs (n=500). Available specifically for the CheXpert dataset, this adjudicated set contains ground-truth labels established through majority voting by five radiologists, alongside the CheXbert labels extracted from the original reports \citep{saporta2022benchmarking}. (2) We evaluate diagnostic accuracy and overall report quality using a larger test dataset constructed from both the original CheXpert (n=45764) and MIMIC-CXR (n=42643) test split reports, using known positives and negatives. 

\textbf{Models} 
We evaluate the robustness and generalizability of our multi-task approach. To ensure our method is agnostic to pretraining and architectural differences, we conduct our analyses using three distinct VLM backbones for our SFT-initialized reference policy ($\pi_{\text{ref}}$): LLaVA-Rad \citep{zambrano2025clinically}, MedGemma \citep{sellergren2025medgemma}, and NV-reason (NVIDIA-reason) \citep{myronenko2025reasoning}. We selected these models because they offer a diverse range of initial capabilities and output strategies, specifically allowing us to verify that our approach remains compatible with chain-of-thought reasoning. 

We compare our approach against our own ablations and established baselines from the literature: (1) standard Supervised Fine-Tuning (SFT) on the target dataset, as well as (2) a widely-adopted, GRPO-based approach \citep{gundersen2026radvlm} that utilizes RadCliQ process rewards—a composite metric scoring text overlap, pathology-level F1, and granular concept F1 via Radgraph. For our approach, we apply multi-task \method training using seven pathology subtypes: small left, right, and bilateral pleural effusions; mild left, right, and bilateral atelectasis; and mild cardiomegaly. We also compare against the same multi-task \method without the PU correction using the same preference dataset (DPO). The subtypes used were curated in collaboration with a board-certified radiologist, and aim to define pathology-specific subtypes defined by attributes (e.g., size, spatial location, and severity) that may influence labeling difficulty. As a motivating example, we demonstrate in Appendix Fig.~\ref{fig:tpr_attributes} that Medgemma model predictions for these subtypes have low detection rates with respect to CheXpert-derived labels. This effect is obscured when evaluation is limited to broad pathology categories.  Additional details for the dataset, training, and evaluation are provided in Appendix~\ref{app:training-details}.

\textbf{$\alpha$-Estimation.} Finally, we consider two methods of estimating the class prior $\alpha$ in our real-world experiments. For a purely data-driven approach, we estimate $\alpha_{\text{BBE}}$ using Best-Bin Estimation (BBE) (see Appendix~\ref{app:bbe}) \citep{garg2021mixture}, training a LightGBM positive-vs-unlabeled classifier for each subtype. Inputs are image embeddings extracted from the frozen MedGemma vision encoder. Second, we generate an estimate $\alpha_{\text{adjudicated}}$ that has access to the oracle using the adjudicated test set, computing the empirical frequency of hidden positives to obtain pathology-level estimates of $\mathbb{P}(Y_{\text{pathology}} = {+}1 \mid S={-}1)$, and multiplying with $\mathbb{P}(Y_{\text{subtype}} \mid Y_{\text{pathology}}={+}1)$ obtained via empirical frequencies in the training dataset. These estimates for real-world datasets are provided in Appendix~\ref{app:alpha_estimation_results}.

\begin{figure}[htbp]
    \centering
    \includegraphics[width=\linewidth]{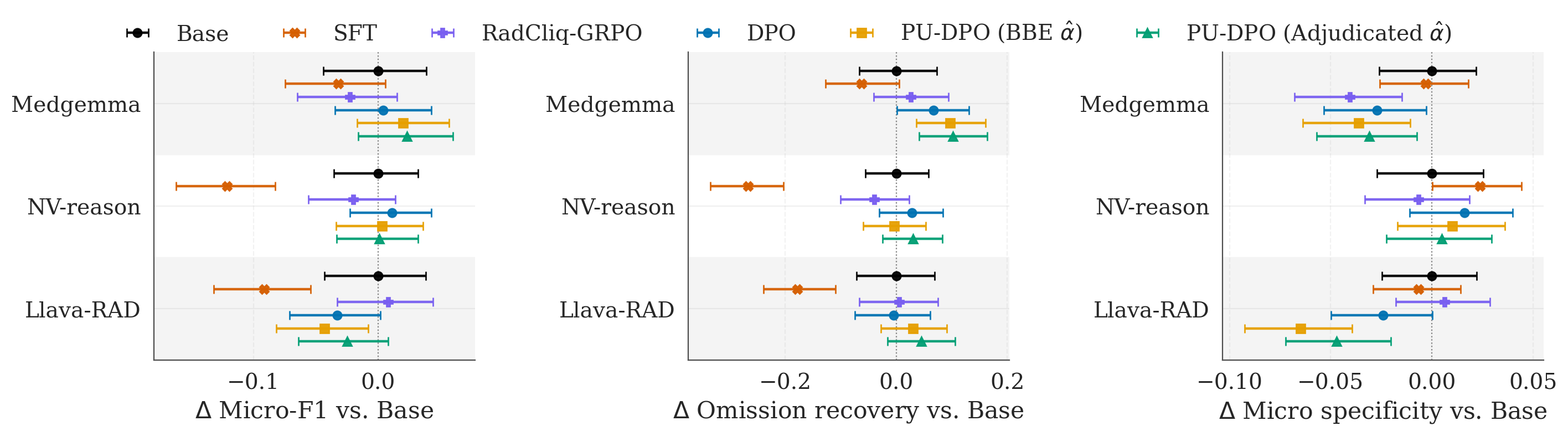}
    \caption{Relative improvement in Micro-F1, omission recovery, and specificity compared to the base model for the three findings (pleural effusion, atelectasis, cardiomegaly), for the adjudicated test set from CheXpert.}
    \label{fig:dot-plot-adjudicated}
    \vspace{-2pt}
\end{figure}

\textbf{Omission recovery in adjudicated test set.} On the adjudicated test set we compute omission recovery, the rate at which a model reports a finding that the original clinical report omitted but the adjudicating radiologists confirmed as present (i.e. recovery of hidden positives). The evaluation on Figure~\ref{fig:dot-plot-adjudicated} shows evidence that \method achieves the intended effect of improving omission recovery over DPO. Because PU-DPO is exactly our DPO baseline plus a 
$\alpha$-scaled PU loss correction term, the choice of $\alpha$ effectively acts as a calibration factor. It dictates the extent to which the VLM shifts its internal diagnostic threshold, controlling the trade-off between sensitivity and specificity. Our underlying DPO approach aims to teach the model to better capture pathology signals. PU-DPO builds on this as a well-specified $\alpha$ calibrates the model to maintain sensitivity when faced with hidden positives in training, without sacrificing the benefits of DPO. However, given the sample size we cannot yet conclude with statistical certainty that our finetuned models yield consistent improvements on all metrics across a broader clinical deployment distribution.

\textbf{Robustness to SCAR violation and $\alpha$ misspecification.} Using the adjudicated test set, we evaluate our method's robustness against violations of the SCAR assumption and the misspecification of $\alpha$. First, a kernel MMD test reveals that cardiomegaly exhibits statistically significant selection bias (violating SCAR), whereas atelectasis and pleural effusion maintain indistinguishable distributions between labeled and hidden positives (Table~\ref{tab:scar_check}). As a result, pathology-specific metrics (Appendix Fig.~\ref{fig:pathologywise_medgemma}, \ref{fig:pathologywise_nvreason}, and \ref{fig:pathologywise_llava}) show that the PU loss correction is ineffective for cardiomegaly, with omission recovery failing to improve over the naive DPO baseline. This suggests that the real-world efficacy of \method may depend on the SCAR assumption holding true. Second, our sensitivity analysis (Fig.~\ref{fig:alpha_medgemma}, \ref{fig:alpha_nvreason}, and \ref{fig:alpha_llava}) demonstrates how $\alpha$ effectively controls the sensitivity-specificity tradeoff. When training data contains hidden positives, conservatively underestimating $\alpha$ serves as a safe strategy, although it may not yield the full performance benefits of \method.

\begin{figure}[htbp]
    \centering
    \includegraphics[width=\linewidth]{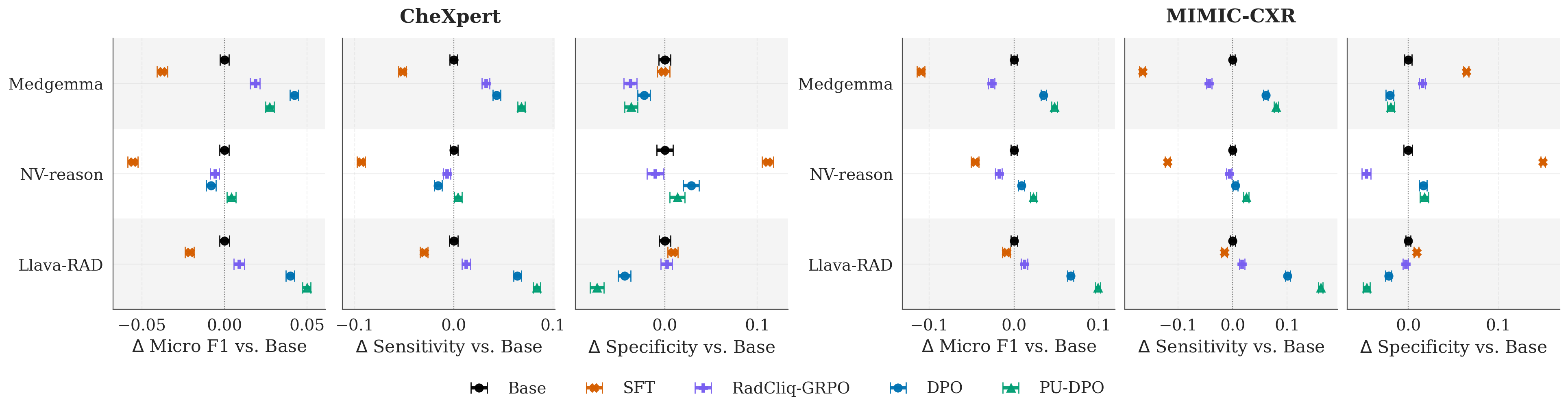}
    \caption{Relative improvement in Micro-F1, sensitivity, and specificity compared to the base model for the three findings (pleural effusion, atelectasis, cardiomegaly), across both the CheXpert (left) and MIMIC-CXR (right) dataset.}
    \label{fig:dot-plot}
    \vspace{-2pt}
\end{figure}

\textbf{Improvements in detection rate across subtypes.}
\method achieves consistently higher F1 scores across datasets and pathologies, driven by the gains in sensitivity (Fig.~\ref{fig:dot-plot}, Appendix Table~\ref{tab:filtered_f1_by_dataset_pathology_method}, ~\ref{tab:eval_metrics_chexpert_by_model_metric_method} and~\ref{tab:eval_metrics_mimic_by_model_metric_method}). The benefit of the PU correction is clearest against the standard DPO ablation, which uses the same contrastive-edit preference data but optimizes a noise-unaware objective. In contrast, offline RadCliq-GRPO yields inconsistent gains at the pathology level for two reasons: (1) it optimizes an aggregate reward that conflates multiple objectives, and (2) the training signal is limited to naive base model responses, reducing the diversity and controllability of the data. For isolating the effect of the reward mechanism and training data, we run GRPO on a fixed dataset matched in size to DPO, which provides a fair comparison but may understate its performance relative to online variants.

However, we provide two caveats: (i) some models show modest specificity degradation relative to the base model, which we attribute either to overestimation of $\alpha_j$ or to base models already calibrated for high specificity (e.g., LLaVA-RAD); and (ii) the chain-of-thought variant (NV-Reason) does not strongly benefit from \method in the CheXpert dataset, which we suspect reflects either high initial sensitivity, or limitations in the editor LLM's ability to produce contrastive edits over long reasoning traces, and a stronger editor or expert adjudication of preference pairs may lead to improvements.

\textbf{\method does not compromise report quality.} Report generation quality is also measured using natural language metrics. In table~\ref{tab:eval_metrics_chexpert_by_model_metric_method} and table~\ref{tab:eval_metrics_mimic_by_model_metric_method} of Appendix~\ref{app:real_exp_results}, we demonstrate that \method maintains robust performance when evaluating overall report quality, such as BLEU, ROUGE, RadGraph-F1 and ChexBert-F1. Details for each metric are provided in  and results across datasets and models are included in Appendix~\ref{app:evaluation}. Notably, there is minimal degradation in report quality with respect to these metrics.

\section{Discussion}
\label{sec:discussion}
In this work, we introduced the problem of omission noise in observational chest radiology reports and presented \method, a preference optimization method for training VLMs under this form of noise. We formalized report generation as a positive-unlabeled (PU) learning problem, treating unmentioned findings as unlabeled rather than negative. We proposed a principled offline preference optimization algorithm that constructs preference pairs via contrastive edits and optimizes a noise-aware objective. We then empirically validated \method across semi-synthetic and real-world settings, demonstrating robustness to omission noise and consistent improvements in detection rates across multiple pathologies. Our comparisons show that prior noise-aware DPO variants do not adequately address omission noise, and that prior task-specific approaches for radiology report generation lead to unpredictable behavior, thereby establishing the effectiveness and necessity of \method.

\paragraph{Limitations.} \textbf{\emph{(i)}} Although theoretically grounded, our method depends on accurate estimation of the class prior $\alpha$, a well-known practical challenge in PU learning. We evaluate a data-driven estimator in this work, but alternative strategies, such as eliciting ground-truth prevalence estimates from radiologists, may yield more reliable values. \textbf{\emph{(ii)}} Our analysis relies on the multi-SCAR assumption, which posits that within each pathology subtype, true positives have a uniform labeling probability. In other medical imaging settings, omission likelihood may depend on factors such as imaging view and patient context (e.g., ICU vs. outpatient). Relaxing this assumption to accommodate such dependencies is an important direction for future work. Furthermore, scaling up radiologist adjudication would strengthen the verification of our SCAR assumption and $\alpha$ estimates, while also increasing the statistical confidence of our real-world validation.
\textbf{\emph{(iii)}} Radiologist reporting typically rely on multiple views, while our experimental validation was restricted to single frontal-views for report generation.

\bibliographystyle{plainnat}
\bibliography{references}

\newpage

% \include{checklist}

%%%%%%%%%%%%%%%%%%%%%%%%%%%%%%%%%%%%%%%%%%%%%%%%%%%%%%%%%%%%

\appendix

\section{PU Learning}
\subsection{Unbiasedness}
\label{app:pu_unbiased}
\begin{proposition}[Unbiased PU Loss Estimator]
\label{prop:pu-unbiased}
Let $\mathbb{P}_p$ and $\mathbb{P}_n$ denote the class-conditional
distributions of $X$ given $Y = +1$ and $Y = -1$ respectively,
and let $\mathbb{P}_X$ denote the marginal distribution of $X$.
Let $\alpha = \mathbb{P}(Y = +1)$ be the positive class prior.
Hence, $\mathbb{P}_X = \alpha \mathbb{P}_p + (1-\alpha) \mathbb{P}_n$. Define the following per-class expected loss for a model
$\mathbb{P}_\phi(a^+ \succ a^- \mid x)$:
\begin{align*}
    R_p^+(\phi) &= \mathbb{E}_{\mathbb{P}_p}
    \big[-\log \mathbb{P}_\phi(a^+ \succ a^- \mid x)\big] \\
    R_p^-(\phi) &= \mathbb{E}_{\mathbb{P}_p}
    \big[-\log\big(1 - \mathbb{P}_\phi(a^+ \succ a^- \mid x)\big)\big] \\
    R_n^-(\phi) &= \mathbb{E}_{\mathbb{P}_n}
    \big[-\log\big(1 - \mathbb{P}_\phi(a^+ \succ a^- \mid x)\big)\big] \\
    R_X^-(\phi) &= \mathbb{E}_{\mathbb{P}_X}
    \big[-\log\big(1 - \mathbb{P}_\phi(a^+ \succ a^- \mid x)\big)\big]
\end{align*}
Then the standard supervised loss
\begin{align*}
    R(\phi) = \mathbb{E}_{\mathbb{P}}\big[
    -\mathbb{I}_{Y=+1}\log \mathbb{P}_\phi(a^+ \succ a^- \mid x)
    -\mathbb{I}_{Y=-1}\log(1 - \mathbb{P}_\phi(a^+ \succ a^- \mid x))
    \big]
\end{align*}
admits the equivalent decomposition
$$R(\phi) = \alpha\, R_p^+(\phi) - \alpha\, R_p^-(\phi) + R_X^-(\phi)$$
which depends only on expectations over $\mathbb{P}_p$ (labeled
positives) and $\mathbb{P}_X$ (the unlabeled marginal), and not
on $\mathbb{P}_n$.
\end{proposition}

\begin{proof}
By the law of total expectation, conditioning on $Y$ gives:
\begin{align*}
    R(\phi)
    &= \mathbb{P}(Y=+1)\,
    \mathbb{E}_{\mathbb{P}_p}
    \big[-\log \mathbb{P}_\phi(a^+ \succ a^- \mid x)\big] \\
    &\quad + \mathbb{P}(Y=-1)\,
    \mathbb{E}_{\mathbb{P}_n}
    \big[-\log(1 - \mathbb{P}_\phi(a^+ \succ a^- \mid x))\big] \\
    &= \alpha\, R_p^+(\phi) + (1 - \alpha)\, R_n^-(\phi)
\end{align*}
The challenge is that $R_n^-(\phi)$ requires expectations over
$\mathbb{P}_n$, which is unobserved in the PU setting. To
eliminate this dependence, we decompose the marginal distribution using the standard mixture identity:
$$\mathbb{P}_X = \alpha\, \mathbb{P}_p + (1 - \alpha)\, \mathbb{P}_n$$
Taking the expectation of $-\log(1 - \mathbb{P}_\phi(a^+ \succ
a^- \mid x))$ under $\mathbb{P}_X$ and applying linearity:
$$R_X^-(\phi) = \alpha\, R_p^-(\phi) + (1 - \alpha)\, R_n^-(\phi)$$
Solving for the unobserved term gives:
$$(1 - \alpha)\, R_n^-(\phi) = R_X^-(\phi) - \alpha\, R_p^-(\phi)$$
Substituting this into the expression for $R(\phi)$ yields:
\begin{align*}
    R(\phi)
    &= \alpha\, R_p^+(\phi) + R_X^-(\phi) - \alpha\, R_p^-(\phi) \\
    &= \alpha\, (R_p^+(\phi) - R_p^-(\phi))+ R_X^-(\phi)
\end{align*}
The right-hand side involves only expectations over $\mathbb{P}_p$
and $\mathbb{P}_X$, completing the proof.
\end{proof}

\subsection{Gradient Derivation and Analysis}
\label{app:gradient}
We derive the per-sample gradient of the \method objective. Using
the shorthand
$$z = \beta \log \frac{\pi_\phi(a^+ \mid x)}{\pi_{\text{ref}}(a^+ \mid x)}
- \beta \log \frac{\pi_\phi(a^- \mid x)}{\pi_{\text{ref}}(a^- \mid x)}$$
the predicted preference probability is
$\mathbb{P}_\phi(a^+ \succ a^- \mid x) = \sigma(z)$. The gradient
of $z$ with respect to $\phi$ is
$$\nabla_\phi z = \beta\big(\nabla_\phi \log \pi_\phi(a^+ \mid x)
- \nabla_\phi \log \pi_\phi(a^- \mid x)\big)$$
The per-sample loss components have gradients:
\begin{itemize}
    \item Positive loss ($R_p^+$):
    $\nabla_\phi[-\log \sigma(z)]
    = -(1 - \sigma(z))\,\nabla_\phi z$
    \item Flipped positive loss ($R_p^-$):
    $\nabla_\phi[-\log(1 - \sigma(z))]
    = \sigma(z)\,\nabla_\phi z$
    \item Marginal negative loss ($R_X^-$):
    $\nabla_\phi[-\log(1 - \sigma(z))]
    = \sigma(z)\,\nabla_\phi z$
\end{itemize}

The empirical PU loss over a batch of $n$ samples, with $n_p$
labeled positives and $n_u = n - n_p$ unlabeled, is
$$\widehat{R}(\phi) = \alpha\, \widehat{R}_p^+(\phi)
+ \max\!\left(0,\; \widehat{R}_X^-(\phi)
- \alpha\, \widehat{R}_p^-(\phi)\right)$$
The gradient contributions per
sample are as follows.

\textbf{Labeled positive sample $i$:} contributes to all three
terms.
\begin{align*}
    \nabla_\phi \mathcal{L}_i^{\text{pos}}
    &= \frac{\alpha}{n_p}\underbrace{\big[-(1 - \sigma(z_i))\,
    \nabla_\phi z_i\big]}_{\text{from } R_p^+}
    - \frac{\alpha}{n_p}\underbrace{\big[\sigma(z_i)\,
    \nabla_\phi z_i\big]}_{\text{from } R_p^-}
    + \frac{1}{n}\underbrace{\big[\sigma(z_i)\,
    \nabla_\phi z_i\big]}_{\text{from } R_X^-} \\
    &= \left(-\frac{\alpha}{n_p}
    + \frac{1}{n}\sigma(z_i)\right) \beta\big(\nabla_\phi \log \pi_\phi(a_i^+ \mid x_i)
- \nabla_\phi \log \pi_\phi(a_i^- \mid x_i)\big)
\end{align*}

\textbf{Unlabeled sample $i$:} contributes only to $R_X^-$.
\begin{align*}
    \nabla_\phi \mathcal{L}_i^{\text{unl}}
    = \frac{1}{n}\sigma(z_i)\,\beta\big(\nabla_\phi \log \pi_\phi(a_i^+ \mid x_i)
- \nabla_\phi \log \pi_\phi(a_i^- \mid x_i)\big)
\end{align*}
Therefore, the update via gradient descent increases the likelihood of $a^-$ for all samples, acting as a constant downward push. However, the upward correction push of the $-\alpha/n_p$ term for known true positives is what reconstructs the noiseless gradient.

\paragraph{Deriving The Batch Gradient.}
We now derive the batch gradient by summing  the gradients over the positive and unlabeled samples in a batch.
\begin{align*}
    &\sum_{i \in P}\!\left(-\frac{\alpha}{n_p}
    + \frac{\sigma(z_i)}{n}\right)\nabla_{\phi} z_i + \sum_{i \in U}\left(\frac{\sigma(z_i)}{n}\right)\nabla_{\phi} z_i
    \\&= \underbrace{-\frac{\alpha}{n_p}\!\sum_{i \in P}\nabla_{\phi} z_i}_{\text{(I)}}
    + \underbrace{\frac{1}{n}\!\sum_{i \in P}\sigma(z_i)\nabla_{\phi} z_i}_{\text{(II)}} + \underbrace{\frac{1}{n}\!\sum_{i \in U}\sigma(z_i)\nabla_{\phi} z_i}_{\text{(III)}}
    \\&=-\frac{\alpha}{n_p}\!\sum_{i \in P}\nabla_{\phi} z_i
    + \frac{1}{n}\!\sum_{i =1}^n\sigma(z_i)\nabla_{\phi} z_i
\end{align*}

\paragraph{Comparison with Standard DPO}
To understand how \method compensates for hidden positives in the
unlabeled set, we examine the batch-average gradient and contrast
it with naive DPO. We decompose the unlabeled set as
$U = U_+ \cup U_-$, where $U_+$ contains hidden positives drawn
from $\mathbb{P}_p$ and $U_-$ contains true negatives drawn from
$\mathbb{P}_n$.

Under \method, the batch gradient decomposes as
\begin{align*}
    \nabla \mathcal{L}_{\text{PU}}
    = \underbrace{\sum_{i \in P}\!\left(-\frac{\alpha}{n_p}
    + \frac{\sigma(z_i)}{n}\right)\!\nabla_{\phi} z_i}_{\text{labeled positives}}
    + \underbrace{\sum_{i \in U_+}\!\frac{\sigma(z_i)}{n}\nabla_{\phi} z_i}_{\text{hidden positives}}
    + \underbrace{\sum_{i \in U_-}\!\frac{\sigma(z_i)}{n}\nabla_{\phi} z_i}_{\text{true negatives}}
\end{align*}
For comparison, naive DPO, which treats all unlabeled samples as
true negatives, gives
\begin{align*}
    \nabla \mathcal{L}_{\text{naive}}
    = \underbrace{\sum_{i \in P}\!\left(-\frac{1 - \sigma(z_i)}{n_p}\right)\!\nabla_{\phi} z_i}_{\text{labeled positives}}
    + \underbrace{\sum_{i \in U_+}\!\frac{\sigma(z_i)}{n_u}\nabla_{\phi} z_i}_{\text{hidden positives}}
    + \underbrace{\sum_{i \in U_-}\!\frac{\sigma(z_i)}{n_u}\nabla_{\phi} z_i}_{\text{true negatives}}
\end{align*}

In both objectives, the hidden positives in $U_+$ contribute a wrong-direction gradient that pushes the model toward $a^-$, even though these samples are truly positive. Naive DPO has no mechanism to counteract this contamination, and the wrong-direction gradient accumulates over
training, leading to systematic suppression of true positives.

\paragraph{Example} Consider a scenario where the $n = 64$, $n_p = 32$, and $\alpha = 0.6$. For illustration we can use a hypothetical model confidence: $\sigma(z) = 0.9$ for all true positives, and $\sigma(z) = 0.2$ for true negatives. Recall that each sample has a gradient $\nabla_\phi z_i = \beta \big( \nabla_\phi \log \pi_\phi(a^+ \mid x_i) - \nabla_\phi \log \pi_\phi(a^- \mid x_i) \big)$. We can define $\mathbf{g}^+ = \mathbb{E}_{x \sim \mathbb{P}_p}[\nabla_\phi z]$ and $\mathbf{g}^- = \mathbb{E}_{x \sim \mathbb{P}_n}[\nabla_\phi z]$ as the average gradients for a true positive and negative sample respectively. 

We first consider the true gradient for standard DPO if there was no label noise. We can compute the expected theoretical gradient (oracle) by the weighted sum of the average positive and negative gradient:
\begin{itemize}
\item Expected Gradient for all True Positives: $\alpha \times [-(1 - \sigma(z))]\mathbf{g}^+ = 0.6 \times -(1 - 0.9)\mathbf{g}^+ = \mathbf{-0.06}\mathbf{g}^+$

\item Expected Gradient for all True Negatives: $(1 - \alpha) \times [\sigma(z)]\mathbf{g}^- = 0.4 \times (0.2)\mathbf{g}^-  = \mathbf{+0.08}\mathbf{g}^- $
\end{itemize}
Net Oracle Expected Gradient: $-0.06\mathbf{g}^+  + 0.08\mathbf{g}^- $

For the PU corrected loss with hidden positives, we have the following:

\begin{itemize}
\item Contribution from Labeled Positives ($n_p$): $32 \times \left(-\frac{0.6}{32} + \frac{0.9}{64}\right)\mathbf{g}^+ = -0.6\mathbf{g}^+ + 0.45\mathbf{g}^+ = \mathbf{-0.15}\mathbf{g}^+$

\item Contribution from Hidden Positives ($U_+$): 
$6.4 \times \left(\frac{0.9}{64}\right)\mathbf{g}^+ = \mathbf{+0.09}\mathbf{g}^+$

\item Contribution from True Negatives ($U_-$):
$25.6 \times \left(\frac{0.2}{64}\right)\mathbf{g}^- = \mathbf{+0.08}\mathbf{g}^-$
\end{itemize}

Net \method Batch Gradient: $-0.15\mathbf{g}^+ + 0.09\mathbf{g}^+ + 0.08\mathbf{g}^- = -0.06\mathbf{g}^+ +0.08\mathbf{g}^-$

% \subsection{Proof of Theorem~\ref{thm:global_optimality}}
\subsection{Discussion on Marginal Editing and Multi-task Loss}
\label{app:theorem-prroof}
In this section, we provide theoretical intuition for our multi-task preference optimization algorithm by reframing report generation as the task of jointly determining an optimal set of finding-level decisions. We note that this set-based formulation is an analytical abstraction. In practice, VLMs generate reports as token sequences and maximize the joint likelihood over those sequences. Our framing treats each clinical finding as a discrete action to expose the structure of the multi-task objective, but the underlying model still operates over token-level distributions.

\paragraph{Setup.}
Treat each report as a $d$-dimensional binary vector $a = (a_1, \dots, a_d)$ with $a_j \in \{-1, +1\}$ indicating the absence or presence of pathology $j$. Given an input image $x$, we sample a base report $a \sim \pi_{\text{ref}}(\cdot \mid x)$, draw an index $j \sim \text{Unif}(1, d)$, and form a contrastive pair $(a_j^+, a_j^-)$ by setting the $j$-th entry to $+1$ and $-1$ respectively while holding the other entries $a_{\setminus j}$ fixed. Write
$f_\phi(x, a_{\setminus j}) = \mathbb{P}_\phi(a_j^+ \succ a_j^- \mid x, a_{\setminus j})$
for the model's predicted preference probability and
$p^*(x, a_{\setminus j}) = \mathbb{P}(a_j^+ \succ a_j^- \mid x, a_{\setminus j})$
for the true preference under the Bradley-Terry model. The \method objective then takes the form
\begin{align*}
    R(\phi) = \mathbb{E}_{j \sim \text{Unif}(1,d)} \Big[
    \mathbb{E}_{X, A_{\setminus j}} \big[
    -p^* \log f_\phi - (1 - p^*) \log(1 - f_\phi)
    \big] \Big].
\end{align*}

\paragraph{The objective recovers per-finding preferences.}
The outer expectation over $j$ is a uniform average of $d$ terms, each depending on $\phi$ only through $f_\phi$ at index $j$. Minimization therefore decomposes across indices: for each fixed $j$ and context $(x, a_{\setminus j})$, the inner term is exactly the cross-entropy between the true preference $p^*$ and the model's estimate $f_\phi$. Decomposing this cross-entropy into entropy plus KL divergence,
\[
-p^* \log f_\phi - (1 - p^*) \log(1 - f_\phi) = \mathcal{H}(p^*) + D_{\mathrm{KL}}(p^* \,\|\, f_\phi),
\]
the first term does not depend on $\phi$, and the second is non-negative with equality only when $f_\phi = p^*$. The minimizer therefore matches the true conditional preference at every $(x, a_{\setminus j})$ and every index $j$. In other words, the per-finding training signal, even though it only ever compares two reports differing in a single action, is enough for the model to recover the correct local preference structure across all positions and all contexts under sufficient coverage.

\paragraph{When does per-finding preferences compose into a global optimum?}
Recovering pointwise preferences is a useful property, but the quantity we ultimately care about is the globally optimal report $a^* = \arg\max_a r^*(a \mid x)$, i.e. the set of actions that maximize the latent reward. We can argue that under strong assumptions on coverage and reward, any alternative report $a' \neq a^*$ can be reached from $a^*$ via a sequence of single-action edits, and the model's per-finding preferences correctly identify which direction is better at every step.

Concretely, fix any $a' \neq a^*$, and let $K$ be the set of indices where $a'$ and $a^*$ disagree. Construct a path $a' = a^{(0)}, a^{(1)}, \dots, a^{(|K|)} = a^*$ by flipping one disagreeing index at a time. Each consecutive pair $(a^{(t-1)}, a^{(t)})$ differs in exactly one index $k_t$, with shared context $a^{(t)}_{\setminus k_t}$. We assume that the latent reward $r^*$ admits a path of monotone single-action improvements from any $a'$ to $a^*$, i.e. the indices in $K$ can be ordered such that each step strictly increases $r^*$. Under this assumption, the Bradley-Terry model implies $\mathbb{P}(a_{k_t}^* \succ a'_{k_t} \mid x, a^{(t)}_{\setminus k_t}) > 0.5$ at every step, and the previous paragraph guarantees that the trained model recovers exactly this preference direction. Following the model's learned preferences therefore traces out a monotone path of single-action improvements from any starting report to $a^*$, identifying the global optimum without ever requiring the model to compare reports differing in more than one index simultaneously.

\paragraph{Example}
Consider a diagnostic setting with two actions
$(a_1, a_2) \in \{-1, +1\}^2$ and suppose the globally optimal
report is $a^* = (+1, -1)$. For the alternative $a' = (-1, +1)$,
which differs in both coordinates, we must assume there exists the path
$(-1, +1) \to (+1, +1) \to (+1, -1)$ which strictly increases the latent reward. Then at the first step, the
model evaluates $\mathbb{P}(a_1^+ \succ a_1^- \mid x,
a_2 = +1)$ and correctly prefers $a_1 = +1$. At the second step,
it evaluates $\mathbb{P}(a_2^+ \succ a_2^- \mid x, a_1 = +1)$
and correctly prefers $a_2 = -1$.

\subsection{Mixture Proportion Estimation with Best Bin Estimation}
\label{app:bbe} 
Our method requires an strategy to estimate $\mathbb{P}(Y={+}1)$ or equivalently $\alpha$, the underlying marginal prevalence. While our approach is general and works with any PU estimator, we use the Best-Bin-Estimation (BBE) \citep{garg2021mixture} method, which we briefly summarize here.

The goal is to estimate $\alpha$ in the formulation $\mathbb{P}_u = \alpha \mathbb{P}_p + (1 - \alpha)\mathbb{P}_n$. Given a classifier $f: \mathcal{X} \rightarrow [0, 1]$, let $q(c) = \mathbb{P}(f(X) \geq c)$ denote the true probability that a sample scores above a threshold $c \in [0, 1]$. For a given dataset of size $n_s$, this tail probability is approximated using the empirical estimator $\hat{q}(c) = \frac{1}{n_s}\sum_{i=1}^{n_s}\mathbb{I}[f(x_i) \geq c]$.

The method estimates the optimal mixture proportion by finding the threshold that minimizes the ratio of these tail probabilities between the unlabeled and positive distributions:
$$\hat{\alpha} = \min_{c \in [0, 1]} \frac{{\hat{q}}_u(c)}{{\hat{q}}_p(c)}$$ 
Mathematically, because the unlabeled distribution is a mixture, the true theoretical ratio expands to: 
$$\frac{q_u(c)}{q_p(c)} = \alpha + (1-\alpha) \frac{q_n(c)}{q_p(c)}$$
Since $q_n(c) \geq 0$ and $q_p(c) > 0$, this ratio naturally serves as an upper bound for $\alpha$. Intuitively, the minimization procedure searches for a sufficiently high threshold $c$ where the classifier perfectly isolates positive samples such that no true negatives exceed the threshold (i.e., $q_n(c) = 0$). At this critical threshold, assuming positive samples still exist ($q_p(c) > 0$), the bound becomes tight, and the calculated ratio exactly recovers the true mixture proportion $\alpha$.

% \subsection{Variational PU learning and Dr-DPO}
% \label{app:variational}

\section{Training Procedures}
\label{app:training-details}

\subsection{Datasets}
\label{app:datasets}
We evaluate our methodology on two large-scale publicly available chest X-ray datasets to show the generalizability of our approach.

\textbf{Mimic-CXR} contains 227,827 chest X-ray imaging studies from the Beth Israel Deaconess Medical Center, each paired with a structured free-text radiology report \citep{johnson2019mimic}. 

\textbf{CheXpert/ CheXpert plus} contains 224,316 chest radiographs from Stanford University labeled across 14 pathology categories using a rule-based extractor; we use the associated reports from CheXpert Plus \citep{irvin2019chexpert, chambon2024chexpert}.

A consistent 60/30/10 train/test/validation is applied on these datasets for the training and evaluation methodology described in the upcoming sections. We also filter to frontal views (PA/AP) to ensure compatibility across all pretrained models described in the next section.

\subsection{Model Backbones}
We conduct experiments across three different state of the art pretrained VLM models to cover different architectures and training methods.

\textbf{LLava-Rad} is a CXR adapted VLM that uses a CLIP-style vision encoder that is paired with a language decoder via contrastive alignment on large-scale radiology datasets \citep{zambrano2025clinically}.

\textbf{MedGemma} is a medical foundational model built on the Gemma-3 architecture with a specialized SigLIP vision encoder. The model is instruction-tuned on diverse clinical text and image data and provides a strong zero shot report generation baseline across tasks like radiology report generation \citep{sellergren2025medgemma}.

\textbf{NV-Reason} is a Qwen-2.5 VL based chain of thought augmented VLM that generates reasoning tokens prior to generating radiology reports and labels. The model uses supervised fine-tuning followed by GRPO reinforcement learning to improve model performance \citep{myronenko2025reasoning}.

\subsection{Model Hyperparameters}
We detail the hyperparameters for supervised fine-tuning, Radcliq-GRPO, and DPO training in Tables \ref{tab:sft-hparams} and  \ref{tab:consolidated-hparams} respectively.

\begin{table}[!ht]
\centering
\resizebox{\linewidth}{!}{%
\begin{tabular}{lccc}
\toprule
\textbf{Hyperparameter} & \textbf{LLaVA-Rad} & \textbf{NV-Reason-CXR} & \textbf{MedGemma} \\
\midrule
Base model & Vicuna-7B v1.5 & NV-Reason-CXR-3B & MedGemma-4B-IT \\
Vision encoder & BiomedCLIP-518px & Qwen2.5-VL & SigLIP (built-in) \\
Learning rate & $1 \times 10^{-4}$ & $1 \times 10^{-4}$ & $1 \times 10^{-4}$ \\
LR scheduler & Cosine & Cosine & Cosine \\
Warmup ratio & 0.03 & 0.03 & 0.03 \\
Weight decay & 0.0 & 0.0 & 0.0 \\
Max grad norm & 1.0 & 0.3 & 0.3 \\
Epochs & 3 & 3 & 3 \\
Effective batch size & 16 & 32 & 32 \\
Max sequence length & 2048 & Default & Default \\
LoRA rank $r$ & 64 & 16 & 16 \\
LoRA $\alpha$ & 128 & 32 & 32 \\
LoRA dropout & 0.05 & 0.05 & 0.05 \\
LoRA target modules & All Linear & Q,K,V,O proj & Q,K,V,O proj \\
Optimizer & AdamW & AdamW (fused) & AdamW (fused) \\
Attention implementation & SDPA & SDPA & Flash Attention 2 \\
Precision & BF16 + TF32 & BF16 + TF32 & BF16 + TF32 \\
\bottomrule
\end{tabular}}
\vspace{6pt}
\caption{SFT Hyperparameters by Model}
\label{tab:sft-hparams}
\end{table}

\begin{table}[!ht]
\centering
\small
\begin{tabular}{lcc}
\toprule
\textbf{Hyperparameter} & \textbf{GRPO} & \textbf{DPO} \\
\midrule
\multicolumn{3}{l}{\textit{Global Shared Parameters (All Models \& Methods)}} \\
Actor learning rate & \multicolumn{2}{c}{$1 \times 10^{-5}$} \\
Gradient clipping & \multicolumn{2}{c}{1.0} \\
Rollout samples per prompt ($n$) & \multicolumn{2}{c}{4} \\
Rollout temperature / top-$p$ & \multicolumn{2}{c}{0.7 / 0.95} \\
LoRA rank $r$ / $\alpha$ & \multicolumn{2}{c}{16 / 32} \\
LoRA target modules & \multicolumn{2}{c}{Q, K, V, O proj} \\
Mixed precision & \multicolumn{2}{c}{BF16} \\
\midrule
\multicolumn{3}{l}{\textit{Method-Specific Training Parameters}} \\
Max training steps & 300 & 200 \\
Effective batch size & 32 & 64 \\
Warmup steps & 50 & -- \\
KL penalty  & 0.01 & 0.01 \\
GRPO clip $\epsilon_{\text{low}}$ / $\epsilon_{\text{high}}$ & 0.20 / 0.28 & -- \\
Reward function & RadCliQ & -- \\
DPO $\beta$ & -- & 0.05 \\
\midrule
\multicolumn{3}{l}{\textit{Model-Specific Parameters}} \\
Max response length (LLaVA-Rad, MedGemma) & 1024 & 1024 \\
Max response length (NV-Reason-CXR) & 4096 & 4096 \\
\bottomrule
\end{tabular}
\vspace{6pt}
\caption{Consolidated Hyperparameters for GRPO and DPO across all models}
\label{tab:consolidated-hparams}
\end{table}

\FloatBarrier

\subsection{Supervised Fine-Tuning Protocol}

For all model-dataset combinations, SFT is performed by maximizing the conditional log-likelihood of the reference radiology report given the input image and instruction prompts. All three models are fine-tuned with LoRA adapters injected into the language-model attention projections; the vision encoder and the multimodal projector are kept frozen \citep{hu2022lora}. All models are fine-tuned using the AdamW optimizer with a cosine learning rate scheduler and a linear warmup over the first 3 \% of steps.

\textbf{Model Selection}: We perform a grid search over learning rates $ \in \{10^{-4},\; 10^{-5}\}$ yielding two candidate runs per model-dataset combinations. For each run, we retain the two best checkpoints with the lowest validation loss as evaluated at regular intervals. Prior work has widely adopted micro/macro-averaged F1 scores of Chexpert labels as a primary factual correctness metric and directly optimize  objectives \citep{tanno2025collaboration, myronenko2025reasoning}. We use micro-averaged CheXbert F1 for checkpoint selection due to its greater stability under class imbalance. All candidate checkpoints are thus evaluated on the validation split using the micro-averaged CheXbert F1 score as the model selection criterion. This score is computed by running the CheXbert labeler on reference and generated reports and comparing the resulting 14 labels \citep{smit2020combining}. The checkpoint achieving the highest point-estimate validation CheXbert-F1, is selected as the final SFT model. For the final evaluation, we report both micro-averaged and macro-averaged CheXbert-F1. Information from the test set is not used at any stage of model selection. All experiments are trained for a maximum of 3 epochs with gradient checkpointing and early stopping to manage memory.

\subsection{Radcliq-GRPO Baseline Protocol}

We adapted the training paradigm of RadVLM-GRPO as the basis for our implementation for applying GRPO on VLMs for Chext X-Ray report generation tasks \citep{gundersen2026radvlm}. The training setup seeks to mimic the RadVLM-GRPO framework with the same optimizer, rollout and KL-penalty configuration as the RadVLM-GRPO codebase. GRPO is applied on the base models of Llava-Rad, Medgemma and NV-Reason.

The RadVLM setup is modified to operate in an offline GRPO setup. All rollouts are generated once using the base model initialized reference policy prior to training. For each prompt, 4 rollouts are drawn from $\pi_{\text{ref}}$ and are used to compute the group-relative advantages for the GRPO update. All other hyperparameters are obtained from the RadVLM-GRPO setup. To manage computational costs, training was capped at 9000 training samples and 1000 validation samples. This design choice is deliberate: the dataset size is the same as the size in our DPO framework and ensures consistency across methods. The use of offline GRPO is motivated by the computational expense of online RadCliQ reward scoring and is used to ensure a consistent data budget with the DPO framework. However, we do acknowledge a known limitation of offline RL approaches as off-policy rollout distribution may drift from the improving policy over the course of training.

\textbf{Reward Function}: Consistent with the RadVLM-GRPO paper, which demonstrates that Radcliq achieves the most effective behavior amongst evaluated reward signals, we adopt RadCliq as the reward for report generation. RadCliq is a composite metric that aggregates BERTScore (lexical overlap), CheXbert similarity (14-label pathology classification) and RadGraph-F1 (granular clinical entity-relation overlap) \citep{yu2023evaluating}.

\textbf{Model Selection}: To maintain consistency with the SFT procedure, we apply the same checkpoint strategy: the two checkpoints with the lowest validation loss are evaluated on the validation set. The micro-averaged CheXbert F1 point estimate on the validation set is used as the model selection criterion.

\subsection{\method and Ablation Protocol}
\label{app:pu-training-protocol}

\paragraph{$\alpha$-Estimation.}
We estimate $\alpha_j$ separately for each finding subtype $j$ using the training split of each dataset. Labeled positives are training samples whose reference report explicitly mentions subtype $Y_j$; the remaining samples form the unlabeled set, which contains both true negatives and unreported positives (i.e., images where the finding is present but its subtype is not mentioned). For each sample, we extract a chest X-ray embedding using the frozen MedGemma image encoder. We then train a positive-vs-unlabeled classifier on these embeddings using LightGBM, fitting a separate classifier per subtype with 3-fold cross-fitting. The resulting held-out prediction probabilities are passed to BBE (Appendix~\ref{app:bbe}).

Note that this is the estimate of the fraction of hidden positives strictly within the unlabeled data, $P(Y=1 \mid S=-1)$. Because our setup follows the one-sample PU learning setting, this fraction does not equal the overall class prior $P(Y=1)$. Therefore, for each training batch during PU-DPO, we compute the empirical prior $\hat{\alpha}$ (estimating $P(Y=1)$) by combining the known labeled positives ($n_p$) with the estimated hidden positives from the unlabeled set ($n_u$). Specifically, we compute $\hat{\alpha} = \frac{n_p + n_u \hat{P}(Y=1 \mid S=-1)}{n}$, where $n$ is the total batch size. This batch-level $\hat{\alpha}$ is then used to scale the PU loss correction.

\textbf{Preference Dataset}: After sampling base responses from $\pi_{\text{ref}}$, we generate pairs using GPT-OSS 120B for editing the base response. We use the prompts in Appendix~\ref{app:prompts} and perform a two stage generation process to first correct the base-model response to faithfully represent the diagnosis in the reference clinician report, and then insert and remove findings from the base-model response.

We applied a length variation filter to ensure the generated responses do not deviate excessively from the original base response. Specifically, we enforce that the relative absolute difference in length between both candidate responses (chosen and rejected) and the base response is bounded by $\tau = 0.5$. Formally, a sample is kept if $\max(\Delta L_{a^+}, \Delta L_{a^-}) \leq 0.5$, where $\Delta L_{a} = \frac{||a| - |base||}{\max(1, |base|)}$. 

\paragraph{Training.}
Our training procedure deliberately deviates from the natural prevalence of pathologies in the preference dataset along two axes. First, at each training step we sample the target pathology $j$ uniformly across types rather than according to its marginal frequency. Second, within each batch we use stratified sampling to enforce a balanced 50-50 split between labeled positives and unlabeled examples. The latter is a practical necessity: the marginal prevalence of fine-grained pathology subtypes (e.g., small left-sided atelectasis) is typically 1--5\%, and unstratified batches would yield gradient estimates dominated by unlabeled examples and high-variance gradient estimates for the positive sample loss terms.

This sampling scheme introduces a distributional shift relative to the natural data distribution, which in principle requires an off-policy correction. We do not observe practical issues from this shift in our setting, in part because \method is intended as a lightweight fine-tuning step on top of an already-trained SFT base model, analogous to the standard SFT $\rightarrow$ DPO pipeline, rather than as a from-scratch training procedure. Consistent with this scope, we fine-tune the base VLM for only 200 gradient steps. Longer training runs or applications outside this short-horizon fine-tuning regime may require principled importance reweighting or related corrections, which we leave to future work.

\subsection{Prompts}\label{app:prompts}

The following prompts are used consistently across models and across protocols. Throughout this section, \texttt{< display\_name>} and \texttt{<dn>} refer to the name of the specific finding (e.g., pleural effusion).

\definecolor{promptblue}{HTML}{2B4C7E}
\definecolor{promptbg}{HTML}{FAFBFC}

\tcbset{
  promptbox/.style={
    enhanced,
    breakable,
    colback=promptbg,
    colframe=promptblue,
    boxrule=0.45pt,
    arc=1mm,
    left=1.5mm,
    right=1.5mm,
    top=3mm,
    bottom=1mm,
    before skip=8pt,
    after skip=8pt,
    fonttitle=\bfseries\small,
    coltitle=white,
    colbacktitle=promptblue,
    attach boxed title to top left={
      xshift=2mm,
      yshift=-1.5mm
    },
    boxed title style={
      sharp corners,
      boxrule=0pt,
      size=minimal,
      top=0.5mm,
      bottom=0.5mm,
      left=2mm,
      right=2mm
    }
  }
}

\setlist[itemize]{
  leftmargin=1.4em,
  itemsep=0.2em,
  topsep=0.2em,
  parsep=0pt,
  partopsep=0pt
}

\newcommand{\promptsection}[1]{%
  \vspace{0.4em}
  \noindent{\color{promptblue}\textbf{#1:}}\hspace{0.4em}
}

\newcommand{\promptinput}[1]{%
{\footnotesize\ttfamily
\begin{quote}
#1
\end{quote}
}
}

\begin{tcolorbox}[promptbox,title=Report Generation Prompt Template]
\footnotesize

Provide a description of the findings and impressions in the chest X-ray image \texttt{<image>}.

\end{tcolorbox}

% =========================================================
% CheXpert Extraction Prompt
% =========================================================

\begin{tcolorbox}[promptbox,title=CheXpert Label Extraction Prompt Template]
\footnotesize

You are extracting CheXpert labels from a chest X-ray report. For each label, output 1 if the finding is present, otherwise 0. If the finding is uncertain, output -1. If the finding is not mentioned, output 0. Only output valid JSON with keys exactly matching these labels: \texttt{<labels>}.

\promptsection{Report}
\texttt{<report\_text>}

\end{tcolorbox}

\begin{tcolorbox}[promptbox,title=Clinical Correction Prompt (1st Stage of Preference Dataset Construction)]
\footnotesize

You are reviewing a chest X-ray report for \texttt{<display\_name>}. The following is the original clinician report for this study. Use it ONLY as a factual reference to align the base response.

\promptsection{Reference Report}
\texttt{<gt\_report>}

\promptsection{Task}

Edit the base response below so its treatment of \texttt{<dn>} is consistent with the reference report.

\begin{itemize}
    \item If the reference report DESCRIBES \texttt{<dn>}: ensure the base response also describes \texttt{<dn>}, using phrasing that captures the reference's key clinical characteristics (size, laterality, severity, etc.).

    \item If the reference report DOES NOT describe \texttt{<dn>} (or explicitly states it is absent / normal / clear): DELETE any sentence, bullet, or clause in the base response that mentions \texttt{<dn>}. This includes BOTH positive mentions (e.g., `small left pleural effusion') AND negation/absence statements (e.g., `no pleural effusion'). If a sentence mentions \texttt{<dn>} alongside other findings, drop only the \texttt{<dn>} portion while keeping the rest intact (e.g., `small left pleural effusion and associated atelectasis' $\rightarrow$ `atelectasis'). Do NOT replace \texttt{<dn>} with a normal/absent descriptor --- the corrected response should contain no reference to \texttt{<dn>} at all.

    \item Leave all sentences that do not mention the finding completely unchanged, word for word.

    \item Do NOT copy sentences verbatim from the reference report; borrow only clinical characteristics.

    \item Do NOT include any temporal comparisons from the reference report (e.g., `unchanged from prior', `increased since previous study', `new compared to prior exam'). Describe findings as they appear on the current study only.
\end{itemize}

\promptsection{Rules}

\begin{itemize}
    \item Return valid JSON only (no markdown, no explanation).

    \item JSON must include the key: \texttt{response}.

    \item \texttt{response} must be the full corrected report text.
\end{itemize}

\promptsection{Base Response}
\texttt{<base\_response>}

\promptsection{Output Schema}
\texttt{<json\_schema>}

\end{tcolorbox}

\begin{tcolorbox}[promptbox,title=CheXpert Variant Generation Template (Finding Removal): 2nd Stage of Preference Dataset Generation]
\footnotesize

You are editing a chest X-ray report.

\promptsection{Task}

Remove all mention of \texttt{<dn>} from the report below so that the resulting report simply does not discuss \texttt{<dn>} at all (neither as present nor as absent). Specifically:

\begin{itemize}
    \item DELETE any sentence, bullet, or clause that mentions \texttt{<dn>} (including positive mentions like `small left pleural effusion' AND negation/absence statements like `no pleural effusion' etc.).

    \item If a sentence mentions \texttt{<dn>} alongside other findings, drop only the \texttt{<dn>} portion and keep the rest of the sentence intact (e.g., `cardiomegaly and small pleural effusion' $\rightarrow$ `cardiomegaly').

    \item Do NOT replace \texttt{<dn>} with a normal/absent descriptor --- the final report should contain no reference to \texttt{<dn>} at all.

    \item Leave all sentences that do not mention the finding completely unchanged, word for word.

    \item Maintain the overall structure and style of the report (bullets, summary line, etc.); just omit the \texttt{<dn>} content.
\end{itemize}

\promptsection{Rules}

\begin{itemize}
    \item Return valid JSON only (no markdown, no explanation).

    \item JSON must include the key: \texttt{response}.

    \item \texttt{response} must be the full edited report text.
\end{itemize}

\promptsection{Base Report}
\texttt{<corrected\_response>}

\promptsection{Output Schema}
\texttt{<json\_schema>}
\end{tcolorbox}

\begin{tcolorbox}[promptbox,title=CheXpert Variant Generation Template (Pathology Addition): 2nd Stage of Preference Dataset Generation]
\footnotesize

You are editing a chest X-ray report.

\promptsection{Task}

Modify the report below so that it is consistent with a study where \texttt{<dn>} IS present. Specifically:

\begin{itemize}
    \item If the report states that \texttt{<dn>} is absent/normal, rewrite that sentence to describe \texttt{<dn>} using EXACTLY the target attributes above.

    \item If the report does not mention \texttt{<dn>} at all, add a description of \texttt{<dn>} using EXACTLY the target attributes above, integrated into the appropriate anatomical section of the report.

    \item Leave all sentences that do not mention the finding completely unchanged, word for word.

    \item Maintain the overall structure and style of the report.
\end{itemize}

\promptsection{Rules}

\begin{itemize}
    \item Return valid JSON only (no markdown, no explanation).

    \item JSON must include the key: \texttt{response}.

    \item \texttt{response} must be the full edited report text.
\end{itemize}

\promptsection{Target attributes for the inserted \texttt{<dn>}}
\texttt{<target\_confounder\_json>}

Each characteristic category maps attributes to true/false. Attributes set to true MUST be present in the inserted description; attributes set to false MUST NOT be present. If all attributes in a category are false, do not mention that category at all.

\promptsection{Base Report}
\texttt{<corrected\_response>}

\promptsection{Output Schema}
\texttt{<json\_schema>}

\end{tcolorbox}

\subsection{Evaluation}
\label{app:evaluation}
The final evaluation uses a combination of clinical correctness metrics and linguistic metrics on the test set.

\textbf{BLEU} measures the n-gram precision between a generated report and reference reports. We report corpus-level BLEU-1 and BLEU-4 computed with SacreBLEU \citep{papineni2002bleu, post2018call}.

\textbf{ROUGE} measures the recall-oriented n-gram overlap between the generated and reference text by computing the fraction of reference n-grams covered by the prediction \citep{lin2004rouge}. We report ROUGE-1, ROUGE-2 and ROUGE-L (longest common subsequence).

\textbf{CheXbert-F1} scores are computed using a BERT-based labeler that extract 14 radiological conditions from free-text reports \citep{smit2020combining}. We compute micro and macro averaged F1 between the predicted and reference condition labels. We report results under two condition subsets: the full 14 label set and a 3 label subset (Cardiomegaly, Atelectasis, Pleural Effusion). Additionally we report the omission rate (conditions where the reference is labeled as positive/uncertain that the prediction misses and are labeled as blank), wrong value rate (conditions where the reference is labeled as positive/uncertain but the predictions are labeled as explicit negatives) and hallucination rates (conditions where the reference is labeled as negative, but the predictions are labeled as positive/uncertain).

\textbf{Radgraph F1} measures the overlap between clinical entities and their relationship using a knowledge graph \citep{jain2021radgraph}. We use the partial reward level which award partial credit for correctly identified entities where relationship labels vary. Radgraph F1 scores compute F1 over matched graph elements.

\textbf{Statistical reporting:} All metrics are calculated with 95 \% confidence intervals (1000 resamples) on the test set.

\subsection {Reproducibility}
All experiments are implemented in PyTorch and are publicly available at \url{https://github.com/reAIM-Lab/CXR-reason}. Inference for all models uses greedy decoding with a standard report generation length. We report results on a single seed and use standardized configuration files across experiments to ensure reproducibility.

\subsection{Compute Details}
\label{app:compute}
All experiments were run on a server with 4 NVIDIA H100 NVL GPUs, 2 Intel(R) Xeon(R) Platinum 8480+ CPUs (56 cores each) with 2Tb of memory.

\subsection{Runtime}
\label{app:runtime}
This section details runtime statistics for a sample model Medgemma:

\begin{table}[h]
\centering
\small
\setlength{\tabcolsep}{8pt}
\begin{tabular}{lccc}
\toprule
 \textbf{Statistics} & \textbf{SFT} & \textbf{GRPO} & \textbf{\method} \\
\midrule
Training steps & 7{,}902 & 300 & 200\\
Total wall-clock time & 22.4\,h & 8.1\,h & 2.5\,h \\
\addlinespace
Mean Time/training step & 10.2\,s & 35.0\,s & 32.3\,s \\
Median Time/training step & 11.1\,s & 31.7\,s & -- \\
\addlinespace
Validation passes & 7 & 12 & 4\\
Mean Time/validation pass & 1{,}216\,s & 906\,s & 904.8\,s \\
\bottomrule
\end{tabular}
\vspace{6pt}
\caption{Runtime Statistics for MedGemma model on 1 NVIDIA H100 NVL GPU}
\label{tab:runtime}
\end{table}

\begin{table}[ht]
\centering
\caption{Runtime breakdown for preference dataset construction (10,000 images, 4 samples/image).}
\label{tab:runtime}
\begin{tabular}{lr}
\toprule
\textbf{Pipeline Stage} & \textbf{Runtime (h)} \\
\midrule
Base sample generation (vLLM) & 1.4 \\
Verification (vLLM) & 2.6 \\
Variant Generation & 1.9 \\
\midrule
\textbf{Total} & \textbf{5.9} \\
\bottomrule
\end{tabular}
\end{table}

\newpage
\section{Radiologists annotations protocol}
\label{app:radio_study}

\paragraph{Annotation Protocol.} To validate the quality of reports in the CheXpert dataset, we enrolled four radiologists in a structured annotation effort.

\paragraph{Patient Sampling}
Patients were sampled from CheXpert and stratified by age (discretized into 20-year bins) and sex. As our initial focus was on pleural effusion, a condition frequently misdiagnosed, we further stratified on pleural effusion status as extracted from original reports using CheXpert's NLP labeler. For each stratum, we sampled 100 patients or the maximum available, yielding a final set of 332 chest X-rays.

\paragraph{Annotation Task}
Annotations were collected via Qualtrics, where the interface displayed clinical notes and different views of the X-ray. For each of the 332 images, radiologists independently assessed a predefined set of conditions using a 5-point Likert scale (definitely absent, probably absent, unsure, probably present, definitely present). When a condition was marked as present, radiologists additionally rated its severity as mild, moderate, or severe. For a random 10\% subset (~33 images), radiologists were also asked to write a full structured report, comprising separate findings and impression sections.

\paragraph{Radiologist Inclusion}
Annotations were performed independently. To ensure annotation quality, only radiologists who contributed more than 5 annotations were retained, resulting in 2 radiologists included in the final analysis.

\begin{table}[ht]
\centering
\caption{Demographic characteristics of the study population.}
\label{tab:demographics}
\begin{tabular}{lr}
\toprule
\textbf{Characteristic} & \textbf{Proportion (\%)} \\
\midrule
\textit{Sex} & \\
Male & 54.0 \\
Female & 46.0 \\
\addlinespace
\textit{Race} & \\
Other & 27.0 \\
White & 27.0 \\
Black & 23.0 \\
Asian & 23.0 \\
\addlinespace
\textit{Insurance Type} & \\
Private Insurance & 35.0 \\
Medicare & 29.0 \\
Medicaid & 18.0 \\
Unknown & 16.0 \\
Other & 2.0 \\
\bottomrule
\end{tabular}
\end{table}

\newpage
\section{Preference dataset audit}
\label{app:audit}

To assess the quality of the GPT-OSS generated preference dataset, we conducted a manual audit of the preference pairs used to train PU-DPO on the Chexpert Dataset.

\paragraph{Sampling} The audit study was performed on the preference pairs generated for the Chexpert dataset. We focused our efforts on two pathologies - pleural effusion and cardiomegaly. For each, we sampled 35 preference pairs in which it served as the target condition. We further stratified by the pair's original label status. Each sample is assigned one of three label statuses based on this extraction: \textit{labeled\_tp}, where the pathology is documented as present; \textit{tn}, where it is documented as absent; and \textit{unlabeled}, where the pathology is not mentioned at all, leaving its true presence unknown. The unlabeled stratum receives the highest sample allocation because it represents 64\% of the training dataset. Conversely, \textit{labeled\_tp} and \textit{tn} pairs exhibit minimal ground-truth ambiguity and comprise a smaller share of the dataset. However, they are oversampled relative to their natural distribution to ensure sufficient sample sizes for evaluation. Specifically, for each pathology, we construct a 10/10/15 split across \textit{labeled\_tp}, \textit{tn}, and \textit{unlabeled} strata, yielding a total of 70 preference pairs for our study.

\paragraph{Rubric} A board-certified radiologist reviewed each pair against four questions:
\begin{enumerate}
    \item \textbf{Chosen response accuracy:} Does the chosen response accurately describe the target pathology and the target attributes when checked against the original report?
    \item \textbf{Target pathology edit} Was the rejected side edit executed correctly? (The edit involves complete removal of the finding for \textit{labeled\_tp} pairs and accurate insertion of the specified attributed for \textit{unlabeled} and \textit{tn} pairs)
    \item \textbf{Preservation of unrelated attributes:} Did the edit leave all other findings in the report unchanged?
    \item \textbf{Overall usability:} Is the pair usable as clean training data overall?
\end{enumerate}

Alongside the reference report, we also provided the radiologist with the corresponding chest radiograph. While the editor never observes the image during the construction of these preference pairs, the image lets the radiologist verify that a counterfactual edit remains in-distribution. The radiologist was also instructed not to penalize a response for omitting a finding that the original report itself already omits to help isolate true editor failure. The radiologist was also provided a Notes section to help understand the reasoning behind the responses.

\paragraph{Results} Table \ref{tab:audit-rates} reports per-question rates with Wilson 95\% confidence intervals across the 70 audited pairs, broken down by label and pathology. We observe that the estimated correctness rate for the overall usability (97.1 \%) and the preservation of unrelated findings (97.1\%) score highly across every stratum.

The one stratum where clinical correctness meaningfully degrades is \textit{labeled\_tp} where only an estimated 70\% of chosen responses were judged to be clinically correct. An analysis of the radiologists' notes for these examples shows that the chosen response inherits overstated diagnostic certainty from the reference reports rather than by errors introduced during editing. This reflects a property of the source report that the editor is instructed to align with, rather than an artifact introduced by the editor. We note that the source report likely had access to multiple views (i.e. lateral views) which allowed for greater diagnostic certainty.

The \textit{tn} and the \textit{unlabeled} strata report high correctness rates across all questions. Results are also broadly consistent across pathologies: pleural effusion shows a modestly lower chosen-response correctness rate (85.7\%) than cardiomegaly (91.4\%) but the confidence intervals overlap and the difference is not statistically distinguishable.

\begin{table}[ht]
\centering
\caption{Radiologist audit of DPO preference pairs: per-question verdict counts and
the rate of \emph{yes} verdicts among answered items, with Wilson 95\% confidence
intervals. Rates are computed over answered items only; blank answers are excluded.}
\label{tab:audit-rates}
\small
\begin{tabular}{lrrrrrr}
\toprule
& \multicolumn{4}{c}{Verdicts} & \multicolumn{2}{c}{Yes rate (\%)} \\
\cmidrule(lr){2-5} \cmidrule(lr){6-7}
Question & $n$ & yes & unsure & no & Est. & 95\% CI \\
\midrule
\multicolumn{7}{l}{\emph{All reviewed pairs} ($n=70$ pairs)} \\
Chosen report clinically correct & 70 & 62 & 4 & 4 & 88.6 & [79.0, 94.1] \\
Target-pathology edit correct & 69 & 68 & 0 & 1 & 98.6 & [92.2, 99.7] \\
Unrelated findings preserved & 69 & 67 & 0 & 2 & 97.1 & [90.0, 99.2] \\
Pair overall usable & 69 & 67 & 1 & 1 & 97.1 & [90.0, 99.2] \\
\addlinespace
\multicolumn{7}{l}{\emph{Stratum: \texttt{labeled\_tp}} ($n=20$ pairs)} \\
Chosen report clinically correct & 20 & 14 & 4 & 2 & 70.0 & [48.1, 85.5] \\
Target-pathology edit correct & 19 & 19 & 0 & 0 & 100.0 & [83.2, 100.0] \\
Unrelated findings preserved & 19 & 18 & 0 & 1 & 94.7 & [75.4, 99.1] \\
Pair overall usable & 20 & 19 & 1 & 0 & 95.0 & [76.4, 99.1] \\
\addlinespace
\multicolumn{7}{l}{\emph{Stratum: \texttt{tn}} ($n=20$ pairs)} \\
Chosen report clinically correct & 20 & 20 & 0 & 0 & 100.0 & [83.9, 100.0] \\
Target-pathology edit correct & 20 & 20 & 0 & 0 & 100.0 & [83.9, 100.0] \\
Unrelated findings preserved & 20 & 19 & 0 & 1 & 95.0 & [76.4, 99.1] \\
Pair overall usable & 20 & 20 & 0 & 0 & 100.0 & [83.9, 100.0] \\
\addlinespace
\multicolumn{7}{l}{\emph{Stratum: \texttt{unlabeled}} ($n=30$ pairs)} \\
Chosen report clinically correct & 30 & 28 & 0 & 2 & 93.3 & [78.7, 98.2] \\
Target-pathology edit correct & 30 & 29 & 0 & 1 & 96.7 & [83.3, 99.4] \\
Unrelated findings preserved & 30 & 30 & 0 & 0 & 100.0 & [88.6, 100.0] \\
Pair overall usable & 29 & 28 & 0 & 1 & 96.6 & [82.8, 99.4] \\
\addlinespace
\multicolumn{7}{l}{\emph{Pathology: cardiomegaly} ($n=35$ pairs)} \\
Chosen report clinically correct & 35 & 32 & 1 & 2 & 91.4 & [77.6, 97.0] \\
Target-pathology edit correct & 35 & 35 & 0 & 0 & 100.0 & [90.1, 100.0] \\
Unrelated findings preserved & 35 & 34 & 0 & 1 & 97.1 & [85.5, 99.5] \\
Pair overall usable & 35 & 34 & 1 & 0 & 97.1 & [85.5, 99.5] \\
\addlinespace
\multicolumn{7}{l}{\emph{Pathology: pleural effusion} ($n=35$ pairs)} \\
Chosen report clinically correct & 35 & 30 & 3 & 2 & 85.7 & [70.6, 93.7] \\
Target-pathology edit correct & 34 & 33 & 0 & 1 & 97.1 & [85.1, 99.5] \\
Unrelated findings preserved & 34 & 33 & 0 & 1 & 97.1 & [85.1, 99.5] \\
Pair overall usable & 34 & 33 & 0 & 1 & 97.1 & [85.1, 99.5] \\
\bottomrule
\end{tabular}
\end{table}

\clearpage
\section{Additional Results}
\label{app:additional_results}

\subsection{Subtype TPR}
To assess model performance at finer granularity, we partition each pathology into clinically meaningful subtypes based on presentation characteristics curated with a board-certified radiologist. For each generated report, we extract a diagnostic prediction using CheXbert \citep{smit2020combining}, a widely used automated label extractor, and compare it against the CheXbert-extracted label from the corresponding ground-truth report in the dataset. We restrict evaluation to samples for which the ground-truth report explicitly asserts the presence or absence of the finding, excluding cases where the label is uncertain or unmentioned. We use these results to select the set of target pathology subtypes that we use for illustrating our multi-task \method approach. 

\begin{figure}[htbp]
    \centering
    \begin{tabular}{ccc}
        \includegraphics[width=0.38\linewidth]{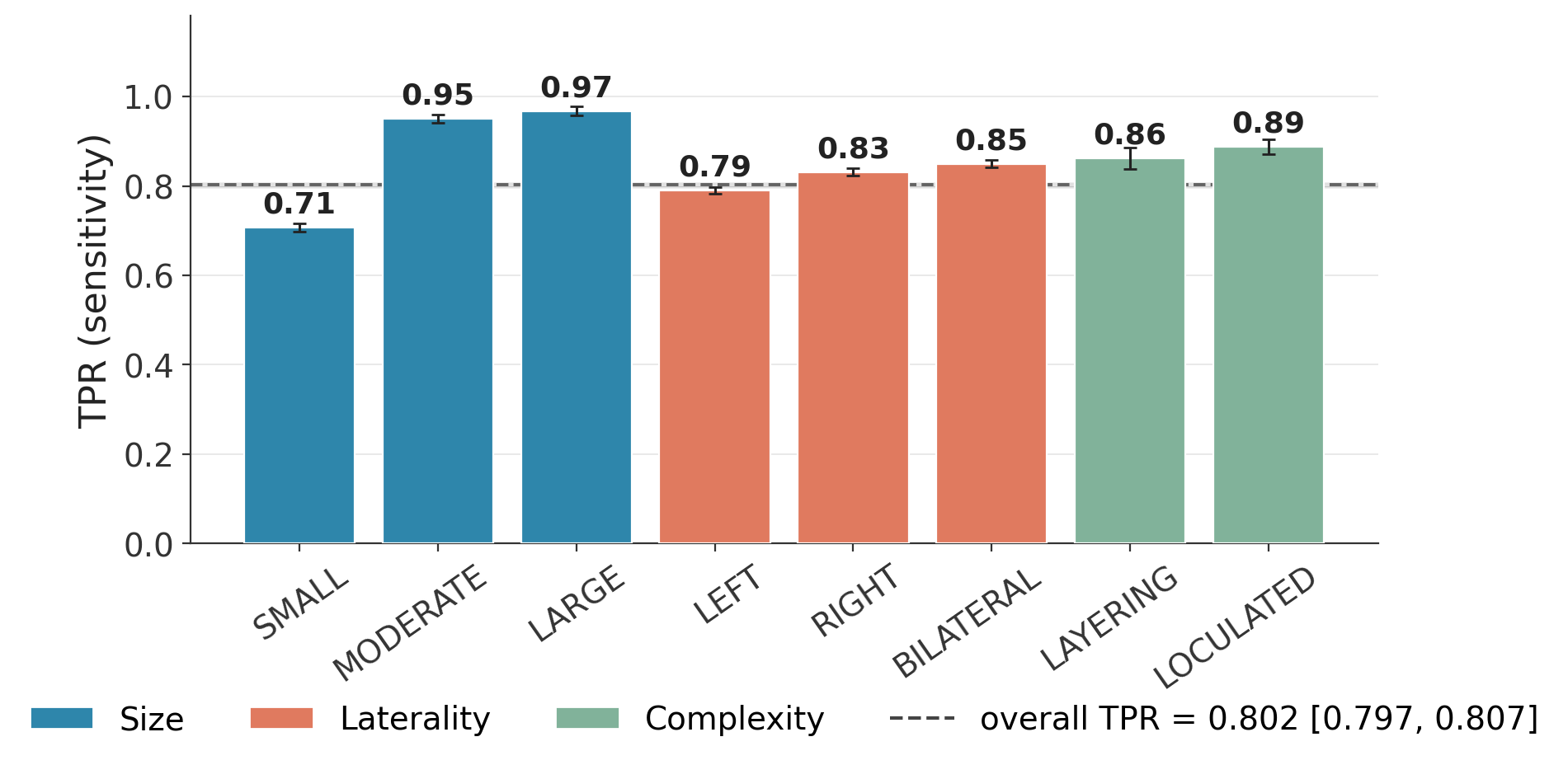} & 
        \includegraphics[width=0.38\linewidth]{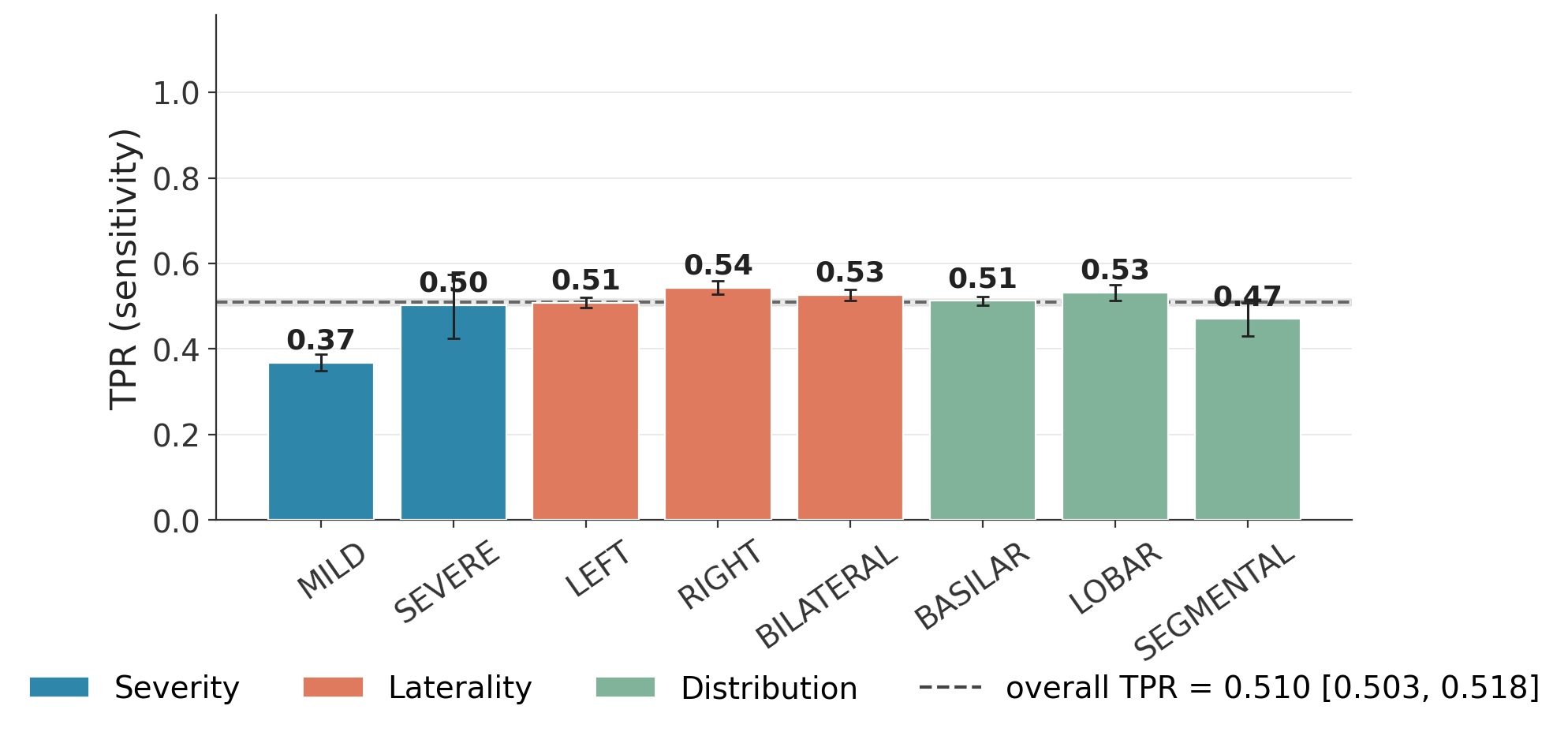} & 
        \includegraphics[width=0.22\linewidth]{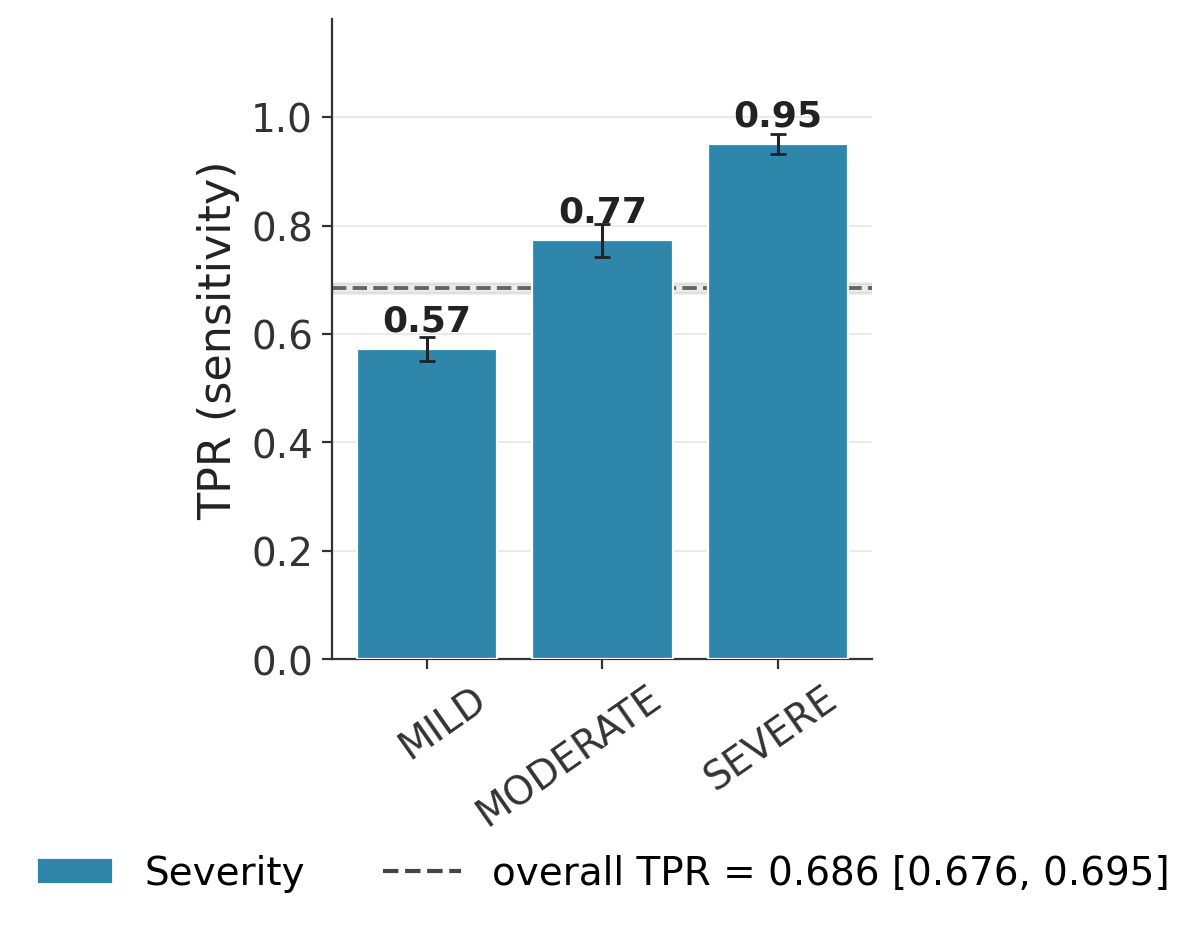} \\
        (a) Pleural Effusion & (b) Atelectasis & (c) Cardiomegaly
    \end{tabular}
    \caption{Certain pathologies (atelectasis) and presentation subtypes (small unilateral pleural effusions) are prone to a lower detection rate, demonstrating the possibility that annotation noise in training data is inherited by the model.}
    \label{fig:tpr_attributes}
\end{figure}

\paragraph{Different model} We demonstrate that these empirical insights are model agnostic by replicating the results using a different model (NV-reason). 

\begin{figure}[!h]
    \centering
    \begin{tabular}{ccc}
        \includegraphics[width=0.38\linewidth]{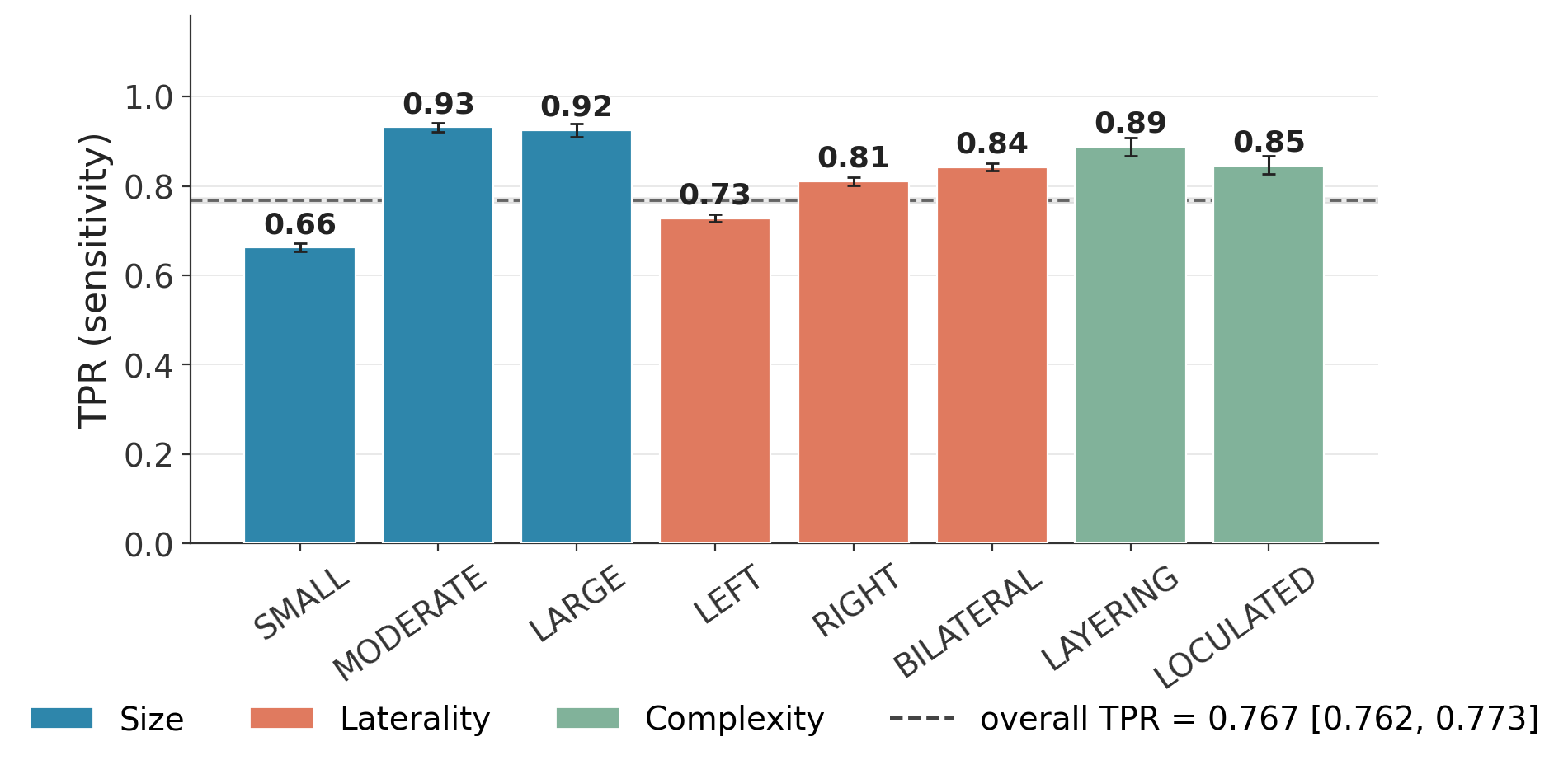} & 
        \includegraphics[width=0.38\linewidth]{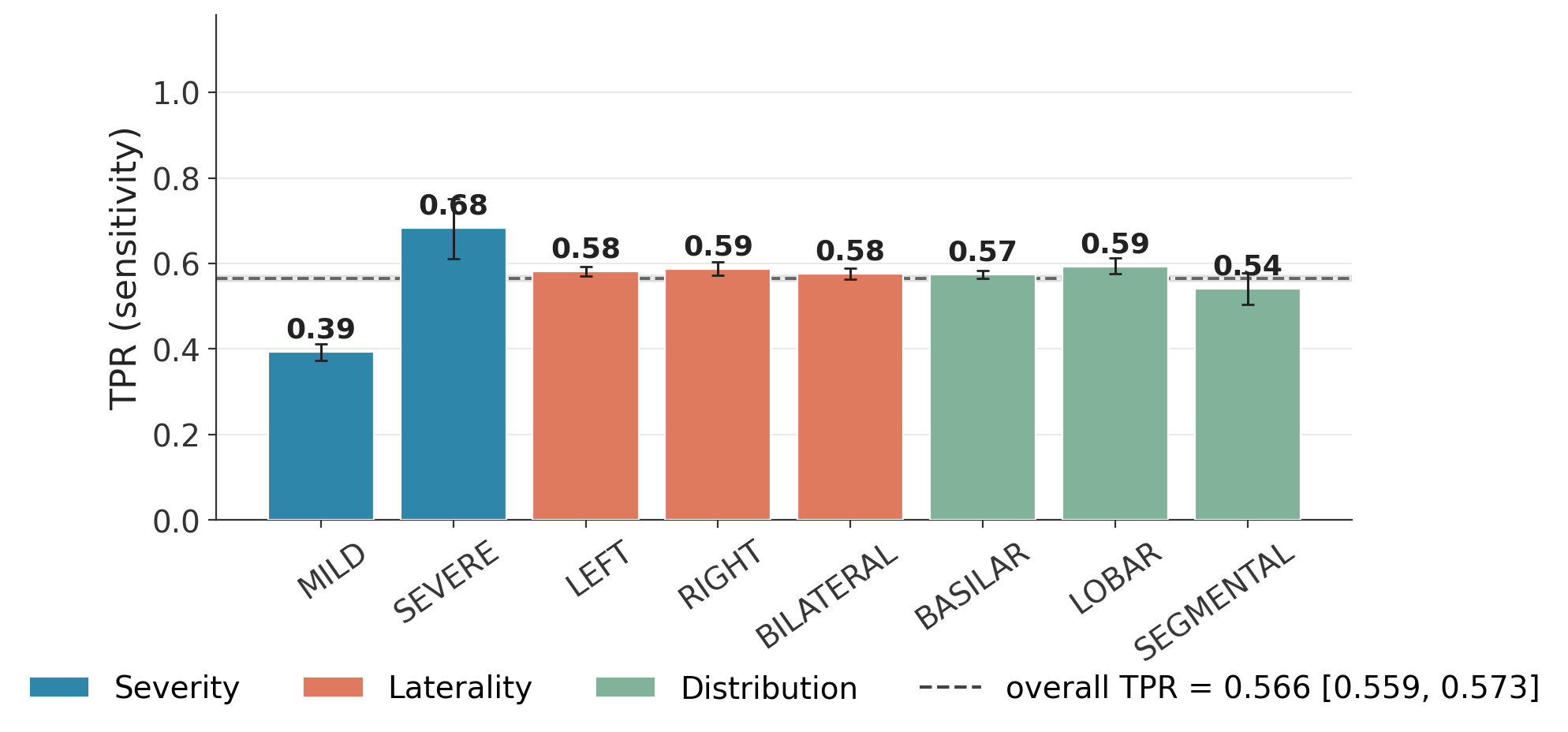} & 
        \includegraphics[width=0.22\linewidth]{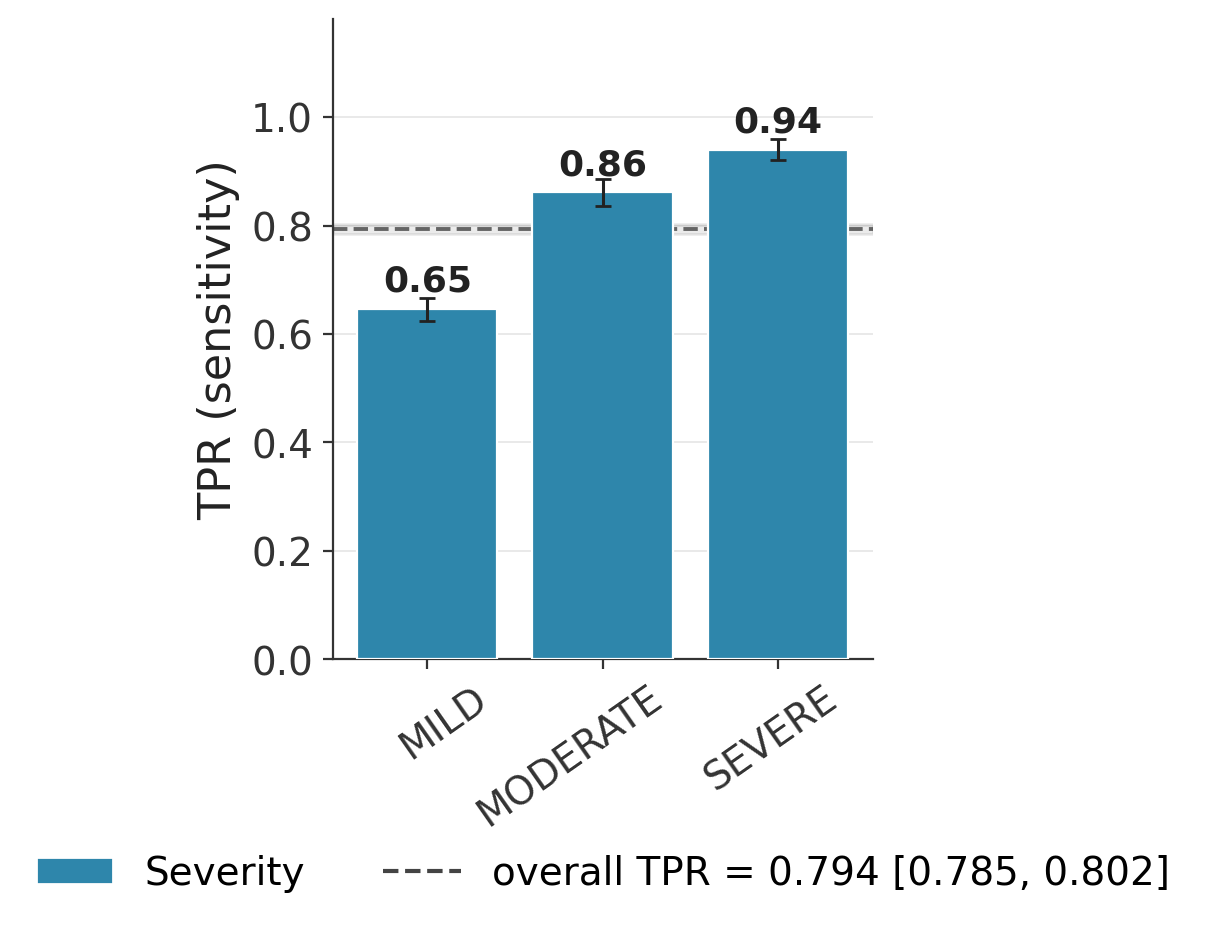} \\
        (a) Pleural Effusion & (b) Atelectasis & (c) Cardiomegaly
    \end{tabular}
    \caption{Subgroup-wise marginal sensitivity for the NV-reason model, showing consistent patterns in subtypes with low detection rate.}
    \label{fig:tpr_attributes_nvreason}
\end{figure}

\paragraph{Different label extractor} We also investigate whether the results are reproduced using an alternative label extraction procedure using GPT-OSS 120B. 

\begin{figure}[htbp]
    \centering
    \begin{tabular}{ccc}
        \includegraphics[width=0.38\linewidth]{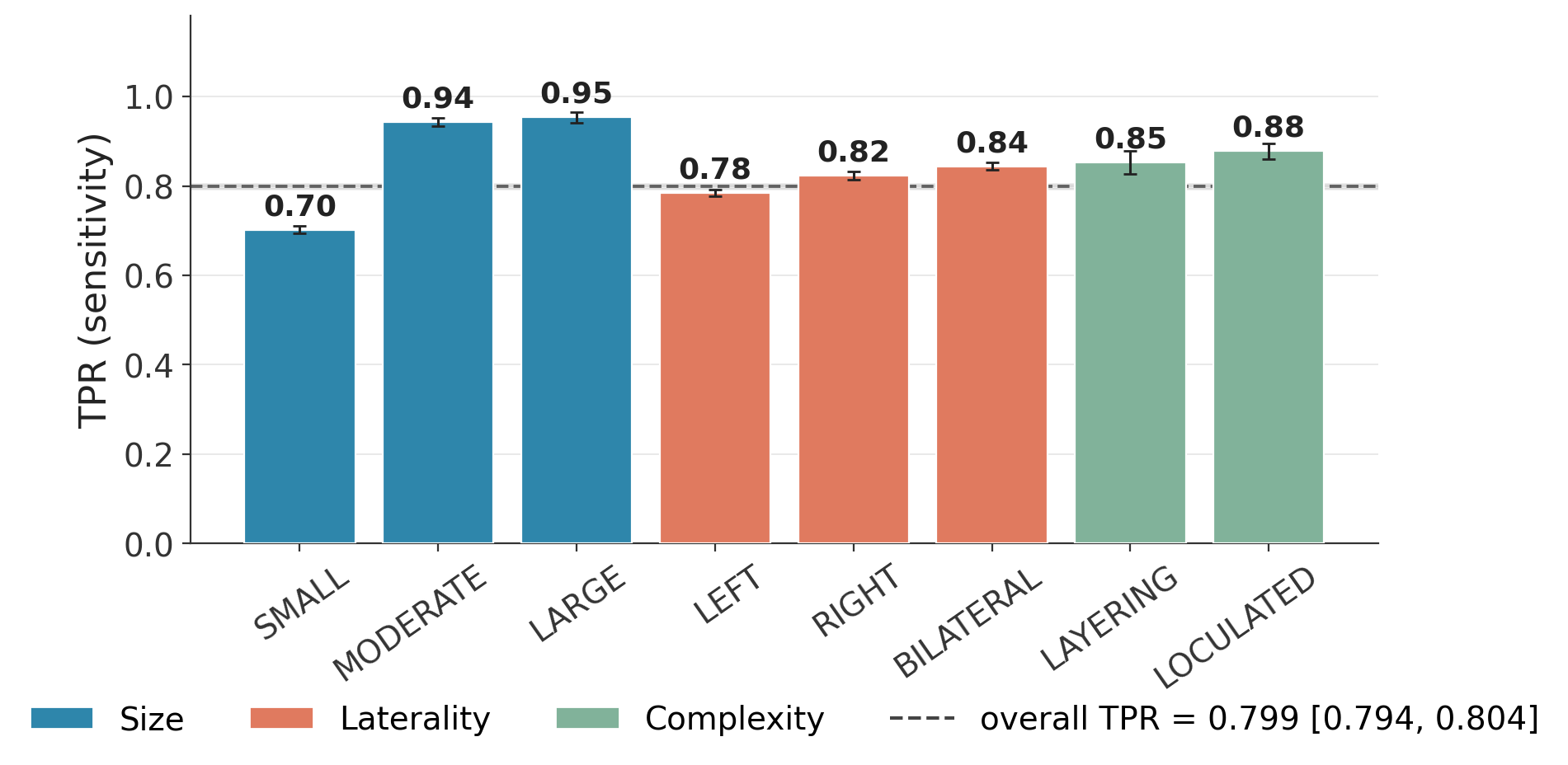} & 
        \includegraphics[width=0.38\linewidth]{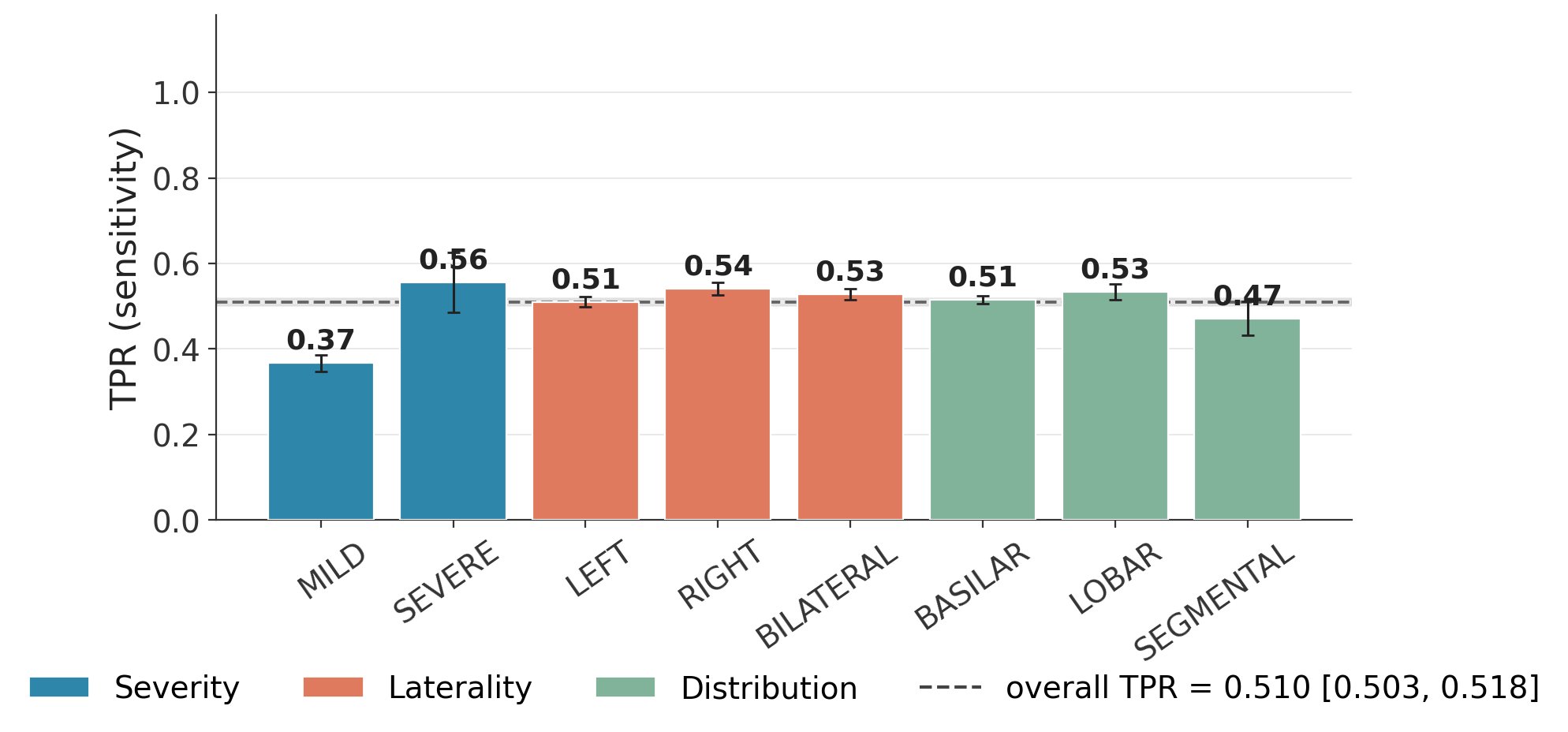} & 
        \includegraphics[width=0.22\linewidth]{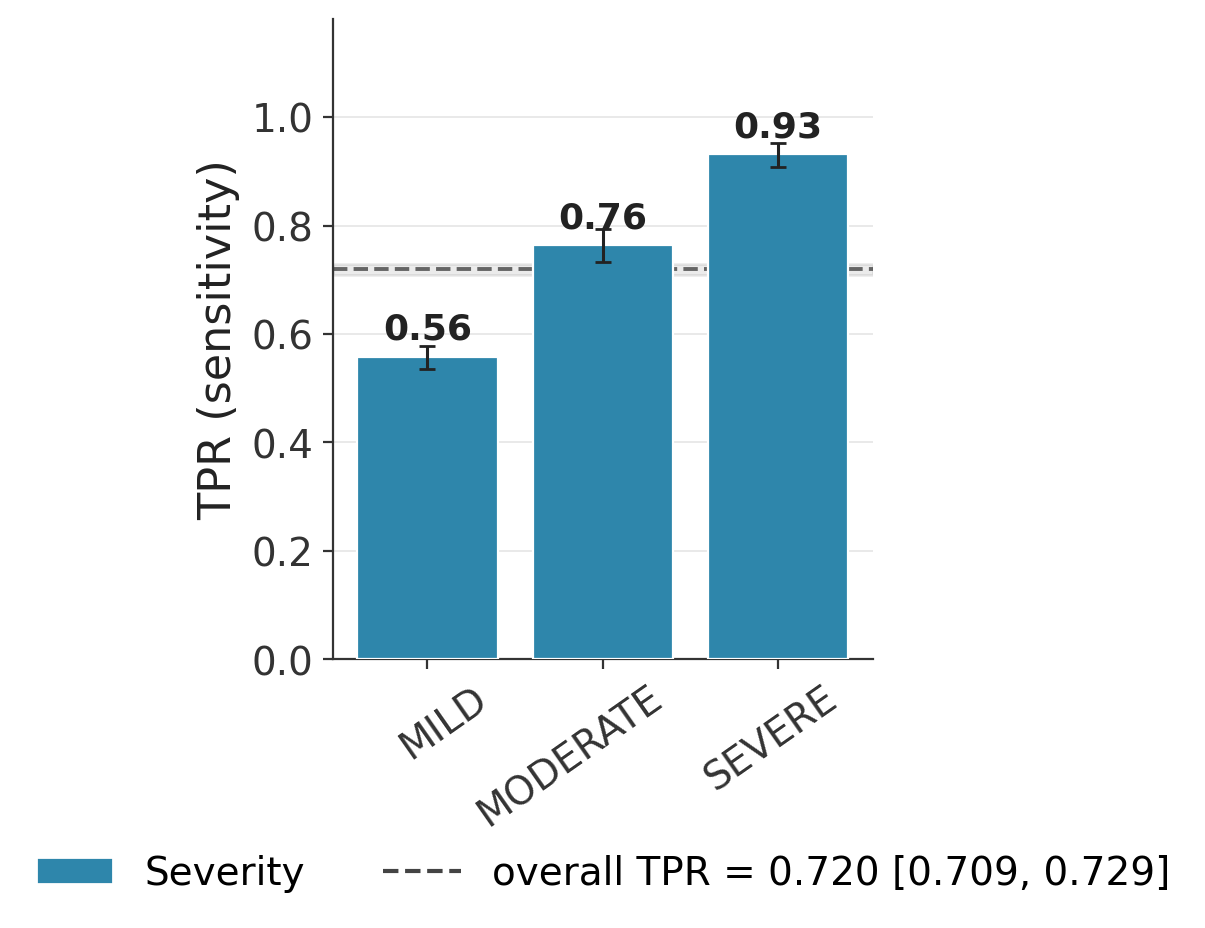} \\
        (a) Pleural Effusion & (b) Atelectasis & (c) Cardiomegaly
    \end{tabular}
    \caption{We find consistent patterns with the GPT-OSS 120B model extracting clinical labels from the generated predictions and ground truth reports.}
    \label{fig:tpr_attributes_gptoss}
\end{figure}

\paragraph{Different dataset}
Finally, we reproduce consistent results using the MIMIC dataset, a separate dataset.

\begin{figure}[htbp]
    \centering
    \begin{tabular}{ccc}
        \includegraphics[width=0.38\linewidth]{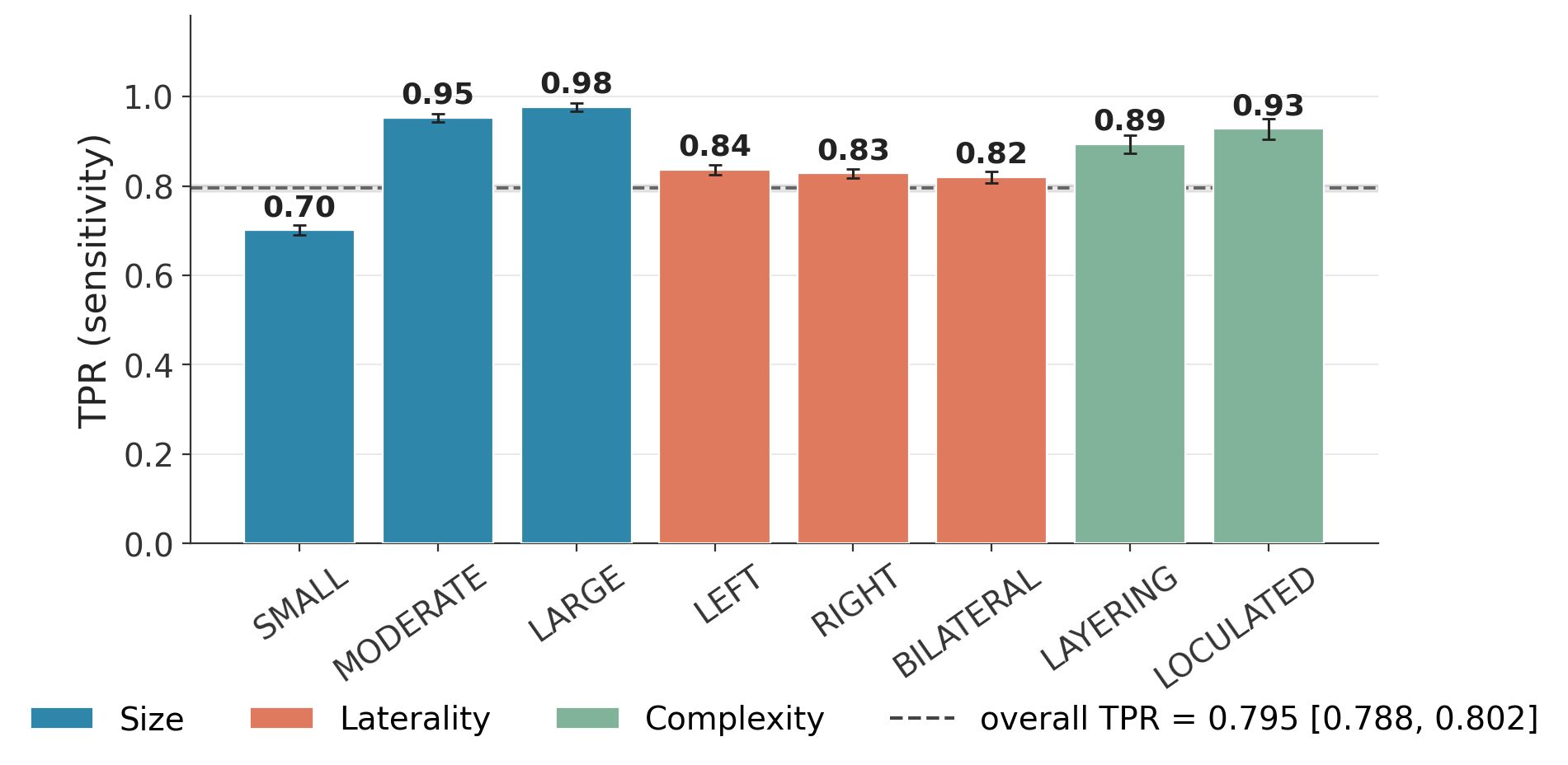} & 
        \includegraphics[width=0.38\linewidth]{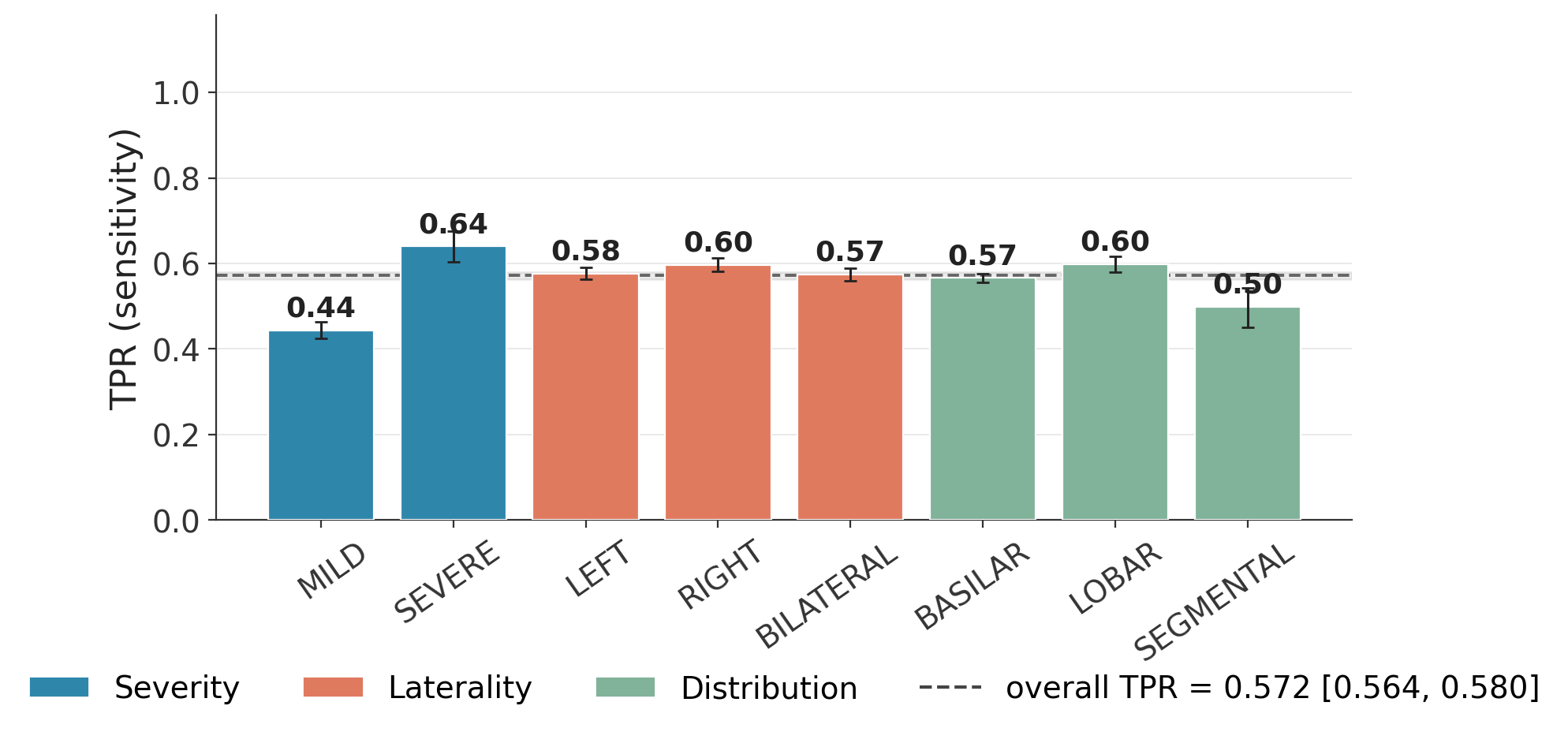} & 
        \includegraphics[width=0.22\linewidth]{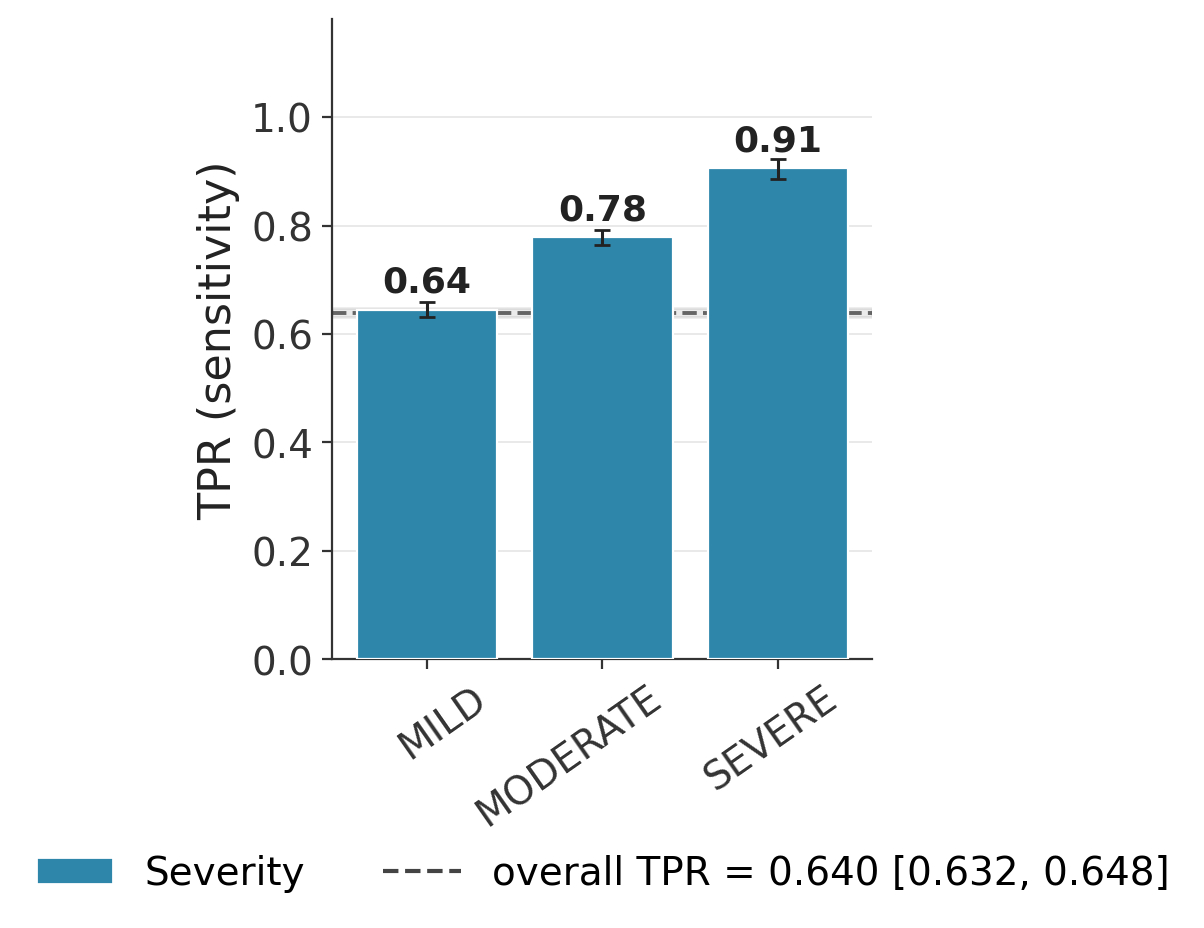} \\
        (a) Pleural Effusion & (b) Atelectasis & (c) Cardiomegaly
    \end{tabular}
    \caption{We find consistent patterns in the MIMIC dataset.}
    \label{fig:tpr_attributes}
\end{figure}

\subsection{Within-pathology Variability in Sensitivity}
We demonstrate that pathology-level detection rates (only considering pleural effusion F1 score) hides the granularity of pathology presentation. We find that models struggle with detecting certain pathology subtypes. 

\begin{figure}[!h]
    \centering
    \includegraphics[width=0.9\linewidth]{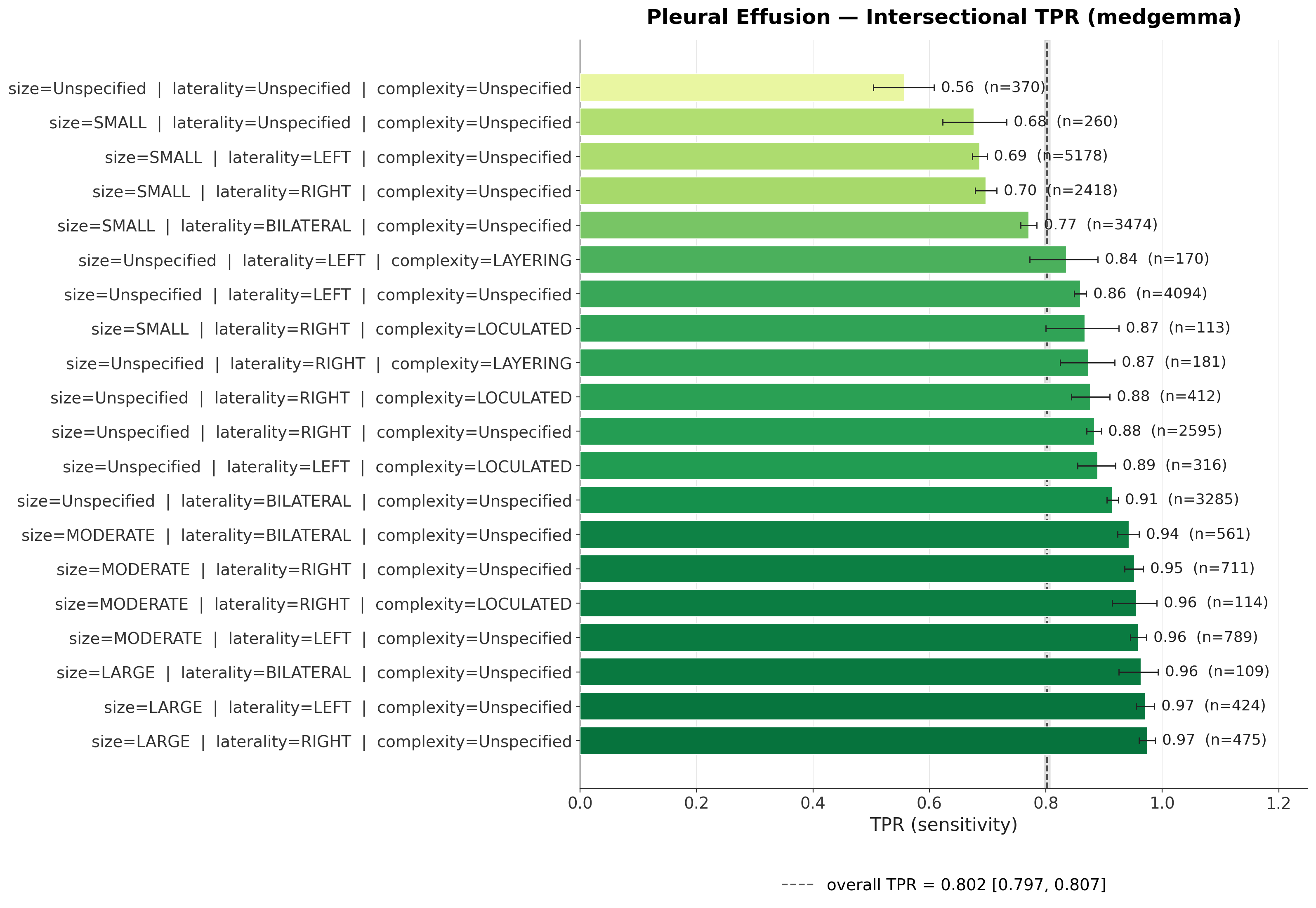}
    \caption{We plot the TPR of pleural effusion for the top 20 most frequent intersectional subtypes defined by size, laterality, and complexity.}
    \label{fig:intersection_attribute_effusion}
\end{figure}

\begin{figure}[!h]
    \centering
    \includegraphics[width=0.9\linewidth]{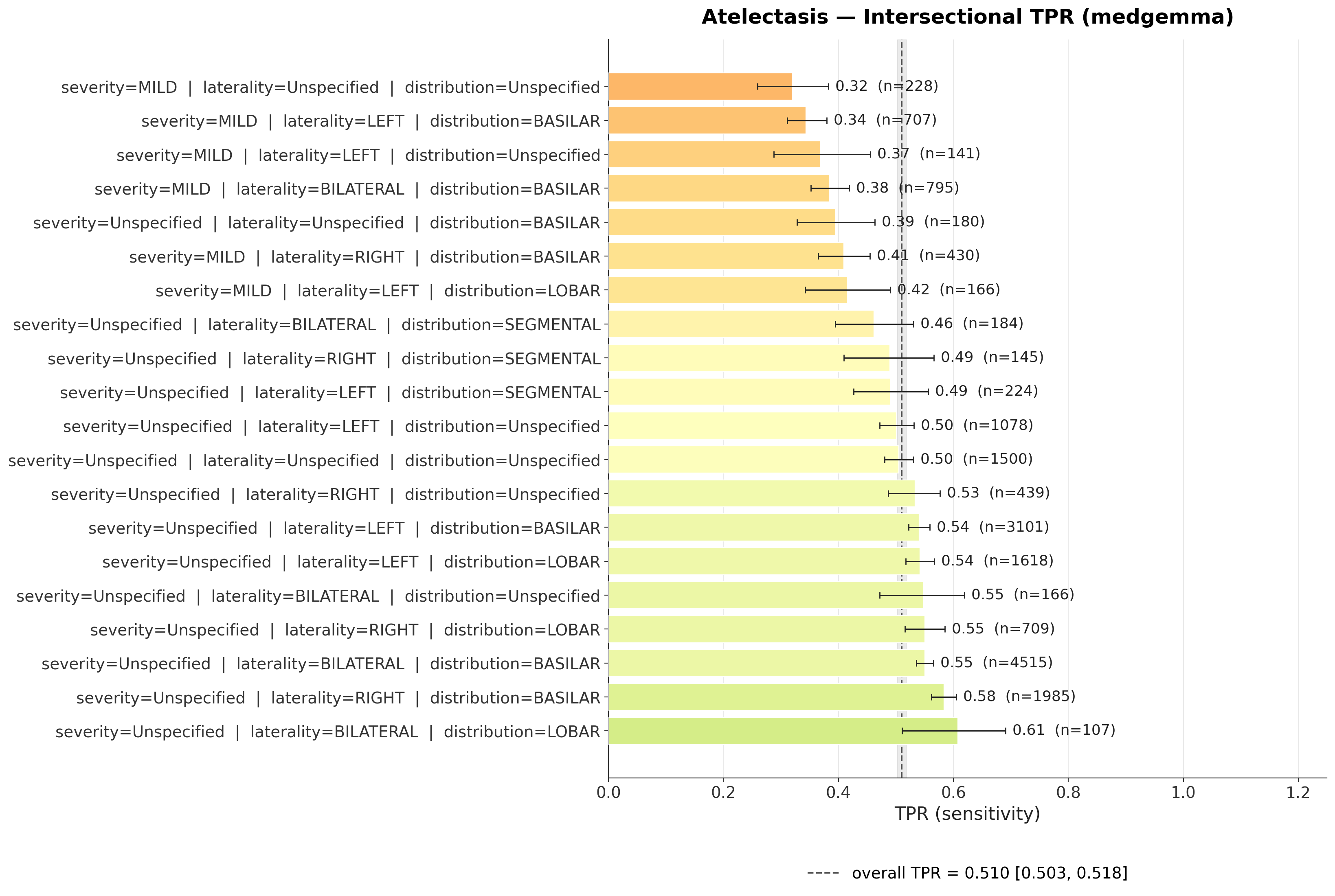}
    \caption{We plot the TPR of atelectasis for the top 20 most frequent intersectional subtypes defined by size, laterality, and distribution.}
    \label{fig:intersection_attribute_atelectasis}
\end{figure}
\FloatBarrier

\subsection{BBE Estimation Results}
\label{app:alpha_estimation_results}
We provide the estimated mixture proportion (hidden positive rates among unlabeled samples) from BBE. 

\begin{table}[htbp]
\centering
\caption{Estimated prior probabilities ($\alpha$) by pathology and specific clinical subtype for CheXpert Dataset.}
\label{tab:estimated_alphas}
\begin{tabular}{@{} l c @{}}
\toprule
Pathology Subtype & Estimated $\alpha$ \\
\midrule

% --- ATELECTASIS ---
\multicolumn{2}{@{}l}{\textbf{Atelectasis}} \\
Mild, Bilateral & 0.1149 \\
Mild, Left & 0.0874 \\
Mild, Right & 0.1122 \\
\addlinespace

% --- CARDIOMEGALY ---
\multicolumn{2}{@{}l}{\textbf{Cardiomegaly}} \\
Mild & 0.0894 \\
\addlinespace

% --- PLEURAL EFFUSION ---
\multicolumn{2}{@{}l}{\textbf{Pleural Effusion}} \\
Small, Bilateral & 0.1059 \\
Small, Left & 0.1692 \\
Small, Right & 0.1016 \\

\bottomrule
\end{tabular}
\end{table}

\begin{table}[htbp]
\centering
\caption{Estimated prior probabilities ($\alpha$) by pathology and specific clinical subtype for the MIMIC dataset.}
\label{tab:mimic_estimated_alphas}
\begin{tabular}{@{} l c @{}}
\toprule
Pathology Subtype & Estimated $\alpha$ \\
\midrule

% --- ATELECTASIS ---
\multicolumn{2}{@{}l}{\textbf{Atelectasis}} \\
Mild, Bilateral & 0.0905 \\
Mild, Left & 0.0885 \\
Mild, Right & 0.0935 \\
\addlinespace

% --- CARDIOMEGALY ---
\multicolumn{2}{@{}l}{\textbf{Cardiomegaly}} \\
Mild & 0.1553 \\
\addlinespace

% --- PLEURAL EFFUSION ---
\multicolumn{2}{@{}l}{\textbf{Pleural Effusion}} \\
Small, Bilateral & 0.1016 \\
Small, Left & 0.1419 \\
Small, Right & 0.1041 \\

\bottomrule
\end{tabular}
\end{table}

\FloatBarrier

\subsection{Semi-synthetic Experiment}\label{sec:synth_results}

\begin{table}[!ht]
\centering
\caption{Results on \textbf{Pleural Effusion} under simulated preference noise. 95\% bootstrap CIs in parentheses.}
\label{tab:results-pleural-effusion}
\resizebox{\linewidth}{!}{%
\begin{tabular}{l ccc ccc ccc}
\toprule
 & \multicolumn{3}{c}{Standard DPO} & \multicolumn{3}{c}{Dr-DPO} & \multicolumn{3}{c}{\method} \\
\cmidrule(lr){2-4} \cmidrule(lr){5-7} \cmidrule(lr){8-10}
Noise Rate & F1 & Sens. & Spec. & F1 & Sens. & Spec. & F1 & Sens. & Spec. \\
\midrule
0\% & 0.83 {\scriptsize(0.82--0.84)} & 0.75 {\scriptsize(0.74--0.77)} & 0.94 {\scriptsize(0.94--0.95)} & 0.84 {\scriptsize(0.83--0.85)} & 0.76 {\scriptsize(0.75--0.78)} & 0.94 {\scriptsize(0.93--0.95)} & 0.84 {\scriptsize(0.83--0.85)} & 0.76 {\scriptsize(0.74--0.78)} & 0.95 {\scriptsize(0.94--0.96)} \\
10\% & 0.82 {\scriptsize(0.80--0.83)} & 0.73 {\scriptsize(0.71--0.74)} & 0.95 {\scriptsize(0.94--0.96)} & 0.82 {\scriptsize(0.81--0.84)} & 0.74 {\scriptsize(0.72--0.75)} & \textbf{0.95 {\scriptsize(0.94--0.96)}} & \textbf{0.84 {\scriptsize(0.83--0.86)}} & \textbf{0.77 {\scriptsize(0.75--0.79)}} & 0.94 {\scriptsize(0.94--0.95)} \\
20\% & 0.80 {\scriptsize(0.78--0.81)} & 0.69 {\scriptsize(0.67--0.70)} & 0.96 {\scriptsize(0.95--0.97)} & 0.79 {\scriptsize(0.78--0.80)} & 0.68 {\scriptsize(0.66--0.70)} & \textbf{0.96 {\scriptsize(0.95--0.97)}} & \textbf{0.84 {\scriptsize(0.82--0.85)}} & \textbf{0.76 {\scriptsize(0.74--0.77)}} & 0.95 {\scriptsize(0.94--0.96)} \\
30\% & 0.81 {\scriptsize(0.80--0.82)} & 0.73 {\scriptsize(0.71--0.74)} & 0.93 {\scriptsize(0.92--0.94)} & 0.76 {\scriptsize(0.74--0.77)} & 0.63 {\scriptsize(0.61--0.65)} & \textbf{0.97 {\scriptsize(0.96--0.97)}} & \textbf{0.84 {\scriptsize(0.83--0.85)}} & \textbf{0.76 {\scriptsize(0.75--0.78)}} & 0.95 {\scriptsize(0.94--0.96)} \\
\midrule
Base & \multicolumn{9}{c}{F1: 0.81 {\scriptsize(0.79--0.82)} \quad Sens.: 0.74 {\scriptsize(0.73--0.76)} \quad Spec.: 0.90 {\scriptsize(0.89--0.91)}} \\
\bottomrule
\end{tabular}
}
\end{table}

\begin{table}[!ht]
\centering
\caption{Results on \textbf{Cardiomegaly} under simulated preference noise. 95\% bootstrap CIs in parentheses.}
\label{tab:results-cardiomegaly}
\resizebox{\linewidth}{!}{%
\begin{tabular}{l ccc ccc ccc}
\toprule
 & \multicolumn{3}{c}{Standard DPO} & \multicolumn{3}{c}{Dr-DPO} & \multicolumn{3}{c}{\method} \\
\cmidrule(lr){2-4} \cmidrule(lr){5-7} \cmidrule(lr){8-10}
Noise Rate & F1 & Sens. & Spec. & F1 & Sens. & Spec. & F1 & Sens. & Spec. \\
\midrule
0\% & 0.82 {\scriptsize(0.80--0.83)} & 0.77 {\scriptsize(0.75--0.79)} & 0.88 {\scriptsize(0.87--0.89)} & 0.82 {\scriptsize(0.81--0.83)} & 0.77 {\scriptsize(0.76--0.79)} & 0.89 {\scriptsize(0.87--0.90)} & 0.83 {\scriptsize(0.81--0.84)} & 0.80 {\scriptsize(0.78--0.82)} & 0.86 {\scriptsize(0.85--0.88)} \\
10\% & 0.82 {\scriptsize(0.81--0.83)} & \textbf{0.79 {\scriptsize(0.77--0.80)}} & 0.87 {\scriptsize(0.86--0.89)} & 0.82 {\scriptsize(0.81--0.83)} & 0.77 {\scriptsize(0.75--0.79)} & \textbf{0.89 {\scriptsize(0.87--0.90)}} & \textbf{0.82 {\scriptsize(0.81--0.83)}} & 0.79 {\scriptsize(0.77--0.80)} & 0.88 {\scriptsize(0.86--0.89)} \\
20\% & 0.80 {\scriptsize(0.79--0.81)} & 0.75 {\scriptsize(0.73--0.76)} & 0.88 {\scriptsize(0.87--0.90)} & 0.81 {\scriptsize(0.79--0.82)} & 0.75 {\scriptsize(0.73--0.77)} & \textbf{0.89 {\scriptsize(0.88--0.90)}} & \textbf{0.83 {\scriptsize(0.82--0.84)}} & \textbf{0.81 {\scriptsize(0.80--0.83)}} & 0.85 {\scriptsize(0.84--0.86)} \\
30\% & 0.77 {\scriptsize(0.75--0.78)} & 0.68 {\scriptsize(0.67--0.70)} & 0.91 {\scriptsize(0.90--0.92)} & 0.78 {\scriptsize(0.76--0.79)} & 0.69 {\scriptsize(0.67--0.71)} & \textbf{0.91 {\scriptsize(0.90--0.92)}} & \textbf{0.82 {\scriptsize(0.81--0.83)}} & \textbf{0.78 {\scriptsize(0.77--0.80)}} & 0.87 {\scriptsize(0.85--0.88)} \\
\midrule
Base & \multicolumn{9}{c}{F1: 0.77 {\scriptsize(0.76--0.79)} \quad Sens.: 0.70 {\scriptsize(0.68--0.72)} \quad Spec.: 0.89 {\scriptsize(0.88--0.90)}} \\
\bottomrule
\end{tabular}
}
\end{table}

We further evaluate \method by examining the calibration of the model's learned preferences over contrastive-edit pairs. For each test image, we treat $\mathbb{P}_\phi(a^+ \succ a^- \mid x)$ — the model's probability of preferring the response that mentions the finding over the one that omits it, as a soft prediction that the pathology is present, and score it against the ground-truth label using log loss, Brier score, and expected calibration error (ECE).

\begin{table}[!ht]
\centering
\caption{Uncertainty quantification metrics on the \textbf{Pleural Effusion} semi-synthetic experiment, examining learned preference probabilities over a set of noiseless present-absent validation pairs. \method maintains low Brier score and log loss across training data noise rates, demonstrating that the model maintains accurate preference probabilities.}
\label{tab:results-uncertainty}
\resizebox{0.8\linewidth}{!}{%
\begin{tabular}{llllll}
\toprule
Flip Rate & Method & Mean Gap & Brier & ECE & Log Loss \\
\midrule
0.00 & Base & 0.269 & 0.238 {\scriptsize (0.227, 0.250)} & 0.138 {\scriptsize (0.108, 0.171)} & 0.665 {\scriptsize (0.638, 0.691)} \\
 & PU-SFT & 0.248 & 0.233 {\scriptsize (0.223, 0.243)} & \textbf{0.127 {\scriptsize (0.100, 0.164)}} & 0.654 {\scriptsize (0.631, 0.678)} \\
 & DPO & 0.618 & 0.171 {\scriptsize (0.161, 0.180)} & 0.184 {\scriptsize (0.158, 0.209)} & 0.517 {\scriptsize (0.495, 0.539)} \\
 & Dr. DPO & 0.623 & \textbf{0.169 {\scriptsize (0.160, 0.179)}} & 0.177 {\scriptsize (0.152, 0.202)} & \textbf{0.514 {\scriptsize (0.493, 0.536)}} \\
 & PU-DPO & 0.628 & 0.170 {\scriptsize (0.161, 0.180)} & 0.185 {\scriptsize (0.160, 0.210)} & 0.515 {\scriptsize (0.493, 0.538)} \\
\midrule
0.10 & Base & 0.169 & 0.238 {\scriptsize (0.227, 0.250)} & 0.138 {\scriptsize (0.108, 0.171)} & 0.665 {\scriptsize (0.638, 0.691)} \\
 & PU-SFT & 0.170 & 0.232 {\scriptsize (0.222, 0.242)} & \textbf{0.126 {\scriptsize (0.099, 0.161)}} & 0.651 {\scriptsize (0.629, 0.674)} \\
 & DPO & 0.472 & 0.171 {\scriptsize (0.163, 0.181)} & 0.178 {\scriptsize (0.153, 0.203)} & 0.520 {\scriptsize (0.500, 0.541)} \\
 & Dr. DPO & 0.477 & 0.170 {\scriptsize (0.161, 0.180)} & 0.176 {\scriptsize (0.149, 0.200)} & 0.517 {\scriptsize (0.498, 0.539)} \\
 & PU-DPO & 0.502 & \textbf{0.170 {\scriptsize (0.161, 0.180)}} & 0.192 {\scriptsize (0.167, 0.217)} & \textbf{0.515 {\scriptsize (0.493, 0.537)}} \\
\midrule
0.20 & Base & 0.072 & 0.238 {\scriptsize (0.227, 0.250)} & 0.138 {\scriptsize (0.108, 0.171)} & 0.665 {\scriptsize (0.638, 0.691)} \\
 & PU-SFT & 0.082 & 0.231 {\scriptsize (0.222, 0.242)} & \textbf{0.135 {\scriptsize (0.105, 0.167)}} & 0.650 {\scriptsize (0.628, 0.673)} \\
 & DPO & 0.321 & 0.178 {\scriptsize (0.169, 0.187)} & 0.170 {\scriptsize (0.146, 0.196)} & 0.535 {\scriptsize (0.516, 0.557)} \\
 & Dr. DPO & 0.331 & 0.176 {\scriptsize (0.167, 0.185)} & 0.168 {\scriptsize (0.151, 0.195)} & 0.531 {\scriptsize (0.512, 0.552)} \\
 & PU-DPO & 0.353 & \textbf{0.174 {\scriptsize (0.165, 0.184)}} & 0.175 {\scriptsize (0.149, 0.200)} & \textbf{0.525 {\scriptsize (0.504, 0.548)}} \\
\midrule
0.30 & Base & -0.049 & 0.238 {\scriptsize (0.227, 0.250)} & 0.138 {\scriptsize (0.108, 0.171)} & 0.665 {\scriptsize (0.638, 0.691)} \\
 & PU-SFT & -0.027 & 0.232 {\scriptsize (0.222, 0.242)} & \textbf{0.125 {\scriptsize (0.100, 0.162)}} & 0.651 {\scriptsize (0.628, 0.674)} \\
 & DPO & 0.188 & 0.182 {\scriptsize (0.173, 0.191)} & 0.154 {\scriptsize (0.137, 0.183)} & 0.544 {\scriptsize (0.524, 0.565)} \\
 & Dr. DPO & 0.194 & 0.179 {\scriptsize (0.171, 0.189)} & 0.163 {\scriptsize (0.147, 0.191)} & 0.540 {\scriptsize (0.520, 0.561)} \\
 & PU-DPO & 0.213 & \textbf{0.173 {\scriptsize (0.163, 0.182)}} & 0.177 {\scriptsize (0.153, 0.202)} & \textbf{0.521 {\scriptsize (0.499, 0.543)}} \\
\bottomrule
\end{tabular}
}
\end{table}

\begin{figure}[!h]
    \centering
    \includegraphics[width=0.95\linewidth]{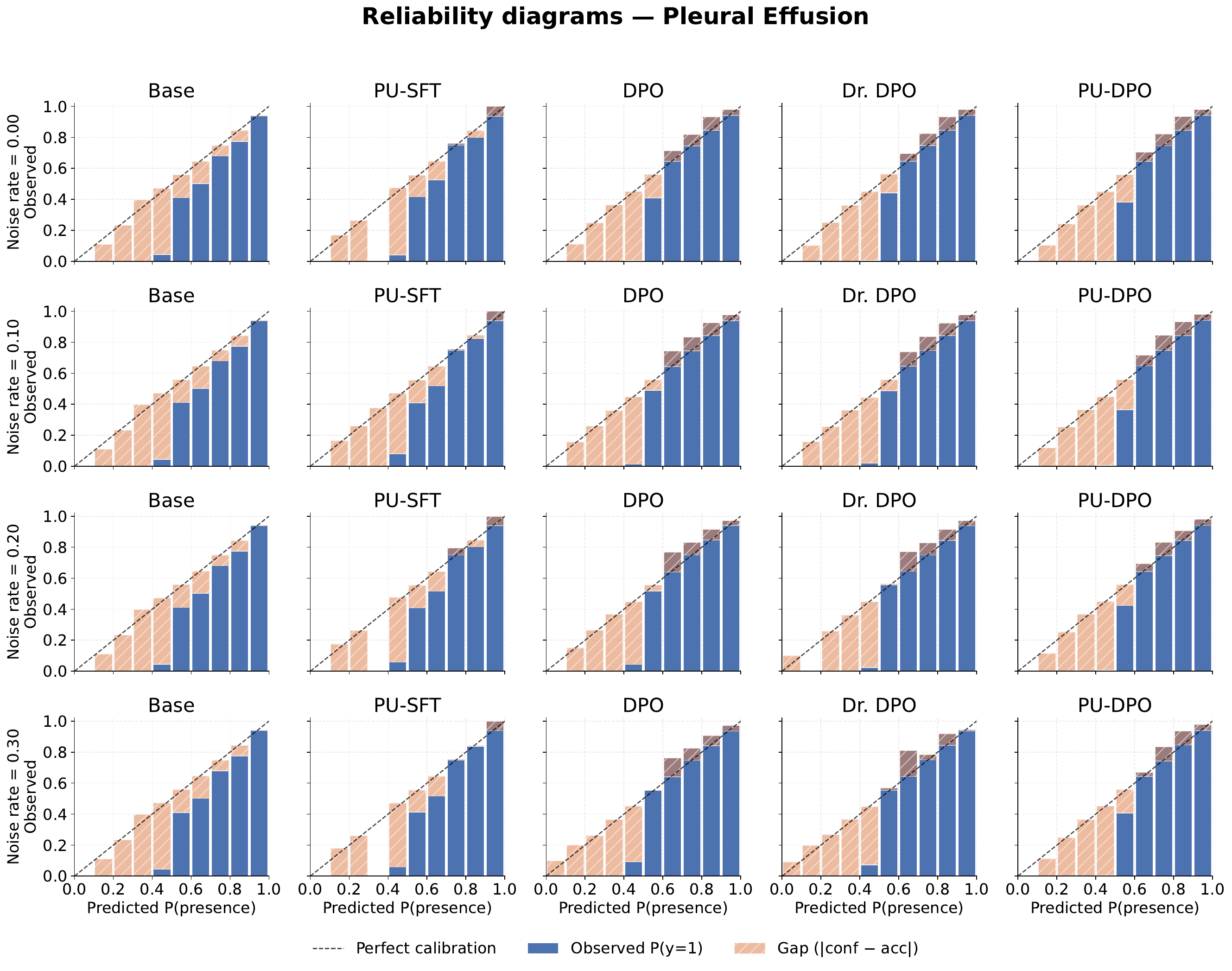}
    \caption{Calibration of learned preference probabilities for pleural effusion. \method leads to slight overconfidence in predicted probabilities.}
    \label{fig:calibration_pleural}
\end{figure}

\FloatBarrier

\textbf{Robustness to misspecification.}
A natural concern with PU-based objectives is sensitivity to the class prior $\alpha$: in practice $\alpha_j$ must be estimated, and downstream performance may degrade if the estimate is off. To probe this, we evaluate \method on the pleural effusion task under 20\% simulated omission noise, sweeping $\alpha$ values around the oracle setting (Table~\ref{tab:results-alpha-misspecification}).

The results show that \method is broadly robust to moderate misspecification of $\alpha$. Setting $\alpha = 0.6$ — overestimating the class prior — yields the best overall performance (F1 = 0.86, sensitivity = 0.81), exceeding both the oracle setting and standard DPO. Underestimating $\alpha$ at $0.4$ collapses \method's behavior toward DPO (F1 = 0.79, sensitivity = 0.69), which is consistent with the form of the PU objective: as $\alpha \to 0$, the contamination correction term vanishes and the loss reduces to standard DPO on the noisy preferences. Across all settings, specificity remains stable in the 0.93--0.96 range, indicating that misspecification primarily affects the trade-off between detection and conservatism rather than introducing spurious positives.

These results have a useful practical implication: erring slightly on the high side when estimating $\alpha_j$ is preferable to erring low, since underestimation reverts to the noise-corrupted baseline while overestimation continues to suppress omission noise effectively.

\begin{table}[!ht]
\centering
\caption{Results on \textbf{Pleural Effusion} under simulated preference noise. 95\% bootstrap CIs in parentheses.}
\label{tab:results-alpha-misspecification}
\begin{tabular}{ll ccc}
\toprule
Noise Rate & Method & F1 & Sens. & Spec. \\
\midrule
\multirow{4}{*}{20\%} 
 & DPO & 0.80 {\scriptsize(0.78--0.81)} & 0.69 {\scriptsize(0.67--0.70)} & \textbf{0.96 {\scriptsize(0.95--0.97)}} \\
 & PU-DPO (oracle) & 0.84 {\scriptsize(0.82--0.85)} & 0.76 {\scriptsize(0.74--0.77)} & 0.95 {\scriptsize(0.94--0.96)} \\
 & PU-DPO ($\alpha=0.4$) & 0.79 {\scriptsize(0.78--0.81)} & 0.69 {\scriptsize(0.67--0.70)} & 0.96 {\scriptsize(0.95--0.97)} \\
 & PU-DPO ($\alpha=0.6$) & \textbf{0.86 {\scriptsize(0.85--0.87)}} & \textbf{0.81 {\scriptsize(0.79--0.83)}} & 0.93 {\scriptsize(0.92--0.94)} \\
\bottomrule
\end{tabular}
\end{table}

\FloatBarrier

\subsection{Real-world Experiment}
\label{app:real_exp_results}

\paragraph{Finding-level F1}
We show the pathology level F1 scores showing reliable gains in \method for F1 score. The NV-reason model does not show the highest performance, it improves over naive DPO and base model for almost all dataset-pathology configurations.

\begin{table}[!ht]
\caption{F1 by dataset, model, pathology, and method (mean [95\% bootstrap CI], test-set resampling with 1000 draws)}
\label{tab:filtered_f1_by_dataset_pathology_method}
\resizebox{\linewidth}{!}{%
\begin{tabular}{llllllll}
\toprule
Dataset & Model & Pathology & Base & SFT & RadCliq-GRPO & DPO & PU-DPO \\
\midrule
\multirow[t]{9}{*}{CheXpert} & \multirow[t]{3}{*}{Llava-RAD} & Atelectasis & 0.728 {\scriptsize [0.721, 0.734]} & 0.633 {\scriptsize [0.625, 0.640]} & 0.742 {\scriptsize [0.736, 0.748]} & 0.730 {\scriptsize [0.724, 0.736]} & \textbf{0.743 {\scriptsize [0.737, 0.749]}} \\
 &  & Cardiomegaly & 0.743 {\scriptsize [0.735, 0.751]} & 0.669 {\scriptsize [0.661, 0.677]} & \textbf{0.745 {\scriptsize [0.738, 0.753]}} & 0.726 {\scriptsize [0.718, 0.733]} & 0.736 {\scriptsize [0.729, 0.744]} \\
 &  & Pleural Effusion & 0.845 {\scriptsize [0.842, 0.849]} & 0.877 {\scriptsize [0.874, 0.880]} & 0.854 {\scriptsize [0.850, 0.858]} & 0.921 {\scriptsize [0.919, 0.924]} & \textbf{0.929 {\scriptsize [0.927, 0.931]}} \\
\cline{2-8}
 & \multirow[t]{3}{*}{Medgemma} & Atelectasis & 0.675 {\scriptsize [0.668, 0.681]} & 0.559 {\scriptsize [0.551, 0.566]} & 0.704 {\scriptsize [0.697, 0.710]} & 0.690 {\scriptsize [0.683, 0.696]} & \textbf{0.738 {\scriptsize [0.732, 0.744]}} \\
 &  & Cardiomegaly & 0.788 {\scriptsize [0.781, 0.795]} & 0.691 {\scriptsize [0.683, 0.700]} & 0.767 {\scriptsize [0.759, 0.774]} & 0.815 {\scriptsize [0.808, 0.821]} & \textbf{0.817 {\scriptsize [0.811, 0.823]}} \\
 &  & Pleural Effusion & 0.874 {\scriptsize [0.870, 0.877]} & 0.888 {\scriptsize [0.886, 0.891]} & 0.899 {\scriptsize [0.896, 0.902]} & 0.907 {\scriptsize [0.904, 0.910]} & \textbf{0.909 {\scriptsize [0.907, 0.912]}} \\
\cline{2-8}
 & \multirow[t]{3}{*}{NV-reason} & Atelectasis & \textbf{0.746 {\scriptsize [0.740, 0.752]}} & 0.600 {\scriptsize [0.592, 0.606]} & 0.715 {\scriptsize [0.709, 0.722]} & 0.715 {\scriptsize [0.708, 0.721]} & 0.740 {\scriptsize [0.733, 0.745]} \\
 &  & Cardiomegaly & 0.735 {\scriptsize [0.728, 0.743]} & 0.677 {\scriptsize [0.668, 0.686]} & \textbf{0.751 {\scriptsize [0.744, 0.759]}} & 0.735 {\scriptsize [0.728, 0.743]} & 0.735 {\scriptsize [0.726, 0.742]} \\
 &  & Pleural Effusion & 0.880 {\scriptsize [0.877, 0.883]} & 0.872 {\scriptsize [0.869, 0.875]} & 0.880 {\scriptsize [0.877, 0.883]} & 0.882 {\scriptsize [0.879, 0.885]} & \textbf{0.892 {\scriptsize [0.889, 0.894]}} \\
\cline{1-8} \cline{2-8}
\multirow[t]{9}{*}{MIMIC-CXR} & \multirow[t]{3}{*}{Llava-RAD} & Atelectasis & 0.700 {\scriptsize [0.693, 0.707]} & 0.710 {\scriptsize [0.703, 0.717]} & 0.729 {\scriptsize [0.722, 0.735]} & 0.812 {\scriptsize [0.807, 0.818]} & \textbf{0.865 {\scriptsize [0.861, 0.870]}} \\
 &  & Cardiomegaly & 0.661 {\scriptsize [0.654, 0.668]} & 0.660 {\scriptsize [0.653, 0.667]} & 0.662 {\scriptsize [0.655, 0.669]} & 0.688 {\scriptsize [0.681, 0.695]} & \textbf{0.732 {\scriptsize [0.725, 0.737]}} \\
 &  & Pleural Effusion & 0.799 {\scriptsize [0.794, 0.805]} & 0.769 {\scriptsize [0.763, 0.775]} & 0.807 {\scriptsize [0.801, 0.812]} & 0.855 {\scriptsize [0.851, 0.860]} & \textbf{0.860 {\scriptsize [0.856, 0.865]}} \\
\cline{2-8}
 & \multirow[t]{3}{*}{Medgemma} & Atelectasis & 0.726 {\scriptsize [0.719, 0.732]} & 0.488 {\scriptsize [0.479, 0.496]} & 0.647 {\scriptsize [0.639, 0.655]} & 0.776 {\scriptsize [0.770, 0.782]} & \textbf{0.803 {\scriptsize [0.797, 0.809]}} \\
 &  & Cardiomegaly & 0.737 {\scriptsize [0.731, 0.744]} & 0.628 {\scriptsize [0.621, 0.636]} & 0.726 {\scriptsize [0.719, 0.732]} & 0.768 {\scriptsize [0.763, 0.774]} & \textbf{0.780 {\scriptsize [0.774, 0.785]}} \\
 &  & Pleural Effusion & 0.809 {\scriptsize [0.803, 0.814]} & 0.795 {\scriptsize [0.789, 0.801]} & 0.809 {\scriptsize [0.804, 0.815]} & 0.834 {\scriptsize [0.829, 0.839]} & \textbf{0.837 {\scriptsize [0.832, 0.842]}} \\
\cline{2-8}
 & \multirow[t]{3}{*}{NV-reason} & Atelectasis & 0.796 {\scriptsize [0.790, 0.801]} & 0.587 {\scriptsize [0.579, 0.595]} & 0.755 {\scriptsize [0.748, 0.761]} & 0.813 {\scriptsize [0.807, 0.818]} & \textbf{0.829 {\scriptsize [0.824, 0.834]}} \\
 &  & Cardiomegaly & 0.597 {\scriptsize [0.588, 0.604]} & \textbf{0.648 {\scriptsize [0.640, 0.655]}} & 0.625 {\scriptsize [0.618, 0.633]} & 0.622 {\scriptsize [0.615, 0.630]} & 0.642 {\scriptsize [0.635, 0.649]} \\
 &  & Pleural Effusion & 0.769 {\scriptsize [0.764, 0.774]} & \textbf{0.798 {\scriptsize [0.793, 0.804]}} & 0.739 {\scriptsize [0.733, 0.744]} & 0.762 {\scriptsize [0.756, 0.767]} & 0.770 {\scriptsize [0.765, 0.776]} \\
\cline{1-8} \cline{2-8}
\bottomrule
\end{tabular}
}
\end{table}

\paragraph{Additional metrics} We report standard chest X-ray report generation metrics from the literature. As expected, surface-level text-similarity metrics (BLEU, ROUGE) are highest for SFT, since SFT directly optimizes for token-level overlap with the reference report. However, this high text overlap masks a substantial degradation in clinical performance: SFT shows lower sensitivity and F1 on the three findings of interest — denoted micro/macro F1 (3) and sensitivity (3) — than the base model itself, indicating that imitating the reference reports containing omission noise suppresses detection. We further hypothesize that SFT disproportionately suppresses pathologies with lower prevalence in the training distribution, since the model receives correspondingly fewer positive mentions to learn from.

In contrast, DPO and \method preserve text-overlap performance relative to the base model on BLEU and ROUGE, and in some cases improve on it, while delivering the detection-rate gains reported in the main paper. Crucially, both methods also maintain micro/macro F1 across all 14 CheXpert findings and RadGraph-F1 at base-model levels, indicating that gains on the targeted findings do not come at the cost of degraded performance on the broader set of clinical content.

\begin{table}[!ht]
\caption{Eval metrics for CheXpert by model, metric, and method (mean [CI])}
\label{tab:eval_metrics_chexpert_by_model_metric_method}
\resizebox{\linewidth}{!}{%
\begin{tabular}{lllllll}
\toprule
 & Method & Base & SFT & Radcliq-GRPO & DPO & PU-DPO \\
\midrule
\multirow[t]{14}{*}{Llava-RAD} & BLEU-1 & 0.122 {\scriptsize [0.121, 0.123]} & \textbf{0.338 {\scriptsize [0.336, 0.340]}} & 0.119 {\scriptsize [0.118, 0.120]} & 0.123 {\scriptsize [0.123, 0.124]} & 0.122 {\scriptsize [0.121, 0.123]} \\
 & BLEU-4 & 0.013 {\scriptsize [0.013, 0.013]} & \textbf{0.132 {\scriptsize [0.131, 0.133]}} & 0.013 {\scriptsize [0.013, 0.013]} & 0.014 {\scriptsize [0.014, 0.015]} & 0.014 {\scriptsize [0.014, 0.015]} \\
 & ROUGE-1 & 0.254 {\scriptsize [0.253, 0.255]} & \textbf{0.378 {\scriptsize [0.377, 0.379]}} & 0.255 {\scriptsize [0.255, 0.256]} & 0.260 {\scriptsize [0.259, 0.261]} & 0.261 {\scriptsize [0.261, 0.262]} \\
 & ROUGE-2 & 0.063 {\scriptsize [0.062, 0.063]} & \textbf{0.156 {\scriptsize [0.155, 0.157]}} & 0.064 {\scriptsize [0.063, 0.064]} & 0.069 {\scriptsize [0.068, 0.069]} & 0.070 {\scriptsize [0.069, 0.070]} \\
 & ROUGE-L & 0.166 {\scriptsize [0.166, 0.166]} & \textbf{0.281 {\scriptsize [0.280, 0.282]}} & 0.167 {\scriptsize [0.167, 0.168]} & 0.168 {\scriptsize [0.167, 0.168]} & 0.169 {\scriptsize [0.168, 0.169]} \\
 & Micro F1 (3) & 0.790 {\scriptsize [0.787, 0.793]} & 0.769 {\scriptsize [0.766, 0.771]} & 0.799 {\scriptsize [0.795, 0.802]} & 0.829 {\scriptsize [0.827, 0.832]} & \textbf{0.840 {\scriptsize [0.837, 0.842]}} \\
 & Micro F1 (14) & 0.675 {\scriptsize [0.673, 0.677]} & \textbf{0.747 {\scriptsize [0.744, 0.749]}} & 0.681 {\scriptsize [0.679, 0.683]} & 0.692 {\scriptsize [0.690, 0.694]} & 0.695 {\scriptsize [0.693, 0.697]} \\
 & Macro F1 (3) & 0.772 {\scriptsize [0.769, 0.776]} & 0.726 {\scriptsize [0.722, 0.730]} & 0.780 {\scriptsize [0.777, 0.784]} & 0.792 {\scriptsize [0.789, 0.796]} & \textbf{0.803 {\scriptsize [0.799, 0.806]}} \\
 & Macro F1 (14) & 0.483 {\scriptsize [0.479, 0.486]} & \textbf{0.522 {\scriptsize [0.518, 0.525]}} & 0.482 {\scriptsize [0.479, 0.485]} & 0.503 {\scriptsize [0.499, 0.506]} & 0.504 {\scriptsize [0.500, 0.508]} \\
 & Sensitivity (3) & 0.726 {\scriptsize [0.722, 0.731]} & 0.658 {\scriptsize [0.654, 0.662]} & 0.734 {\scriptsize [0.729, 0.738]} & 0.773 {\scriptsize [0.770, 0.777]} & \textbf{0.793 {\scriptsize [0.790, 0.797]}} \\
 & Sensitivity (14) & 0.596 {\scriptsize [0.593, 0.598]} & \textbf{0.658 {\scriptsize [0.655, 0.660]}} & 0.599 {\scriptsize [0.596, 0.601]} & 0.599 {\scriptsize [0.596, 0.601]} & 0.604 {\scriptsize [0.602, 0.606]} \\
 & Specificity (3) & 0.890 {\scriptsize [0.884, 0.896]} & \textbf{0.899 {\scriptsize [0.893, 0.904]}} & 0.892 {\scriptsize [0.886, 0.898]} & 0.847 {\scriptsize [0.840, 0.854]} & 0.817 {\scriptsize [0.809, 0.824]} \\
 & Specificity (14) & 0.903 {\scriptsize [0.901, 0.905]} & \textbf{0.929 {\scriptsize [0.927, 0.930]}} & 0.905 {\scriptsize [0.903, 0.907]} & 0.901 {\scriptsize [0.899, 0.903]} & 0.902 {\scriptsize [0.900, 0.904]} \\
 & RadGraph F1 & 0.040 {\scriptsize [0.039, 0.041]} & \textbf{0.245 {\scriptsize [0.244, 0.247]}} & 0.041 {\scriptsize [0.040, 0.041]} & 0.043 {\scriptsize [0.043, 0.044]} & 0.044 {\scriptsize [0.043, 0.045]} \\
\cline{1-7}
\multirow[t]{14}{*}{Medgemma} & BLEU-1 & 0.136 {\scriptsize [0.135, 0.137]} & \textbf{0.293 {\scriptsize [0.291, 0.296]}} & 0.145 {\scriptsize [0.144, 0.147]} & 0.096 {\scriptsize [0.095, 0.097]} & 0.097 {\scriptsize [0.096, 0.098]} \\
 & BLEU-4 & 0.022 {\scriptsize [0.021, 0.022]} & \textbf{0.112 {\scriptsize [0.111, 0.113]}} & 0.026 {\scriptsize [0.026, 0.026]} & 0.013 {\scriptsize [0.013, 0.013]} & 0.013 {\scriptsize [0.013, 0.013]} \\
 & ROUGE-1 & 0.270 {\scriptsize [0.269, 0.271]} & \textbf{0.360 {\scriptsize [0.359, 0.361]}} & 0.279 {\scriptsize [0.278, 0.280]} & 0.242 {\scriptsize [0.241, 0.242]} & 0.243 {\scriptsize [0.242, 0.243]} \\
 & ROUGE-2 & 0.074 {\scriptsize [0.074, 0.075]} & \textbf{0.143 {\scriptsize [0.142, 0.143]}} & 0.081 {\scriptsize [0.080, 0.081]} & 0.067 {\scriptsize [0.066, 0.067]} & 0.066 {\scriptsize [0.066, 0.067]} \\
 & ROUGE-L & 0.165 {\scriptsize [0.165, 0.166]} & \textbf{0.267 {\scriptsize [0.266, 0.268]}} & 0.172 {\scriptsize [0.172, 0.173]} & 0.148 {\scriptsize [0.147, 0.148]} & 0.149 {\scriptsize [0.148, 0.149]} \\
 & Micro F1 (3) & 0.798 {\scriptsize [0.796, 0.801]} & 0.761 {\scriptsize [0.757, 0.764]} & 0.817 {\scriptsize [0.814, 0.820]} & 0.826 {\scriptsize [0.823, 0.828]} & \textbf{0.841 {\scriptsize [0.838, 0.843]}} \\
 & Micro F1 (14) & 0.691 {\scriptsize [0.688, 0.693]} & \textbf{0.735 {\scriptsize [0.733, 0.737]}} & 0.669 {\scriptsize [0.666, 0.671]} & 0.724 {\scriptsize [0.722, 0.726]} & 0.723 {\scriptsize [0.721, 0.725]} \\
 & Macro F1 (3) & 0.779 {\scriptsize [0.776, 0.782]} & 0.713 {\scriptsize [0.709, 0.717]} & 0.790 {\scriptsize [0.786, 0.793]} & 0.804 {\scriptsize [0.800, 0.807]} & \textbf{0.822 {\scriptsize [0.818, 0.825]}} \\
 & Macro F1 (14) & 0.431 {\scriptsize [0.428, 0.434]} & \textbf{0.478 {\scriptsize [0.475, 0.481]}} & 0.408 {\scriptsize [0.405, 0.411]} & 0.466 {\scriptsize [0.462, 0.469]} & 0.459 {\scriptsize [0.456, 0.462]} \\
 & Sensitivity (3) & 0.816 {\scriptsize [0.813, 0.819]} & 0.658 {\scriptsize [0.654, 0.662]} & 0.837 {\scriptsize [0.834, 0.840]} & 0.827 {\scriptsize [0.825, 0.830]} & \textbf{0.849 {\scriptsize [0.846, 0.852]}} \\
 & Sensitivity (14) & 0.657 {\scriptsize [0.655, 0.659]} & 0.648 {\scriptsize [0.646, 0.650]} & 0.626 {\scriptsize [0.623, 0.628]} & \textbf{0.689 {\scriptsize [0.687, 0.691]}} & 0.683 {\scriptsize [0.681, 0.686]} \\
 & Specificity (3) & \textbf{0.872 {\scriptsize [0.865, 0.878]}} & 0.870 {\scriptsize [0.864, 0.876]} & 0.835 {\scriptsize [0.828, 0.842]} & 0.849 {\scriptsize [0.842, 0.856]} & 0.835 {\scriptsize [0.828, 0.842]} \\
 & Specificity (14) & 0.914 {\scriptsize [0.912, 0.915]} & \textbf{0.924 {\scriptsize [0.922, 0.925]}} & 0.902 {\scriptsize [0.900, 0.904]} & 0.914 {\scriptsize [0.913, 0.916]} & 0.923 {\scriptsize [0.922, 0.925]} \\
 & RadGraph F1 & 0.051 {\scriptsize [0.050, 0.052]} & \textbf{0.223 {\scriptsize [0.222, 0.225]}} & 0.050 {\scriptsize [0.049, 0.051]} & 0.047 {\scriptsize [0.047, 0.048]} & 0.048 {\scriptsize [0.047, 0.049]} \\
\cline{1-7}
\multirow[t]{14}{*}{NV-reason} & BLEU-1 & 0.040 {\scriptsize [0.040, 0.041]} & \textbf{0.265 {\scriptsize [0.263, 0.267]}} & 0.040 {\scriptsize [0.040, 0.041]} & 0.042 {\scriptsize [0.041, 0.042]} & 0.042 {\scriptsize [0.041, 0.042]} \\
 & BLEU-4 & 0.004 {\scriptsize [0.004, 0.004]} & \textbf{0.098 {\scriptsize [0.097, 0.098]}} & 0.004 {\scriptsize [0.004, 0.004]} & 0.004 {\scriptsize [0.004, 0.005]} & 0.004 {\scriptsize [0.004, 0.005]} \\
 & ROUGE-1 & 0.116 {\scriptsize [0.116, 0.117]} & \textbf{0.328 {\scriptsize [0.327, 0.330]}} & 0.115 {\scriptsize [0.115, 0.116]} & 0.118 {\scriptsize [0.118, 0.119]} & 0.118 {\scriptsize [0.118, 0.119]} \\
 & ROUGE-2 & 0.033 {\scriptsize [0.033, 0.033]} & \textbf{0.125 {\scriptsize [0.124, 0.126]}} & 0.033 {\scriptsize [0.032, 0.033]} & 0.034 {\scriptsize [0.034, 0.034]} & 0.034 {\scriptsize [0.034, 0.034]} \\
 & ROUGE-L & 0.076 {\scriptsize [0.076, 0.076]} & \textbf{0.243 {\scriptsize [0.242, 0.244]}} & 0.076 {\scriptsize [0.075, 0.076]} & 0.078 {\scriptsize [0.078, 0.078]} & 0.078 {\scriptsize [0.078, 0.078]} \\
 & Micro F1 (3) & 0.814 {\scriptsize [0.811, 0.817]} & 0.759 {\scriptsize [0.755, 0.762]} & 0.808 {\scriptsize [0.805, 0.811]} & 0.806 {\scriptsize [0.803, 0.809]} & \textbf{0.818 {\scriptsize [0.816, 0.821]}} \\
 & Micro F1 (14) & 0.728 {\scriptsize [0.726, 0.730]} & 0.714 {\scriptsize [0.712, 0.716]} & \textbf{0.735 {\scriptsize [0.733, 0.737]}} & 0.709 {\scriptsize [0.707, 0.711]} & 0.720 {\scriptsize [0.718, 0.722]} \\
 & Macro F1 (3) & 0.787 {\scriptsize [0.784, 0.790]} & 0.716 {\scriptsize [0.712, 0.720]} & 0.782 {\scriptsize [0.779, 0.785]} & 0.777 {\scriptsize [0.774, 0.781]} & \textbf{0.789 {\scriptsize [0.785, 0.792]}} \\
 & Macro F1 (14) & 0.448 {\scriptsize [0.444, 0.451]} & \textbf{0.460 {\scriptsize [0.457, 0.463]}} & 0.451 {\scriptsize [0.447, 0.454]} & 0.430 {\scriptsize [0.427, 0.433]} & 0.434 {\scriptsize [0.431, 0.437]} \\
 & Sensitivity (3) & \textbf{0.849 {\scriptsize [0.846, 0.852]}} & 0.645 {\scriptsize [0.641, 0.649]} & 0.833 {\scriptsize [0.830, 0.836]} & 0.834 {\scriptsize [0.831, 0.837]} & 0.847 {\scriptsize [0.844, 0.850]} \\
 & Sensitivity (14) & 0.658 {\scriptsize [0.656, 0.660]} & 0.612 {\scriptsize [0.610, 0.614]} & \textbf{0.664 {\scriptsize [0.662, 0.667]}} & 0.635 {\scriptsize [0.633, 0.637]} & 0.644 {\scriptsize [0.641, 0.646]} \\
 & Specificity (3) & 0.751 {\scriptsize [0.743, 0.760]} & \textbf{0.862 {\scriptsize [0.855, 0.869]}} & 0.741 {\scriptsize [0.733, 0.749]} & 0.780 {\scriptsize [0.772, 0.788]} & 0.765 {\scriptsize [0.756, 0.773]} \\
 & Specificity (14) & 0.938 {\scriptsize [0.937, 0.940]} & 0.925 {\scriptsize [0.923, 0.926]} & 0.935 {\scriptsize [0.934, 0.937]} & \textbf{0.945 {\scriptsize [0.944, 0.947]}} & 0.941 {\scriptsize [0.939, 0.942]} \\
 & RadGraph F1 & 0.036 {\scriptsize [0.036, 0.037]} & \textbf{0.185 {\scriptsize [0.184, 0.186]}} & 0.036 {\scriptsize [0.035, 0.037]} & 0.036 {\scriptsize [0.035, 0.037]} & 0.036 {\scriptsize [0.036, 0.037]} \\
\cline{1-7}
\bottomrule
\end{tabular}%
}
\end{table}

\begin{table}[!ht]
\caption{Eval metrics for MIMIC by model, metric, and method (mean [CI])}
\label{tab:eval_metrics_mimic_by_model_metric_method}
\resizebox{\linewidth}{!}{%
\begin{tabular}{lllllll}
\toprule
 & Method & Base & SFT & Radcliq-GRPO & DPO & PU-DPO \\
\midrule
\multirow[t]{14}{*}{Llava-RAD} & BLEU-1 & 0.289 {\scriptsize [0.287, 0.290]} & \textbf{0.319 {\scriptsize [0.317, 0.321]}} & 0.290 {\scriptsize [0.289, 0.292]} & 0.295 {\scriptsize [0.294, 0.297]} & 0.292 {\scriptsize [0.290, 0.293]} \\
 & BLEU-4 & 0.081 {\scriptsize [0.081, 0.082]} & \textbf{0.113 {\scriptsize [0.112, 0.115]}} & 0.083 {\scriptsize [0.082, 0.084]} & 0.081 {\scriptsize [0.080, 0.082]} & 0.080 {\scriptsize [0.079, 0.081]} \\
 & ROUGE-1 & 0.342 {\scriptsize [0.341, 0.343]} & \textbf{0.383 {\scriptsize [0.381, 0.384]}} & 0.344 {\scriptsize [0.343, 0.345]} & 0.347 {\scriptsize [0.346, 0.348]} & 0.346 {\scriptsize [0.345, 0.347]} \\
 & ROUGE-2 & 0.125 {\scriptsize [0.124, 0.126]} & \textbf{0.156 {\scriptsize [0.155, 0.157]}} & 0.127 {\scriptsize [0.126, 0.128]} & 0.126 {\scriptsize [0.125, 0.127]} & 0.125 {\scriptsize [0.124, 0.126]} \\
 & ROUGE-L & 0.237 {\scriptsize [0.236, 0.238]} & \textbf{0.284 {\scriptsize [0.283, 0.286]}} & 0.238 {\scriptsize [0.237, 0.239]} & 0.237 {\scriptsize [0.236, 0.238]} & 0.235 {\scriptsize [0.234, 0.236]} \\
 & Micro F1 (3) & 0.722 {\scriptsize [0.718, 0.726]} & 0.713 {\scriptsize [0.709, 0.718]} & 0.735 {\scriptsize [0.731, 0.738]} & 0.789 {\scriptsize [0.786, 0.792]} & \textbf{0.821 {\scriptsize [0.818, 0.824]}} \\
 & Micro F1 (14) & 0.667 {\scriptsize [0.664, 0.670]} & 0.634 {\scriptsize [0.631, 0.637]} & 0.675 {\scriptsize [0.673, 0.678]} & 0.707 {\scriptsize [0.704, 0.709]} & \textbf{0.716 {\scriptsize [0.713, 0.718]}} \\
 & Macro F1 (3) & 0.720 {\scriptsize [0.716, 0.724]} & 0.713 {\scriptsize [0.708, 0.717]} & 0.733 {\scriptsize [0.729, 0.736]} & 0.785 {\scriptsize [0.782, 0.789]} & \textbf{0.819 {\scriptsize [0.816, 0.822]}} \\
 & Macro F1 (14) & 0.534 {\scriptsize [0.530, 0.538]} & 0.471 {\scriptsize [0.468, 0.476]} & 0.541 {\scriptsize [0.537, 0.546]} & 0.566 {\scriptsize [0.562, 0.571]} & \textbf{0.568 {\scriptsize [0.564, 0.572]}} \\
 & Sensitivity (3) & 0.678 {\scriptsize [0.674, 0.683]} & 0.648 {\scriptsize [0.643, 0.653]} & 0.697 {\scriptsize [0.693, 0.702]} & 0.772 {\scriptsize [0.768, 0.776]} & \textbf{0.825 {\scriptsize [0.821, 0.828]}} \\
 & Sensitivity (14) & 0.638 {\scriptsize [0.635, 0.641]} & 0.597 {\scriptsize [0.594, 0.601]} & 0.648 {\scriptsize [0.645, 0.651]} & 0.683 {\scriptsize [0.680, 0.685]} & \textbf{0.693 {\scriptsize [0.690, 0.696]}} \\
 & Specificity (3) & 0.935 {\scriptsize [0.932, 0.937]} & \textbf{0.944 {\scriptsize [0.941, 0.946]}} & 0.932 {\scriptsize [0.929, 0.935]} & 0.913 {\scriptsize [0.909, 0.916]} & 0.888 {\scriptsize [0.885, 0.892]} \\
 & Specificity (14) & 0.917 {\scriptsize [0.915, 0.918]} & 0.911 {\scriptsize [0.910, 0.912]} & 0.919 {\scriptsize [0.918, 0.920]} & \textbf{0.921 {\scriptsize [0.919, 0.922]}} & 0.918 {\scriptsize [0.916, 0.919]} \\
 & RadGraph F1 & 0.222 {\scriptsize [0.221, 0.223]} & \textbf{0.248 {\scriptsize [0.246, 0.250]}} & 0.227 {\scriptsize [0.225, 0.228]} & 0.230 {\scriptsize [0.228, 0.231]} & 0.234 {\scriptsize [0.233, 0.236]} \\
\cline{1-7}
\multirow[t]{14}{*}{Medgemma} & BLEU-1 & 0.280 {\scriptsize [0.278, 0.281]} & \textbf{0.349 {\scriptsize [0.347, 0.351]}} & 0.290 {\scriptsize [0.288, 0.291]} & 0.177 {\scriptsize [0.176, 0.178]} & 0.176 {\scriptsize [0.174, 0.177]} \\
 & BLEU-4 & 0.067 {\scriptsize [0.066, 0.068]} & \textbf{0.117 {\scriptsize [0.116, 0.118]}} & 0.072 {\scriptsize [0.071, 0.073]} & 0.036 {\scriptsize [0.036, 0.037]} & 0.036 {\scriptsize [0.035, 0.036]} \\
 & ROUGE-1 & 0.350 {\scriptsize [0.349, 0.352]} & \textbf{0.377 {\scriptsize [0.376, 0.379]}} & 0.354 {\scriptsize [0.352, 0.355]} & 0.308 {\scriptsize [0.307, 0.310]} & 0.310 {\scriptsize [0.308, 0.311]} \\
 & ROUGE-2 & 0.128 {\scriptsize [0.127, 0.129]} & \textbf{0.148 {\scriptsize [0.146, 0.149]}} & 0.131 {\scriptsize [0.130, 0.132]} & 0.105 {\scriptsize [0.104, 0.106]} & 0.106 {\scriptsize [0.105, 0.107]} \\
 & ROUGE-L & 0.242 {\scriptsize [0.241, 0.243]} & \textbf{0.277 {\scriptsize [0.276, 0.278]}} & 0.248 {\scriptsize [0.246, 0.249]} & 0.209 {\scriptsize [0.208, 0.210]} & 0.210 {\scriptsize [0.209, 0.211]} \\
 & Micro F1 (3) & 0.759 {\scriptsize [0.756, 0.763]} & 0.649 {\scriptsize [0.645, 0.654]} & 0.733 {\scriptsize [0.729, 0.737]} & 0.794 {\scriptsize [0.791, 0.798]} & \textbf{0.807 {\scriptsize [0.804, 0.810]}} \\
 & Micro F1 (14) & 0.673 {\scriptsize [0.670, 0.675]} & 0.587 {\scriptsize [0.584, 0.590]} & 0.655 {\scriptsize [0.652, 0.658]} & 0.699 {\scriptsize [0.697, 0.702]} & \textbf{0.708 {\scriptsize [0.706, 0.711]}} \\
 & Macro F1 (3) & 0.757 {\scriptsize [0.754, 0.761]} & 0.637 {\scriptsize [0.632, 0.642]} & 0.727 {\scriptsize [0.723, 0.731]} & 0.793 {\scriptsize [0.789, 0.796]} & \textbf{0.806 {\scriptsize [0.803, 0.810]}} \\
 & Macro F1 (14) & 0.476 {\scriptsize [0.472, 0.480]} & 0.420 {\scriptsize [0.416, 0.424]} & 0.454 {\scriptsize [0.450, 0.457]} & 0.512 {\scriptsize [0.508, 0.516]} & \textbf{0.516 {\scriptsize [0.512, 0.520]}} \\
 & Sensitivity (3) & 0.833 {\scriptsize [0.830, 0.836]} & 0.628 {\scriptsize [0.623, 0.633]} & 0.810 {\scriptsize [0.807, 0.813]} & 0.861 {\scriptsize [0.857, 0.864]} & \textbf{0.874 {\scriptsize [0.871, 0.878]}} \\
 & Sensitivity (14) & 0.709 {\scriptsize [0.706, 0.712]} & 0.599 {\scriptsize [0.596, 0.602]} & 0.701 {\scriptsize [0.698, 0.703]} & 0.736 {\scriptsize [0.733, 0.738]} & \textbf{0.745 {\scriptsize [0.742, 0.747]}} \\
 & Specificity (3) & 0.878 {\scriptsize [0.874, 0.881]} & \textbf{0.942 {\scriptsize [0.940, 0.945]}} & 0.893 {\scriptsize [0.890, 0.897]} & 0.857 {\scriptsize [0.853, 0.861]} & 0.858 {\scriptsize [0.854, 0.863]} \\
 & Specificity (14) & 0.908 {\scriptsize [0.906, 0.909]} & 0.890 {\scriptsize [0.888, 0.891]} & 0.912 {\scriptsize [0.911, 0.914]} & 0.914 {\scriptsize [0.912, 0.915]} & \textbf{0.918 {\scriptsize [0.917, 0.920]}} \\
 & RadGraph F1 & 0.246 {\scriptsize [0.244, 0.247]} & 0.246 {\scriptsize [0.245, 0.248]} & \textbf{0.256 {\scriptsize [0.254, 0.257]}} & 0.221 {\scriptsize [0.219, 0.222]} & 0.224 {\scriptsize [0.222, 0.225]} \\
\cline{1-7}
\multirow[t]{14}{*}{NV-reason} & BLEU-1 & 0.072 {\scriptsize [0.072, 0.073]} & \textbf{0.333 {\scriptsize [0.332, 0.335]}} & 0.070 {\scriptsize [0.070, 0.071]} & 0.081 {\scriptsize [0.081, 0.082]} & 0.082 {\scriptsize [0.081, 0.082]} \\
 & BLEU-4 & 0.012 {\scriptsize [0.012, 0.012]} & \textbf{0.118 {\scriptsize [0.117, 0.119]}} & 0.012 {\scriptsize [0.012, 0.012]} & 0.012 {\scriptsize [0.012, 0.012]} & 0.012 {\scriptsize [0.012, 0.013]} \\
 & ROUGE-1 & 0.125 {\scriptsize [0.125, 0.126]} & \textbf{0.377 {\scriptsize [0.376, 0.379]}} & 0.122 {\scriptsize [0.122, 0.122]} & 0.133 {\scriptsize [0.132, 0.133]} & 0.134 {\scriptsize [0.133, 0.134]} \\
 & ROUGE-2 & 0.038 {\scriptsize [0.038, 0.038]} & \textbf{0.154 {\scriptsize [0.153, 0.156]}} & 0.036 {\scriptsize [0.036, 0.036]} & 0.037 {\scriptsize [0.036, 0.037]} & 0.037 {\scriptsize [0.037, 0.037]} \\
 & ROUGE-L & 0.082 {\scriptsize [0.082, 0.083]} & \textbf{0.281 {\scriptsize [0.280, 0.283]}} & 0.080 {\scriptsize [0.080, 0.081]} & 0.086 {\scriptsize [0.086, 0.086]} & 0.087 {\scriptsize [0.087, 0.087]} \\
 & Micro F1 (3) & 0.729 {\scriptsize [0.726, 0.733]} & 0.683 {\scriptsize [0.679, 0.688]} & 0.712 {\scriptsize [0.708, 0.715]} & 0.738 {\scriptsize [0.734, 0.742]} & \textbf{0.752 {\scriptsize [0.749, 0.756]}} \\
 & Micro F1 (14) & 0.663 {\scriptsize [0.661, 0.666]} & 0.616 {\scriptsize [0.613, 0.618]} & 0.658 {\scriptsize [0.656, 0.661]} & 0.659 {\scriptsize [0.656, 0.661]} & \textbf{0.666 {\scriptsize [0.663, 0.668]}} \\
 & Macro F1 (3) & 0.720 {\scriptsize [0.717, 0.724]} & 0.678 {\scriptsize [0.673, 0.682]} & 0.706 {\scriptsize [0.702, 0.710]} & 0.732 {\scriptsize [0.728, 0.736]} & \textbf{0.747 {\scriptsize [0.743, 0.751]}} \\
 & Macro F1 (14) & 0.459 {\scriptsize [0.455, 0.464]} & 0.454 {\scriptsize [0.450, 0.457]} & 0.449 {\scriptsize [0.445, 0.452]} & 0.448 {\scriptsize [0.445, 0.452]} & \textbf{0.461 {\scriptsize [0.456, 0.465]}} \\
 & Sensitivity (3) & 0.867 {\scriptsize [0.864, 0.871]} & 0.625 {\scriptsize [0.620, 0.630]} & 0.845 {\scriptsize [0.841, 0.848]} & 0.879 {\scriptsize [0.876, 0.882]} & \textbf{0.889 {\scriptsize [0.886, 0.892]}} \\
 & Sensitivity (14) & \textbf{0.705 {\scriptsize [0.702, 0.708]}} & 0.593 {\scriptsize [0.590, 0.597]} & 0.700 {\scriptsize [0.697, 0.702]} & 0.699 {\scriptsize [0.696, 0.702]} & 0.705 {\scriptsize [0.702, 0.708]} \\
 & Specificity (3) & 0.802 {\scriptsize [0.797, 0.807]} & \textbf{0.951 {\scriptsize [0.949, 0.954]}} & 0.755 {\scriptsize [0.751, 0.761]} & 0.818 {\scriptsize [0.814, 0.823]} & 0.820 {\scriptsize [0.815, 0.824]} \\
 & Specificity (14) & 0.929 {\scriptsize [0.928, 0.931]} & 0.908 {\scriptsize [0.907, 0.909]} & 0.918 {\scriptsize [0.917, 0.920]} & \textbf{0.938 {\scriptsize [0.936, 0.939]}} & 0.937 {\scriptsize [0.935, 0.938]} \\
 & RadGraph F1 & 0.131 {\scriptsize [0.130, 0.131]} & \textbf{0.260 {\scriptsize [0.259, 0.262]}} & 0.131 {\scriptsize [0.130, 0.132]} & 0.129 {\scriptsize [0.128, 0.129]} & 0.130 {\scriptsize [0.129, 0.131]} \\
\cline{1-7}
\bottomrule
\end{tabular}%
}
\end{table}

\FloatBarrier

\subsection{Adjudicated test set}
\label{app:adjudicated_results}

We use the CheXpert test set (500 studies labeled by 5 radiologists), which serves as a radiologist adjudicated set that is also in-distribution (same institution) but unseen for our finetuned model. While the original reports are unavailable for these studies, the report-derived CheXbert labels are publicly available, so we evaluate recovery on hidden positives by treating the provided labels as the noisy diagnosis in the original radiologist report. We first report the statistics of this test set of 500 frontal images.

\begin{table}[htbp]
    \centering
    \caption{\textbf{Per-pathology statistics of the radiologist-adjudicated CheXpert test set (500 studies).} Prevalence is the fraction of studies the adjudicating radiologists labeled positive. Hidden positives are radiologist-positive studies whose original clinical report never mentioned the finding (blank CheXbert label); they form the denominator of the omission-recovery metric.}
    \label{tab:chexpert_stats}
    \begin{tabular}{lrrr}
        \toprule
        \textbf{Pathology} & \textbf{Positive} & \textbf{Prevalence} & \textbf{Hidden positive} \\
        \midrule
        Cardiomegaly     & 151 & 30.2\% & 102 \\
        Atelectasis      & 153 & 30.6\% & 80  \\
        Pleural Effusion & 104 & 20.8\% & 20  \\
        \bottomrule
    \end{tabular}
\end{table}

\textbf{Multi-SCAR Assumption.}
\label{sec:appendix_multi_scar} Intuitively, the multi-SCAR assumption dictates that labeled positives are a representative subsample of all true positives for a particular disease subtype. When this assumption is violated, for instance, if a certain finding is systematically reported only for ICU patients but present in both the ICU and ED, a distribution mismatch may occur between the labeled positives and the unlabeled hidden positives. Consequently, PU-DPO's reliance on labeled positive samples as a proxy for hidden positives breaks down.

We demonstrate a verification process for the SCAR assumption at the pathology level using our radiologist-adjudicated test set. To test for statistical independence between labeled and hidden positives, we extracted MedGemma image embeddings and evaluated the cohorts using a kernel MMD test, with corresponding $p$-values. As shown in \autoref{tab:scar_check}, Cardiomegaly exhibits detectable selection bias (violating SCAR) with statistical significance, whereas Atelectasis and Pleural Effusion show indistinguishable distributions, indicating that the SCAR assumption may hold.

\begin{table}[htbp]
    \centering
    \caption{\textbf{SCAR check on the radiologist-adjudicated CheXpert test set.} For each finding, labeled positives (the original report asserts the finding) are compared against hidden positives (the report excludes but the adjudicating radiologists confirmed the finding) in MedGemma image-embedding space. Under SCAR the two cohorts are assumed as draws from the same distribution, so the kernel-MMD permutation test should not reject and the cross-validated labeled-vs-hidden; $p$-values are permutation tails over 5000 resamples.}
    \label{tab:scar_check}
    \begin{tabular}{lrrrr}
        \toprule
        \textbf{Finding} & \textbf{Labeled} & \textbf{Hidden} & \textbf{MMD$^2$} & \textbf{MMD $p$} \\
        \midrule
        Cardiomegaly     & 46 & 102 &  0.0079 & 0.016 \\
        Atelectasis      & 73 &  80 & -0.0040 & 0.979 \\
        Pleural Effusion & 80 &  20 & -0.0081 & 0.940 \\
        \bottomrule
    \end{tabular}
\end{table}

\textbf{Pathology-specific results.} We provide a breakdown of the metrics across pathologies. For cardiomegaly, we find that PU-DPO does not improve omission recovery compared to standard DPO. Because the preference dataset and training parameters were strictly controlled, this suggests that the 
$\alpha$-correction term was ineffective at recovering hidden positives for this specific condition that violates SCAR. Conversely, for atelectasis, the PU-DPO loss successfully achieves its intended effect of hidden positive recovery, demonstrating noticeable gains over standard DPO. We will include these granular results, along with pathology-specific recall and F1 scores, and this corresponding interpretation in the revised manuscript.

\begin{figure}[!h]
    \centering
    \includegraphics[width=0.95\linewidth]{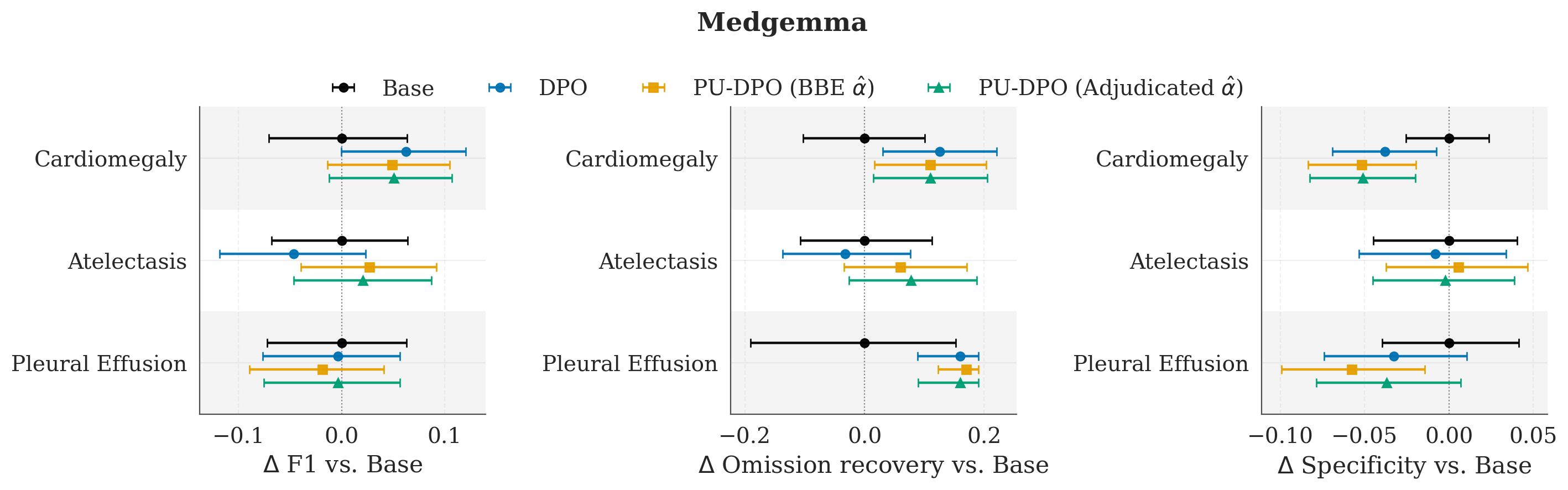}
    \caption{Pathology-specific results for F1, omission recovery, and specificity, Medgemma model.}
    \label{fig:pathologywise_medgemma}
\end{figure}

\begin{figure}[!h]
    \centering
    \includegraphics[width=0.95\linewidth]{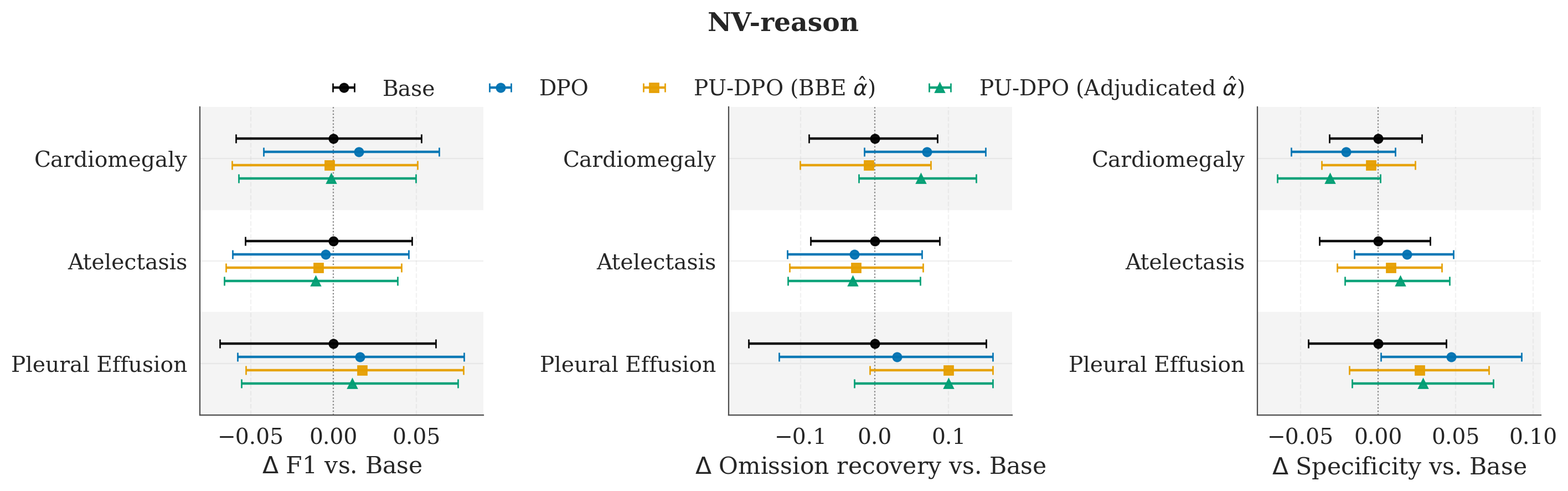}
    \caption{Pathology-specific results for F1, omission recovery, and specificity, NV-reason model.}
    \label{fig:pathologywise_nvreason}
\end{figure}

\begin{figure}[!h]
    \centering
    \includegraphics[width=0.95\linewidth]{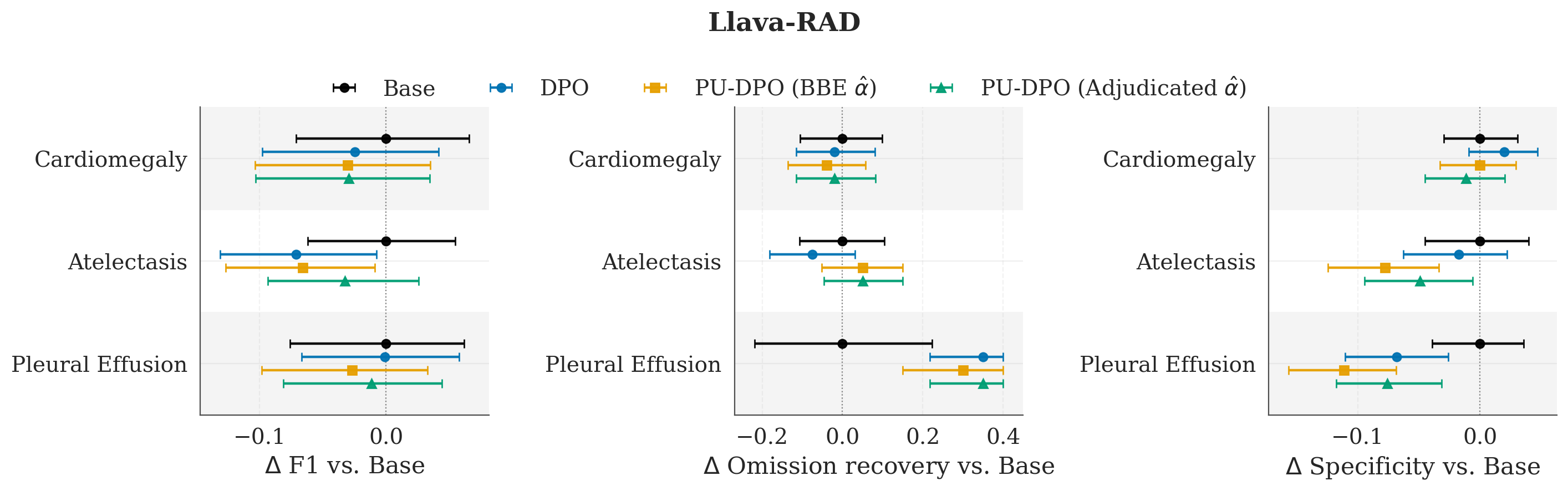}
    \caption{Pathology-specific results for F1, omission recovery, and specificity, Llava-RAD model.}
    \label{fig:pathologywise_llava}
\end{figure}

\FloatBarrier

\textbf{$\alpha$ Sensitivity Analysis} We provide checks for robustness to $\alpha$ misspecification across all models, where we perturb the adjudicated $\alpha$ estimate by $\pm 0.03$ to simulate over- and under-estimation.

\begin{figure}[!h]
    \centering
    \includegraphics[width=0.95\linewidth]{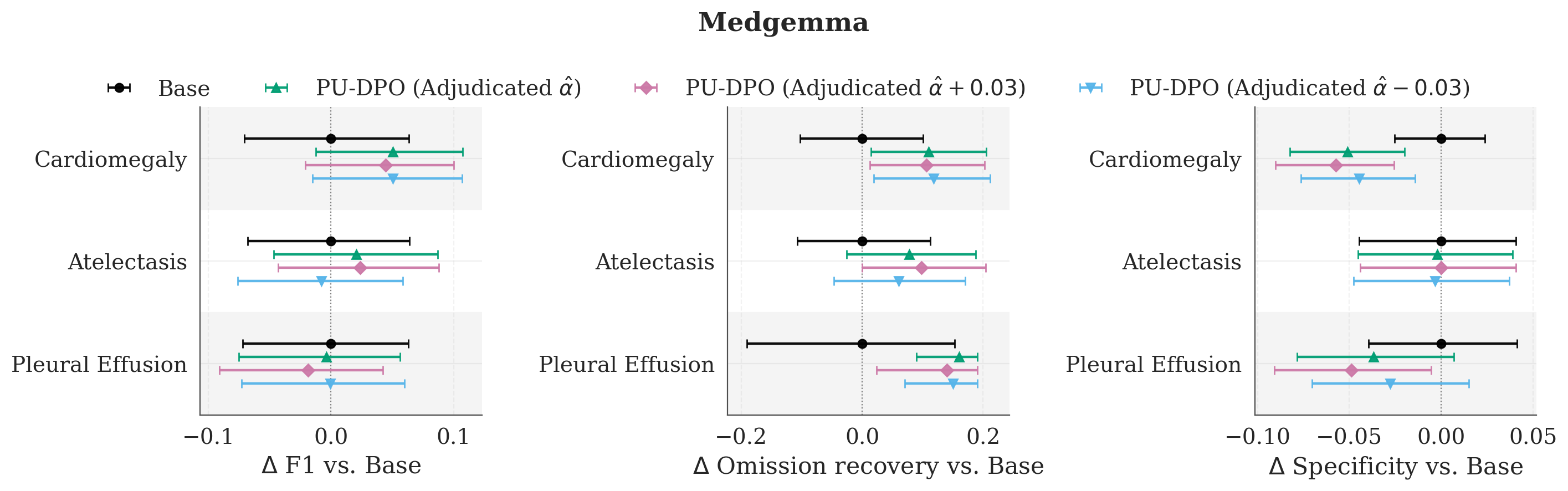}
    \caption{Sensitivity to $\alpha$ misspecification on F1, omission recovery, and specificity, Medgemma model.}
    \label{fig:alpha_medgemma}
\end{figure}

\begin{figure}[!h]
    \centering
    \includegraphics[width=0.95\linewidth]{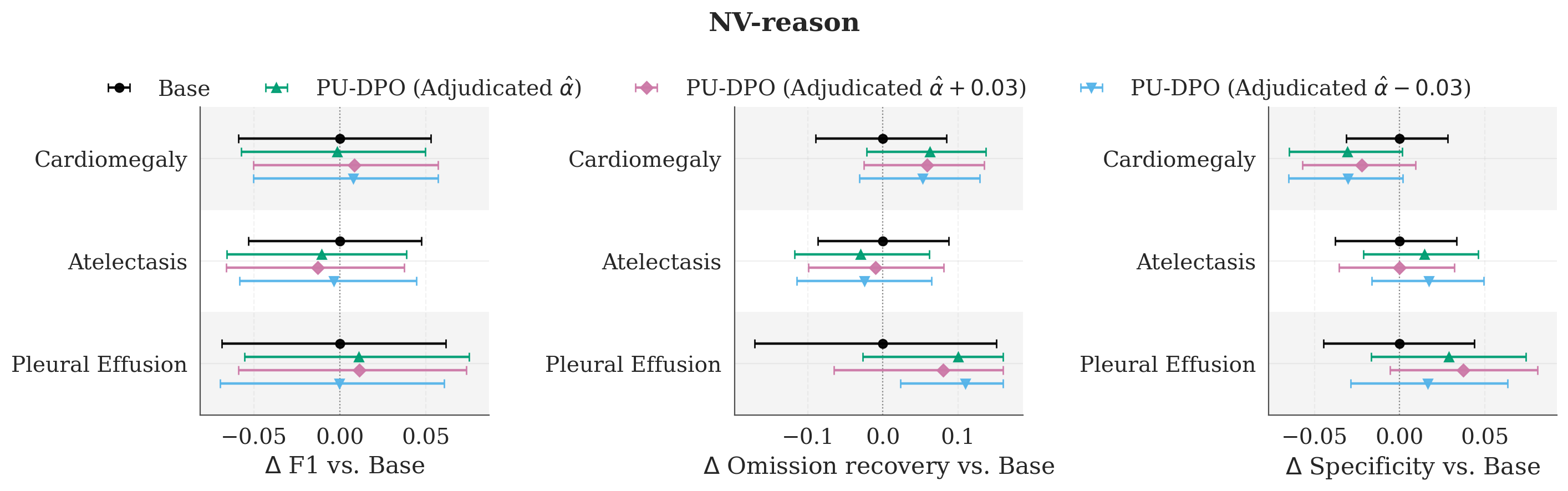}
    \caption{Sensitivity to $\alpha$ misspecification on F1, omission recovery, and specificity, NV-reason model.}
    \label{fig:alpha_nvreason}
\end{figure}

\begin{figure}[!h]
    \centering
    \includegraphics[width=0.95\linewidth]{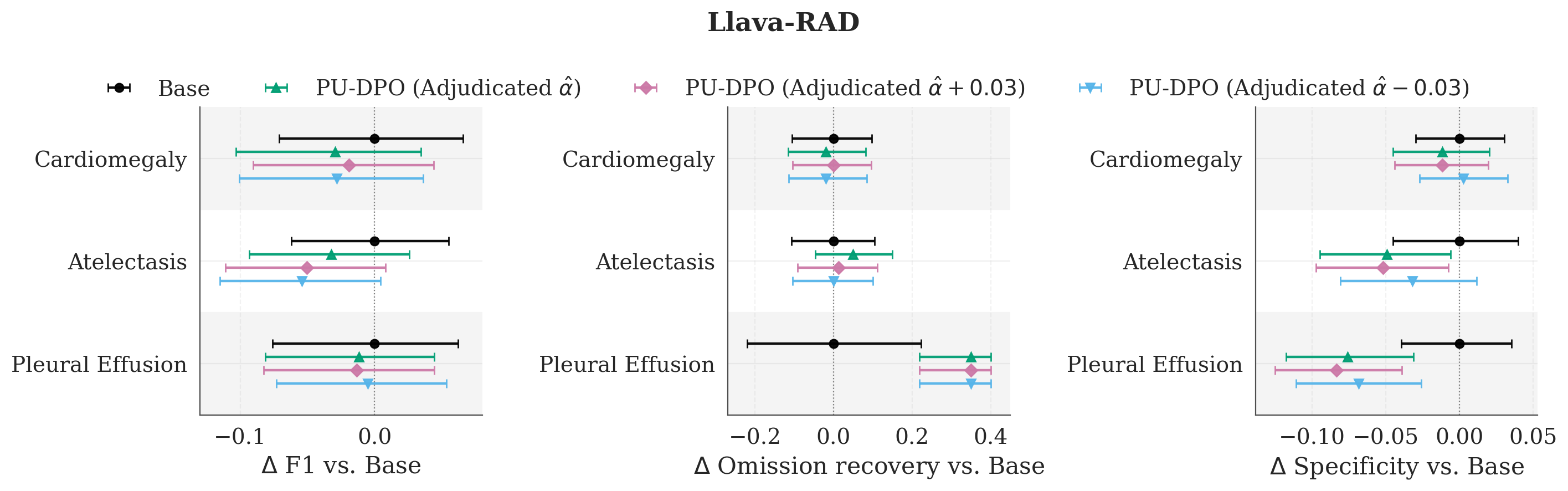}
    \caption{Sensitivity to $\alpha$ misspecification on F1, omission recovery, and specificity, Llava-RAD model.}
    \label{fig:alpha_llava}
\end{figure}

%%%%%%%%%%%%%%%%%%%%%%%%%%%%%%%%%%%%%%%%%%%%%%%%%%%%%%%%%%%%

% \newpage
% \input{checklist}

\end{document}